\documentclass[twocolumn]{article}
\usepackage[T1]{fontenc}
\usepackage[utf8]{inputenc}
\usepackage{color}
\usepackage{bm}
\usepackage{amsmath}
\usepackage{amsthm}
\usepackage{graphicx}
\usepackage[numbers]{natbib}
\usepackage{microtype}
\usepackage[unicode=true,
 bookmarks=true,bookmarksnumbered=true,bookmarksopen=true,bookmarksopenlevel=1,
 breaklinks=false,pdfborder={0 0 0},pdfborderstyle={},backref=false,colorlinks=true]
 {hyperref}

\makeatletter
\theoremstyle{plain}
\newtheorem{thm}{\protect\theoremname}

\usepackage{xcolor}
\definecolor{mycol}{rgb}{0,0,0.65}
\hypersetup{
    colorlinks,
    linkcolor={mycol},
    citecolor={mycol},
    urlcolor={mycol}
}

\usepackage{amsfonts,bm,amsmath}
\usepackage{lipsum}
\usepackage{nicefrac}

\usepackage[accepted]{icml2020}

\usepackage{url}
\usepackage{amsfonts}
\usepackage{nicefrac}

\DeclareMathAlphabet{\mathbfsf}{\encodingdefault}{\sfdefault}{bx}{n}
\newcommand{\upgreektemplate}[2]{#2{
\renewcommand{\alpha}{\upalpha}
\renewcommand{\beta}{\upbeta}
\renewcommand{\theta}{\uptheta}
\renewcommand{\gamma}{\upgamma}
\renewcommand{\lambda}{\uplambda}
\renewcommand{\delta}{\updelta}
\renewcommand{\phi}{\upphi}
\renewcommand{\zeta}{\upzeta}
\renewcommand{\Lambda}{\Uplambda}
\renewcommand{\Gamma}{\Upgamma}
\renewcommand{\Delta}{\Updelta}
\renewcommand{\Theta}{\Uptheta}
#1
}}
\newcommand{\upgreek}[1]{\upgreektemplate{#1}{\mathsf}}
\newcommand{\bupgreek}[1]{\upgreektemplate{#1}{\mathbfsf}}
\usepackage{upgreek}

\usepackage{thmtools}
\usepackage{thm-restate}
\usepackage{hyperref}

\usepackage[nameinlink,capitalize]{cleveref}
\AtBeginDocument{%
\let\ref\cref
}
\Crefname{equation}{Eq.}{Eqs.}
\Crefname{fig1}{Fig.}{Figs.}
\Crefname{lem1}{Lem.}{Lems.}
\Crefname{thm1}{Thm.}{Thms.}
\Crefname{section}{Sec.}{Secs.}

\makeatother

\providecommand{\theoremname}{Theorem}

\begin{document}
\global\long\def\argmin{\operatornamewithlimits{argmin}}%

\global\long\def\argmax{\operatornamewithlimits{argmax}}%

\global\long\def\prox{\operatornamewithlimits{prox}}%

\global\long\def\diag{\operatorname{diag}}%

\global\long\def\lse{\operatorname{lse}}%

\global\long\def\R{\mathbb{R}}%

\global\long\def\E{\operatornamewithlimits{\mathbb{E}}}%

\global\long\def\P{\operatornamewithlimits{\mathbb{P}}}%

\global\long\def\V{\operatornamewithlimits{\mathbb{V}}}%

\global\long\def\N{\mathcal{N}}%

\global\long\def\L{\mathcal{L}}%

\global\long\def\C{\mathbb{C}}%

\global\long\def\tr{\operatorname{tr}}%

\global\long\def\norm#1{\left\Vert #1\right\Vert }%

\global\long\def\norms#1{\left\Vert #1\right\Vert ^{2}}%

\global\long\def\pars#1{\left(#1\right)}%

\global\long\def\pp#1{(#1)}%

\global\long\def\bracs#1{\left[#1\right]}%

\global\long\def\bb#1{[#1]}%

\global\long\def\verts#1{\left\vert #1\right\vert }%

\global\long\def\vv#1{\vert#1\vert}%

\global\long\def\Verts#1{\left\Vert #1\right\Vert }%

\global\long\def\VV#1{\Vert#1\Vert}%

\global\long\def\angs#1{\left\langle #1\right\rangle }%

\global\long\def\KL#1{[#1]}%

\global\long\def\KL#1#2{KL\pars{#1\middle\Vert#2}}%

\global\long\def\div{\text{div}}%

\global\long\def\erf{\text{erf}}%

\global\long\def\vvec{\text{vec}}%

\global\long\def\b#1{\bm{#1}}%

\global\long\def\r#1{\upgreek{#1}}%

\global\long\def\br#1{\bupgreek{\bm{#1}}}%

\global\long\def\T{\b t}%

\global\long\def\ep{\bm{\varepsilon}}%

\global\long\def\rep{\bm{\upvarepsilon}}%

\global\long\def\marker{\checkmark}%

\global\long\def\zo{\b z^{*}}%

\global\long\def\locscale{\mathrm{LocScale}}%

\global\long\def\z{\b z}%

\global\long\def\zr{\r z}%

\global\long\def\u{\b u}%

\global\long\def\ur{\r u}%

\global\long\def\w{\b w}%

\global\long\def\x{\b x}%

\global\long\def\qw{q_{\w}}%

\twocolumn[
\icmltitle{Provable Smoothness Guarantees for Black-Box Variational Inference}

\begin{icmlauthorlist}
\icmlauthor{Justin Domke}{umass}
\end{icmlauthorlist}

\icmlaffiliation{umass}{College of Computing and Information Sciences, University of Massachusetts, Amherst, USA}

\icmlcorrespondingauthor{Justin Domke}{domke@cs.umass.edu} 

\icmlkeywords{variational inference, markov chain monte carlo}

\vskip 0.3in
]

\printAffiliationsAndNotice{}
\begin{abstract}
Black-box variational inference tries to approximate a complex target
distribution through a gradient-based optimization of the parameters
of a simpler distribution. Provable convergence guarantees require
structural properties of the objective. This paper shows that for
location-scale family approximations, if the target is M-Lipschitz
smooth, then so is the ``energy'' part of the variational objective.
The key proof idea is to describe gradients in a certain inner-product
space, thus permitting the use of Bessel's inequality. This result
gives bounds on the location of the optimal parameters, and is a key
ingredient for convergence guarantees.
\end{abstract}

\section{Introduction}

Variational inference (VI) approximates a complex distribution with
a simpler one. Take a target distribution $p(\z,\b x)$ where $\b x$
is observed data and $\z$ are latent variables. Let $\qw\pp{\z}$
be a simpler distribution with parameters $\b w$. VI algorithms minimize
the (negative) ``evidence lower bound''

\begin{equation}
-\mathrm{ELBO}\pp{\w}=\underbrace{\E_{\zr\sim\qw}\bracs{-\log p\pp{\zr,\x}}}_{\text{Energy term}\ l\pp{\w}}+\underbrace{\E_{\zr\sim\qw}\bracs{\log\qw\pp{\zr}}}_{\text{Neg-Entropy term }h\pp{\w}},\label{eq:ELBO-def}
\end{equation}
equivalent to minimizing the KL-divergence from $\qw\pp{\z}$ to $p\pp{\z|\b x}$.

Traditionally, this was done with message-passing algorithms. This
requires that $q$ and $p$ be relatively simple, essentially so that
expectations of parts of $\log p$ can be computed with respect to
$q$ \citep{Ghahramani_2001_PropagationAlgorithmsVariational,Winn_2005_VariationalMessagePassing,Blei_2017_VariationalInferenceReview}.
Recent work \citep[e.g.][]{Salimans_2013_FixedFormVariationalPosterior,Wingate_2013_AutomatedVariationalInference,Ranganath_2014_BlackBoxVariational,Regier_2017_FastBlackboxVariational,Kucukelbir_2017_AutomaticDifferentiationVariational}
has focused on a ``black box'' model where the algorithm can only
evaluate $\log p\pp{\z,\b x}$ or its gradient $\nabla_{\z}\log p\pp{\z,\x}$
at chosen points $\z$. The key idea is that it is still possible
to create an unbiased estimator of the gradient of ELBO, and therefore
to optimize it through stochsatic gradient methods. This strategy
applies to a large range of distributions, and is widely used.

It is important to know when black-box inference will work. While
often empirically successful, black-box VI can and does fail to find
the optimum \citep{Yao_2018_YesDidIt,Regier_2017_FastBlackboxVariationala,Fan_2015_FastSecondOrderStochastic}.
Stochastic optimization convergence guarantees \citep{Bottou_2016_OptimizationMethodsLargeScale}
typically require two types of assumptions:
\begin{itemize}
\item \emph{Gradient variance} must be controlled. The variance of VI gradient
estimators has been studied \citep{Fan_2015_FastSecondOrderStochastic,Xu_2018_variancereductionproperties,Domke_2019_ProvableGradientVariancea},
leading to the result that if $\log p$ is smooth, then the variance
of reparameterization gradient estimators can be controlled. While
an important step, these results alone cannot fully explain convergence
behavior.
\item Structural properties of the \emph{objective itself} are needed. \ref{fig:motivation}
shows an example where \emph{exact} gradients are available. While
a careful step-size and initialization appear to lead to convergence,
other times there are worrying ``jumps''. Is any general guarantee
possible?
\end{itemize}
\begin{figure}[b]
\vspace{-10pt}\includegraphics[width=1\linewidth]{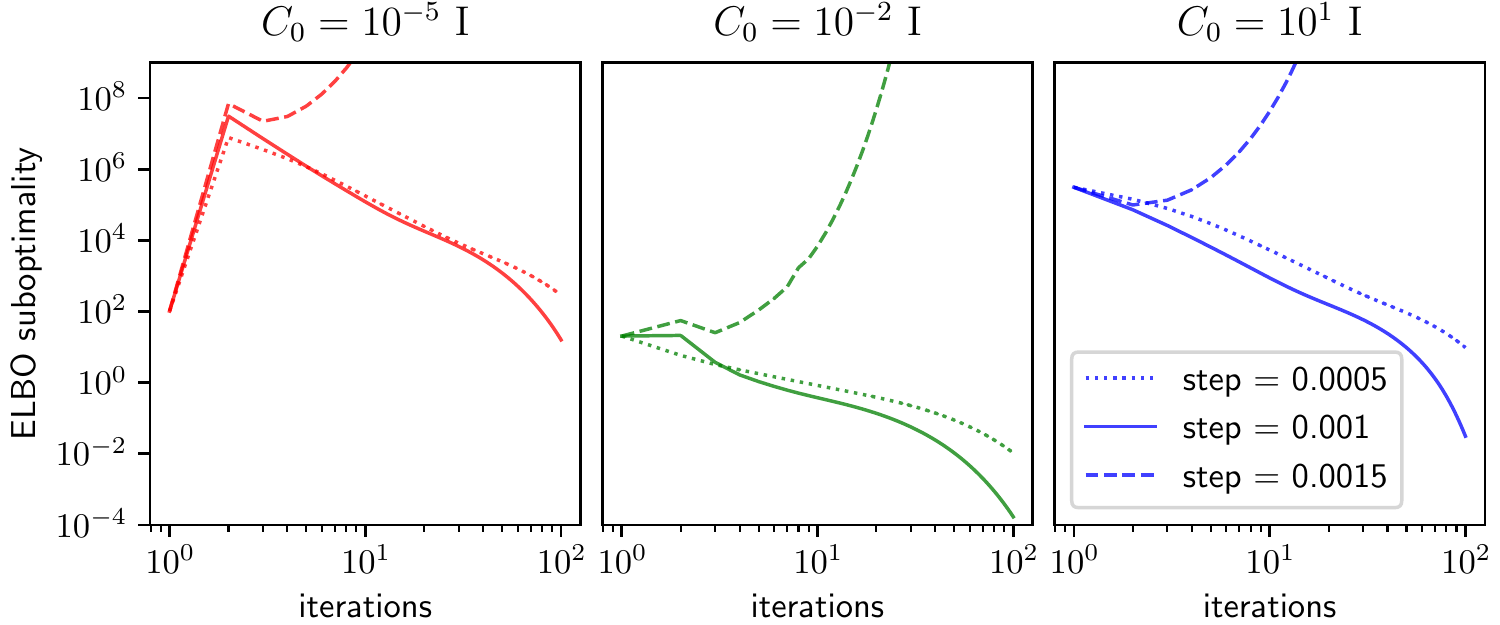}

\caption{Gradient descent on \texttt{fires} with various step-sizes, initialized
with $\protect\b m=0$ and various $C$. Results are sensitive to
both initialization and the stepsize. Can this be explained?\label{fig:motivation}}
\end{figure}
The ultimate goal of the line of research in this paper is to obtain
full convergence guarantees for practical black-box variational inference
algorithms. The results in this paper are a step towards that goal.

One very fundamental property is Lipschitz \emph{smoothness} which
means the gradient cannot change too quickly. Formally, a function
$f$ is $M$-smooth in the $l_{2}$ norm if $\VV{\nabla f\pp{\x}-\nabla f\pp{\b y}}_{2}\leq M\VV{\x-\b y}_{2}.$
For non-convex objectives, essentially all convergence guarantees
require smoothness, both for regular stochastic gradient descent (SGD)
\citep{Lee_2016_GradientDescentOnly,Ge_2015_EscapingSaddlePoints,Ghadimi_2013_StochasticFirstZerothOrder},
proximal SGD \citep{Ghadimi_2016_Minibatchstochasticapproximation},
or momentum or ``accelerated'' SGD \citep{Yang_2016_UnifiedConvergenceAnalysis}.
Convergence guarantees are possible for convex objectives with or
without smoothness\citep{Bottou_2016_OptimizationMethodsLargeScale,Rakhlin_2012_Makinggradientdescent}.

Because this property is so fundamental, several works on variational
inference have \emph{assumed} that the VI objective (or part of it)
is smooth. These include:
\begin{itemize}
\item \citet{Khan_2015_KullbackLeiblerproximalvariational} Sec. 4
\item \citet{Khan_2016_FasterStochasticVariational} Assumption A1
\item \citet{Regier_2017_FastBlackboxVariational} Condition 1
\item \citet{Fan_2016_TriplyStochasticVariational} Thm. 1
\item \citet{Buchholz_2018_QuasiMonteCarloVariational} Thm. 1
\item \citet{Mohamad_2018_AsynchronousStochasticVariational} Sec. 3.2
\item \citet{Alquier_2017_Concentrationtemperedposteriors} Assumption 3.2.
\end{itemize}
Yet, to the best of our knowledge, no rigorous guarantees that this
is actually true are known. The purpose of this paper is to fill that
theoretical gap by providing conditions under which smoothness provably
holds.

\subsection{Contributions}
\begin{description}
\item [{Smoothness~of~the~energy:}] \eqref{thm:lsmoothness} Our main
result is more general than variational inference: If $f\pp{\z}$
is $M$-smooth then $\E_{\zr\sim\qw}f(\zr)$ is $M$-smooth over $\w$,
when $\qw$ is in the location-scale family with a ``standardized''
base distribution. In particular, if $\log p(\z,\x)$ is $M$\emph{-}smooth
over $\z$ then the energy $l(\w)$ in \ref{eq:ELBO-def} is $M$-smooth.
This\emph{ }requires\emph{ }no convexity assumptions\emph{.}
\item [{Solution~guarantees:}] \eqref{thm:smoothness-sol} Intuitively,
structural properties of $\log p$ should imply properties of the
optimal parameters $\w$. Using the above smoothness result, we show
that at the optimal $\w$, all eigenvalues of the covariance of $q$
are at least $1/M$. This is important because the neg-entropy in
\ref{eq:ELBO-def} is non-smooth when eigenvalues are very small.
\item [{Convergence~considerations:}] Even if $\log p$ is smooth, the
ELBO is \emph{not} smooth, due to the entropy. We propose two solutions:
a \emph{projected} gradient descent scheme that leverages the above
solution guarantee and a \emph{proximal} scheme that uses the full
structure of the entropy.
\item [{Understanding~plain~gradient~descent:}] Given that the full
ELBO is non-smooth, why does \emph{plain} gradient descent sometimes
succeed, and sometimes -- even with exact gradients -- produce huge
``jumps'' as seen in \ref{fig:motivation}? We give insight into
this using the smoothness result.
\end{description}
As a minor contribution, we extend existing work \citep{Challis_2013_GaussianKullbackLeiblerApproximate,Titsias_2014_DoublyStochasticVariational}
to show that if $-\log p\pp{\z,\x}$ happens to be strongly-convex
over $\z$ then so is the energy term $l\pp{\w}$ \eqref{thm:convexity}.
This gives another parameter-space solution guarantee where essentially
the covariance of $q$ cannot be too large \eqref{thm:strongsol}.

\section{Preliminaries\label{sec:preliminaries}}

A multivariate \textbf{location-scale family} \citep{Geyer_2011_Statistics5101Lecture}
is the result of drawing a sample from a ``base'' distribution and
applying an affine transformation to it. Formally,

\begin{equation}
\zr\sim\locscale(\b m,C,s)\Longleftrightarrow\zr\overset{d}{=}C\ur+\b m,\ \ur\sim s,
\end{equation}
where $\overset{d}{=}$ indicates equality in distribution.

Black-box VI using these families was studied by \citet{Titsias_2014_DoublyStochasticVariational}.
A simple example is the multivariate Gaussian, for which $\locscale\pp{\b m,C,\N(0,I)}=\N(\b m,CC^{\top}).$
Many families are representable, e.g. elliptical distributions such
as the multivariate Student-T or Cauchy distributions. More generally,
the base distribution need not be symmetric.

\textbf{Notation.} Let $\b w=\pp{\b m,C}$ be a vector containing
all components of $\b m$ and $C$. We write $\qw$ to denote $\locscale\pp{\b m,C,s}$,
leaving $s$ implicit. Proofs use $\T_{\b w}\pp{\u}=C\u+\b m$ to
denote the affine mapping determined by parameters $\b w$. $A\preceq B$
means that $B-A$ is positive semidefinite. We assume $\z\in\R^{d}$.
Sans-serif font ($\ur,\zr$) distinguishes random variables.

\textbf{Density.} If the base distribution has a density and $C$
is invertible, then the location-scale distribution also has a density,
which is $\qw\pp{\z}=\locscale\pp{\z\vert\b m,C,s}=\frac{1}{\verts C}s\pars{C^{-1}\pars{\z-\b m}}.$

\textbf{Entropy.} The entropy of random variables under affine transformations
is $\mathrm{Entropy}\bb{A\ur+\b b}=\mathrm{Entropy}\bb{\ur}+\log\verts{\det A}$
\citep[Sec. 8.6]{Cover_2006_Elementsinformationtheory}. Thus, the
neg-entropy is  $h\pp{\w}=-\mathrm{Entropy}\bb s-\log\verts{\det C},$
with gradient $\nabla h\pp{\w}=\pars{0,-C^{-\top}}$.

\textbf{Standardized Representations.} We say that $s$ is ``standardized''
if it has mean zero and unit variance, i.e. $\E_{\ur\sim s}\ur=0$
and $\V_{\ur\sim s}\ur=I$. When $s$ is standardized, the mean of
the location-scale distribution is $\b m$ while the variance is $CC^{\top}.$
A result we will use in \ref{sec:Solution-Guarantees} is that any
location-scale family can be represented using a standardized base
distribution, provided the variance exists: If $s$ has mean $\mu$
and variance $\Sigma$, then $s'=\locscale\pp{-\Sigma^{-1/2}\mu,\Sigma^{-1/2},s}$,
is standardized, yet $\locscale\pp{\b m,C,s'}$ and $\locscale\pp{\b m,C,s}$
index the same set of distributions.

\textbf{Bessel's inequality} states that if $\left\{ \b a_{1},\cdots,\b a_{k}\right\} $
are orthonormal in some inner-product $\angs{\cdot,\cdot}$ with corresponding
norm $\Verts{\cdot},$ then
\begin{equation}
\sum_{i=1}^{k}\verts{\angs{\b a_{i},\b x}}^{2}\leq\Verts{\b x}^{2}.
\end{equation}
This can be seen as a generalization of the Cauchy-Schwarz inequality
$\verts{\angs{\b y,\b x}}^{2}\leq\Verts{\b y}^{2}\Verts{\b x}^{2}$,
which follows from using the singleton set $\left\{ \b a_{1}\right\} =\left\{ \b y/\Verts{\b y}\right\} $
\citep{Rooin_2012_EquivalencyCauchySchwarzBessel,Hasegawa__GeneralizationsCauchySchwarzProbability}.

\subsection{Convergence Guarantees and Smoothness}

It is impossible to review the vast optimization literature relevant
to solving \ref{eq:ELBO-def}: There are many algorithms (gradient
descent, stochastic gradient descent, momentum or accelerated variants,
proximal or mirror descent variants) that can be analyzed with different
hyper-parameters (step-sizes, iterate averaging) yielding different
types of guarantees. These guarantees depend on properties of the
target objective being optimized. Different distributions $p\pp{\x,\z}$
different properties for the objective (smoothness, convexity, strong
convexity) or gradient estimators of it (variance bounds).

Still, at a very high level, the story is simple. To the best of our
knowledge, all existing convergence guarantees require \emph{either}
smoothness (to guarantee a stationary point) \emph{or} convexity (to
guarantee a global optima), or both. For concreteness, suppose $f$
is \emph{only} known to be smooth (and possibly non-convex): \citet{Ghadimi_2013_StochasticFirstZerothOrder}
analyze the the iteration $\r w_{n+1}=\r w_{n}-\gamma\r g_{n}$ where
$\r g_{1},\cdots\r g_{N}$ are independent and $\E\r g_{n}=\nabla f\pp{\r w_{n}}$.
A simplified statement of their result is as follows: If (i) $f$
is $M$-smooth, (ii) $\mathbb{E}\VV{\r g_{n}-\nabla f\left(\r w_{n}\right)}_{2}^{2}\leq\sigma^{2},$
and (iii) the starting point $\r w_{1}$ obeys $\VV{\r w_{1}-\w^{*}}_{2}\leq D$,
then with a step-size of $\gamma=\min\pp{1/M,D/\pp{\sigma\sqrt{N}}}$,
\begin{equation}
\frac{1}{N}\sum_{n=1}^{N}\E\VV{\nabla f\pp{\r w_{n}}}^{2}\leq\frac{M^{2}D^{2}}{N}+\frac{2DM\sigma}{\sqrt{N}},
\end{equation}
where the expectation is over all possible executation traces $\r w_{1},\cdots,\r w_{N}.$

Both the step-size and convergence rate depend on smoothness. If there
is noise ($\sigma>0$) the convergence is $1/\sqrt{N}.$ With no noise
($\sigma=0$) convergence is $1/N$. Similar rates are known for proximal
or projected stochastic gradient descent \citep{Ghadimi_2016_Minibatchstochasticapproximation}
and stochastic gradient descent with momentum or Nesterov acceleration
\citep{Yang_2016_UnifiedConvergenceAnalysis}.  Recent work seeks
to understand when these iterations will converge to a (local) minima
instead of a saddle point \citep{Ge_2015_EscapingSaddlePoints,Lee_2016_GradientDescentOnly};
here too, smoothness is a key assumption.

Similar guarantees are possible if the objective is convex or strongly
convex, without requiring smoothness \citep[Section 6.2]{Rakhlin_2012_Makinggradientdescent,Bottou_2016_OptimizationMethodsLargeScale,Bubeck_2015_ConvexOptimizationAlgorithms}.
When gradients are stochastic, it may be helpful to average ``minibatches''
of gradient estimates, both in the convex \citep[Section 6.2]{Bubeck_2015_ConvexOptimizationAlgorithms}
and non-convex cases \citep{Ghadimi_2016_Minibatchstochasticapproximation}.

\section{Smoothness of the energy}

In this section, we set $l\pp{\w}=\E_{\zr\sim\qw}f\pp{\zr}$. The
energy $l\pp{\w}$ in \ref{eq:ELBO-def} is recovered when $f\pp{\z}=-\log p\pp{\z,\x}$
(since $\x$ is constant). This is done to simplify the notation and
because the results apply to general $f$ and might be of independent
interest.

\subsection{Main Result}

The following is the main technical result of this paper. It states
that if $\qw$ is a location-scale family with a zero-mean, unit variance
base distribution and $f(\b z)$ is $M$-smooth, then $l(\b w)$ is
also $M$-smooth. We emphasize that $f$ is \emph{not} assumed to
be convex.

\begin{restatable}{thm1}{lsmoothness}Let $\qw=\locscale(\b m,C,s)$
with parameters $\w=\pp{\b m,C}$ and a standardized base distribution
$s$. If $f(\b z)$ is $M$-smooth, then $l(\b w)=\E_{\zr\sim\qw}f\pp{\zr}$
is also $M$-smooth.\label{thm:lsmoothness}

\end{restatable}

Before proving this, we give four technical lemmas, all proven in
\ref{sec:Technical-Lemmas} (in the supplement). The idea is to define
a certain inner-product $\langle\cdot,\cdot\rangle_{s}$ over functions
and a set of orthonormal functions $\left\{ \b a_{i}\right\} $ such
that derivatives $l\pp{\w}$ can be written as an inner-product of
$\b a_{i}$ and $\nabla f\circ\T_{w}$ in $\langle\cdot,\cdot\rangle_{s},$
where $\circ$ indicates composition of functions.

\begin{restatable}{lem1}{validinnerproduct}

$\angs{\b a,\b b}_{s}=\E_{\ur\sim s}\b a(\ur)^{\top}\b b(\ur)$ is
a valid inner-product on squared-integrable $\b a:\R^{d}\rightarrow\R^{k}$.\label{lem:valid-inner-product}

\end{restatable}

The proof consists of verifying each of the defining properties of
an inner-product.

\begin{restatable}{lem1}{gradasinnerproduct}

Let $\b a_{i}\pp{\u}=\frac{d}{dw_{i}}\b t_{\b w}\pp{\u}$. This is
independent of $\w$ and $\frac{dl\pp{\w}}{dw_{i}}=\angs{\b a_{i},\nabla f\circ\T_{\b w}}_{s}$.
\label{lem:grad-as-inner-product}

\end{restatable}

This is proven by first verifying that both $\frac{d}{dC_{ij}}\T_{w}\pp{\u}$
and $\frac{d}{dm_{i}}\T_{w}\pp{\u}$ are independent of $\w$, then
calculating $\frac{dl}{dw_{i}}$ and performing some manipulations.

\begin{restatable}{lem1}{orthonormality}

If $s$ is standardized, then the functions $\left\{ \b a_{i}\right\} $
are orthonormal in $\angs{\cdot,\cdot}_{s}.$\label{lem:orthonormality}

\end{restatable}

To prove this, note that the components $\b a_{i}$ have two ``types''
namely $\frac{d}{dC_{ij}}\T_{w}\pp{\u}$ and $\frac{d}{dm_{i}}\T_{w}\pp{\u}$.
Thus, an inner-product $\angs{\b a_{i},\b a_{j}}_{s}$ reduces to
the expected inner product of two such terms. It can be shown that
this inner-product is one when $\b a_{i}$ and $\b a_{j}$ have the
same type and indices, and zero otherwise.

\begin{restatable}{lem1}{expecteddifferenceoftransforms}

If $s$ is standardized, then $\E_{\ur\sim s}\Verts{\T_{\b w}\pp{\ur}-\T_{\b v}\pp{\ur}}_{2}^{2}=\Verts{\b w-\b v}_{2}^{2}.$\label{lem:expected-difference-of-transforms}

\end{restatable}

This is shown by substituting the exact form of $\T_{\b w}$ and $\T_{\b v}$.
Taking the expectation leads a result of $\Verts{\Delta C}_{F}^{2}+\Verts{\Delta\b m}_{2}^{2},$
where $\Delta\b m$ denotes the difference of the $\b m$ components
of $\b w$ and $\b v$, and similarly for $\Delta C$. This is equivalent
to the squared Euclidean distance of $\b w$ and $\b v$.
\begin{proof}[Proof of \ref{thm:lsmoothness}]
Take two parameter vectors, $\w$ and $\b v$. Apply \ref{lem:grad-as-inner-product}
to each component of the gradients $\nabla l\pp{\w}$ and $\nabla l\pp{\b v}$
to get that

\begin{alignat*}{1}
 & \Verts{\nabla l(\b w)-\nabla l(\b v)}_{2}^{2}\\
 & =\sum_{i}\pars{\angs{\b a_{i},\nabla f\circ\T_{\b w}}_{s}-\angs{\b a_{i},\nabla f\circ\T_{\b v}}_{s}}^{2}\\
 & =\sum_{i}\angs{\b a_{i},\nabla f\circ\T_{\b w}-\nabla f\circ\T_{\b v}}_{s}^{2}.
\end{alignat*}
\ref{lem:orthonormality} showed that the functions $\left\{ \b a_{i}\right\} $
are orthonormal in the inner-product $\angs{\cdot,\cdot}_{s}$. Thus,
by Bessel's inequality,
\begin{align}
\Verts{\nabla l(\b w)-\nabla l(\b v)}_{2}^{2} & \leq\Verts{\nabla f\circ\T_{\b w}-\nabla f\circ\T_{\b v}}_{s}^{2},\label{eq:applybessel}\\
 & =\E_{\ur\sim s}\Verts{\nabla f\pars{\T_{\b w}\pp{\ur}}-\nabla f\pars{\T_{\b v}\pp{\ur}}}_{2}^{2}\nonumber 
\end{align}
where $\VV{\cdot}_{s}$ denotes the norm corresponding to $\angs{\cdot,\cdot}_{s}.$
Now apply the smoothness of $f$ to get that

\begin{alignat}{1}
\Verts{\nabla l(\b w)-\nabla l(\b v)}_{2}^{2} & \leq M^{2}\E_{\ur\sim s}\Verts{\T_{\b w}\pp{\ur}-\T_{\b v}\pp{\ur}}_{2}^{2}\label{eq:apply-smoothness}\\
 & =M^{2}\Verts{\w-\b v}_{2}^{2},
\end{alignat}
where the last equality follows from \ref{lem:expected-difference-of-transforms}.
\end{proof}
Note that the \emph{only} inequalities used in this proof are (i)
Bessel's inequality and (ii) the bound on the difference of gradients
of $f$ provided by the assumption that $f$ is $M$-smooth. Thus,
the tightness of the final bound that $\Verts{\nabla l(\b w)-\nabla l(\b v)}_{2}\leq M\Verts{\b w-\b v}_{2}$
is determined by the tightness of these two inequalities. It's natural
to ask when this bound will be tight or loose. The following section
will show that it is tight when $f$ is closer to an isotropic quadratic.
On the other hand, the starting assumption that $f$ is smooth $\Verts{\nabla f(\x)-\nabla l(\b y)}_{2}\leq M\Verts{\x-\b y}_{2}$
might often be loose, e.g. if $f$ is much smoother in ``some directions''
than others. In this case, the final bound on $l$ will also be loose
due to looseness created in moving from \ref{eq:applybessel} to \ref{eq:apply-smoothness}
(even $M$ might still be the best possible \emph{smoothness constant}).

The idea of expressing the gradient using a fixed base distribution
and a transformation $\T_{w}\pp{\u}$ is also used in ``reparameterization''
type estimators \citep{Titsias_2014_DoublyStochasticVariational,Rezende_2014_StochasticBackpropagationApproximate,Kingma_2014_Autoencodingvariationalbayes}.
Smoothness, however, is a deterministic property of the function $l(\b w)$,
independent of any method one might use for estimating or optimizing
it.

\subsection{Unimprovability}

This section gives an example function $f\pp{\z}$ that is $M$-smooth,
but leads to a function $l\pp{\w}$ that is $M$-smooth (but not smoother),
meaning that \ref{thm:lsmoothness} is unimprovable. Intuitively,
smoothness is a quadratic upper-bound. So, it is natural to suppose
that $f\pp{\z}$ is \emph{exactly} quadratic. The following shows
that in this case, $l\pp{\b w}$ has a closed form.

\begin{restatable}{thm1}{egfungeneral}

\label{thm:egfun-general}Let $\qw=\locscale(\b m,C,s)$ with parameters
$\w=\pp{\b m,C}$ and a standardized base distribution $s$ and let
$f(\b z)=\frac{a}{2}\Verts{\b z-\zo}_{2}^{2}.$ Then $l\pp{\b w}=\E_{\zr\sim\qw}f\pp{\zr}=\frac{a}{2}\pp{\VV{\b m-\zo}_{2}^{2}+\VV C_{F}^{2}}.$

\end{restatable}

To see that \ref{thm:lsmoothness} is unimprovable, define $\bar{\b w}=\pp{\b z^{*},0_{d,d}},$
where $0_{d,d}$ is a $d\times d$ matrix of zeros. Then, $l\pp{\w}=\frac{a}{2}\VV{\w-\bar{\w}}_{2}^{2}$
. This is $M$-smooth for $M=a,$ but not for any smaller value.

\subsection{Solution Guarantees\label{sec:Solution-Guarantees}}

Intuitively, properties of the target distribution $p\pp{\z|\x}$
might imply properties of the variational distribution $q_{w^{*}}$
at the optimal parameters $\w^{*}$. In particular, if $\log p\pp{\z,\x}$
is smooth over $\z$, then it is ``spread out'' so we might expect
that $q_{\w^{*}}$ would also be. This section formalizes and proves
a version of this intuition. This will be used in \ref{sec:Convergence-Considerations}
to give a convergence guarantee for projected stochastic gradient
descent. The core idea is that the ELBO in \ref{eq:ELBO-def} is poorly
conditioned for low-variance distributions. However, if we can guarantee
that the optimum lies in a well-conditioned region, we can constrain
optimization to that region.

We define $\mathcal{W}_{M}$ to be the set of parameters where all
singular values of $C$ are at least $1/\sqrt{M},$ i.e.
\begin{equation}
\mathcal{W}_{M}=\left\{ \pp{\b m,C}\ \Big\vert\ \sigma_{\min}\pp C\geq\frac{1}{\sqrt{M}}\right\} .
\end{equation}
We could equivalently define $\mathcal{W}_{M}$ to be the set of of
parameters where all eigenvalues of $CC^{\top}$ are at least $\frac{1}{M}$,
i.e. $CC^{\top}\succeq\frac{1}{M}I$. Recall from \ref{sec:preliminaries}
that for standardized $s$, $\V_{\zr\sim\qw}\zr=CC^{\top}$, so this
is the parameters with variance at least $\frac{1}{M}I.$

The following result shows that if minimizing the ELBO with a smooth
target distribution, the optimal parameters must fall in $\mathcal{W}_{M},$
i.e. the variance cannot be smaller than $\frac{1}{M}I$. This requires
the stronger assumption that the base distribution $s$ is spherically
symmetric.

\begin{restatable}{thm1}{smoothnesssol}

Let $\qw=\locscale(\b m,C,s)$ with parameters $\w=\pp{\b m,C}$ and
a standardized and spherically symmetric base distribution $s$. Suppose
$\w$ minimizes $l\pp{\w}+h\pp{\w}$ from \ref{eq:ELBO-def} and $\log p\pp{\z,\x}$
is $M$-smooth over $\z$. Then, $\w\in\mathcal{W}_{M}.$\label{thm:smoothness-sol}

\end{restatable}

The proof of this theorem (in \ref{sec:Solution-Guarantees-appendix})
first establishes the following Lemma.

\begin{restatable}{lem1}{diagCgradbound}

Let $\qw=\locscale(\b m,C,s)$ with parameters $\w=\pp{\b m,C}$ and
a standardized and spherically symmetric base distribution $s$. Let
$l(\b w)=\E_{\zr\sim\qw}f\pp{\zr}$. Suppose $C$ is diagonal and
$f$ is $M$-smooth. Then, $\vv{\frac{dl\pp{\w}}{dC_{ii}}}\leq M\vv{C_{ii}}.$\label{lem:diag-C-grad-bound}

\end{restatable}

The proof of \ref{lem:diag-C-grad-bound} first shows that if $C_{ii}=0$,
then $\frac{dl}{dC_{ii}}=0,$ which uses that $s$ is symmetric. Then,
given an arbitrary $\w$, let $\w'$ be $\w$ with $C_{ii}$ set to
zero. Since we know from \ref{thm:egfun-general} that $l$ is $M$-smooth,
we then get that $\vv{dl\pp{\w}/dC_{ii}}\leq\VV{\nabla l\pp{\w'}-\nabla l\pp{\w}}_{2}\leq M\vv{C_{ii}}.$

Now, the proof of \ref{thm:smoothness-sol} uses the fact that if
$\w$ is a minimum, then $\nabla l\pp{\w}=-\nabla h\pp{\w}.$ If $C$
happens to be diagonal, the result is easy to show using the previous
lemma along with the exact gradient of $h$. Given an arbitrary $C$,
we can use the singular value decomposition of $C$ to define another
$M$-smooth function which must have a diagonal solution.

\section{Analogous Result for Convex Functions}

Smoothness and strong convexity are complementary in that they give
upper and lower bounds on the eigenvalues of the Hessian. As a minor
contribution, we observe that a guarantee complementary to \ref{thm:lsmoothness}
holds: if $-\log p$ is (strongly) convex, then so is $l\pp{\w}$.
The example in \ref{thm:egfun-general} shows this result is also
unimprovable.

\begin{restatable}{thm1}{convexity}

Let $\qw=\locscale(\b m,C,s)$ with parameters $\w=\pp{\b m,C}$.
If $f\pp{\z}$ is convex, then $l(\b w)=\E_{\zr\sim\qw}f\pp{\zr}$
is also convex. If, in addition, $s$ is standardized and $f\pp{\z}$
is $c$-strongly convex, then $l\pp{\w}$ is also $c$-strongly convex.\label{thm:convexity}

\end{restatable}
\begin{proof}
\textbf{(Convexity)} Represent $l$ as $l(\b w)=\E_{\ur\sim s}f\pars{\T_{\b w}\pp{\ur}}$
where $\T_{\b w}\pp{\u}=C\u+\b m$. For fixed $\u$, $\T_{\b w}\pp{\u}$
is \emph{linear} in $\b w.$ Thus, given any two parameter vectors
$\b w$ and $\b v$ and any $\alpha,\beta\in(0,1)$ with $\alpha+\beta=1,$
since $f$ is convex, $l\pp{\alpha\w+\beta\b v}$ is equal to
\begin{alignat*}{1}
\E_{\ur\sim s}f\pars{\T_{\alpha\b w+\beta\b v}\pp{\ur}} & =\E_{\ur\sim s}f\pars{\alpha\T_{\b w}\pp{\ur}+\beta\T_{\b v}\pp{\ur}}\\
 & \leq\E_{\ur\sim s}\alpha f\pars{\T_{\b w}\pp{\ur}}+\beta f\pars{\T_{\b v}\pp{\ur}}\\
 & =\alpha l(\b w)+\beta l(\b v).
\end{alignat*}

\textbf{(Strong convexity)} If $f$ is $c$-strongly convex then $f\pp{\b z}=f_{0}(\b z)+\frac{c}{2}\VV{\z}_{2}^{2}$
for some convex function $f_{0}$. Thus, $l(\b w)=l_{0}(\b w)+\frac{c}{2}\E_{\zr\sim\qw}\VV{\zr}_{2}^{2},$
where $l_{0}(\b w)=\E_{\zr\sim\qw}f_{0}(\zr)$ is convex by the previous
reasoning. Then, it isn't too hard to show that $\E_{\zr\sim\qw}\VV{\zr}_{2}^{2}=\E_{\ur\sim s}\VV{C\ur+\b m}_{2}^{2}=\VV C_{F}^{2}+\VV{\b m}_{2}^{2}=\VV{\w}_{2}^{2}.$
Thus, we have that $l(\b w)=l_{0}(\b w)+\frac{c}{2}\Verts{\b w}^{2}$
is $c$-strongly convex.
\end{proof}
The convexity result (and proof) is essentially the same as that of
\citet[Appendix, Proposition 1]{Titsias_2014_DoublyStochasticVariational}.
The strong-convexity result generalizes a previous result due to \citet[Sec. 3.2]{Challis_2013_GaussianKullbackLeiblerApproximate}
who give a strong-convexity guarantee for Gaussian variational distributions
applied to targets with Gaussian priors.

The following result gives a bound on the location of the optimal
parameters. The proof uses the fact that, at the optimum, $\nabla l\pp{\w}=-\nabla h\pp{\w},$
so the exact gradient is known. However, strong convexity means that
only certain gradients are possible at a given part of parameter space.
(This result is complementary to \ref{thm:smoothness-sol}.)

\begin{restatable}{thm1}{vilocconvex}

Let $\qw=\locscale(\b m,C,s)$ with parameters $\w=\pp{\b m,C}$ and
a standardized and spherically symmetric base distribution $s$. Suppose
$\w$ minimizes $l\pp{\w}+h\pp{\w}$ from \ref{eq:ELBO-def} and $-\log p\pp{\z,\x}$
is $c$-strongly convex over $\z$. Then, $\VV C_{F}^{2}+\VV{\b m-\zo}_{2}^{2}\leq\frac{d}{c},$
where $\zo=\argmax_{\b z}\log\pp{\b z,\b x}$.\label{thm:strongsol}

\end{restatable}

\section{Convergence Considerations\label{sec:Convergence-Considerations}}

In optimizing the ELBO in \ref{eq:ELBO-def}, the negative entropy
term $h$ creates complications. The gradient is $\nabla h\pp{\w}=(0,-C^{-\top})$
(\ref{sec:preliminaries}), which can change arbitrarily rapidly when
the singular values of $C$ are close to zero. So $h\pp{\w}$ is \emph{not}
Lipschitz-smooth, posing a challenge for establishing convergence
guarantees for pure gradient descent applied to the full ELBO. In
this section, we consider two strategies for coping with this: \emph{Projected}
gradient descent, and \emph{proximal} gradient descent. Finally, we
seek to understand the performance of \emph{regular} gradient descent
seen in \ref{fig:motivation}. 

The following result gives one way of dealing with the fact that the
negentropy is non-smooth. As in \ref{sec:Solution-Guarantees}, this
requires the additional assumption that the base distribution is spherically
symmetric.

\begin{restatable}{thm1}{projectedsummary}\label{thm:projectedsummary}

Let $\qw=\locscale(\b m,C,s)$ with parameters $\w=\pp{\b m,C}$ and
a standardized and spherically symmetric base distribution $s$. Suppose
$\log p\pp{\z,\x}$ is $M$-smooth. Then $\mathrm{ELBO}\pp{\w}$ as
in \ref{eq:ELBO-def} is $2M$ smooth over $\mathcal{W}_{M}$ and
if $\w^{*}$ is an optima of $\mathrm{ELBO}\pp{\w},$ then $\w^{*}\in\mathcal{W}_{M}$.

\end{restatable}

The proof (in \ref{sec:Convergence-Proofs}) first shows that $h$
is $M$-smooth over $\mathcal{W}_{M}$ by taking two arbitrary parameter
vectors $\w,\b v\in\mathcal{W}_{M}$ and using a matrix norm inequality
to bound the difference of the gradients $\nabla h\pp{\w}$ and $\nabla h\pp{\b v}$.
Then, we combine our main result that $l$ is smooth (\ref{thm:lsmoothness})
with the bound on the location of the optimum (\ref{thm:smoothness-sol})
and the fact that $h$ is smooth over $\mathcal{W}_{M}$. (By the
triangle inequality, the sum of two $M$-smooth functions is $2M$
smooth.)

Given this result, a natural approach to optimizing the ELBO is to
use projected (stochastic) gradient descent, i.e. to iterate $\r w'=\mathrm{proj}_{\mathcal{W}_{M}}\pp{\r w-\gamma\r g}$
where $\r g=\nabla l\pp{\r w}+\nabla h\pp{\r w}$ (or a stochastic
estimator) and $\mathrm{proj}_{\mathcal{W}}$ is Euclidean projection.
\ref{thm:gaussprox} (in \ref{sec:Convergence-Proofs}, supplement)
shows that if $\w=\pp{\b m,C}$, and $C$ has singular value decomposition
$C=USV^{\top},$ then
\[
\mathrm{proj}_{\mathcal{W}_{M}}\pp{\w}=\argmin_{\b v\in\mathcal{W}_{M}}\VV{\b w-\b v}_{2}^{2}=\pp{\b m,UTV^{\top}}
\]
where $T$ is a diagonal matrix with $T_{ii}=\max\pp{S_{ii},1/\sqrt{M}}.$

Another way of dealing the fact that $h$ is non-smooth is to use
\emph{proximal} optimization \citep{Beck_2009_Gradientbasedalgorithmsapplications,Parikh_2014_ProximalAlgorithms,Bubeck_2015_ConvexOptimizationAlgorithms,Ghadimi_2016_Minibatchstochasticapproximation,Ghadimi_2012_OptimalStochasticApproximation}.
Intuitively, the idea is as follows: With a step-size $\gamma$ gradient
descent on $l+h$ gives the update $\r w'=\r w-\gamma\pp{\nabla l(\r w)+\nabla h(\r w)}$,
which can equivalently be seen as minimizing a linear approximation
of $l+h$ at $\b w$, with a quadratic penalty, i.e. setting

\begin{multline}
\r w'=\argmin_{\b v}l(\r w)+h(\r w)+\angs{\nabla l(\r w)+\nabla h(\r w),\b v-\r w}\\
+\frac{1}{2\gamma}\Verts{\b v-\r w}^{2},\label{eq:minimize-approx-1}
\end{multline}

If $h(\b w)$ is non-smooth, even if $\b v$ is close to $\r w$,
$h(\r w)+\angs{\nabla h(\r w),\b v-\r w}$ can be an arbitrarily poor
approximation of $h(\b v).$ Thus, a natural idea is to leave $h$
\emph{unapproximated}, i.e. to linearize $l$ only. This would mean
instead using

\begin{multline}
\r w'=\argmin_{\b v}l(\r w)+\angs{\nabla l(\r w),\b v-\r w}+h(\b v)\\
+\frac{1}{2\gamma}\Verts{\b v-\r w}^{2}.\label{eq:prox}
\end{multline}

This is equivalent to 
\[
\r w'=\prox_{\gamma}\bb{\r w-\gamma\nabla l(\r w)},
\]

where 
\[
\prox_{\gamma}\bb{\b w}=\argmin_{\b v}h(\b v)+\frac{1}{2\gamma}\VV{\b v-\b w}_{2}^{2}.
\]

\begin{figure*}
\begin{centering}
\includegraphics[viewport=0bp 0bp 293.625bp 251.8131bp,clip,scale=0.6]{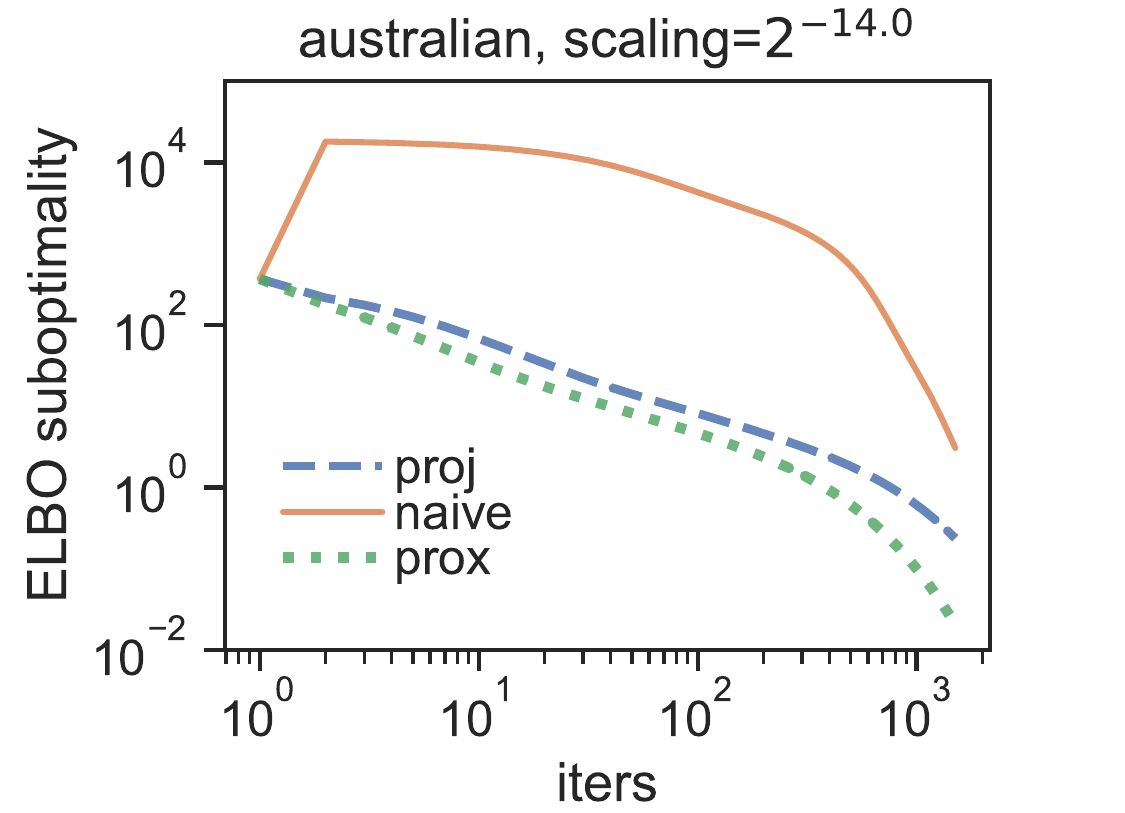}\includegraphics[viewport=60.75bp 0bp 293.625bp 251.8131bp,clip,scale=0.6]{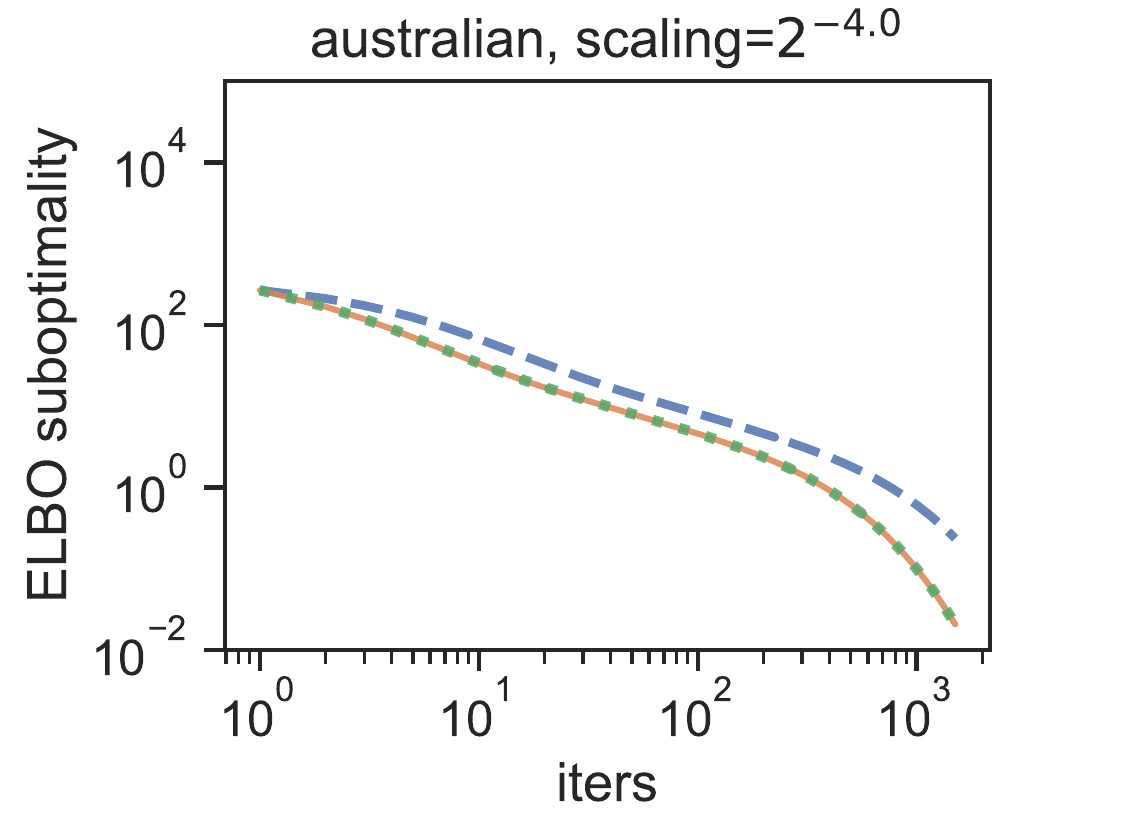}\includegraphics[viewport=60.75bp 0bp 293.625bp 251.8131bp,clip,scale=0.6]{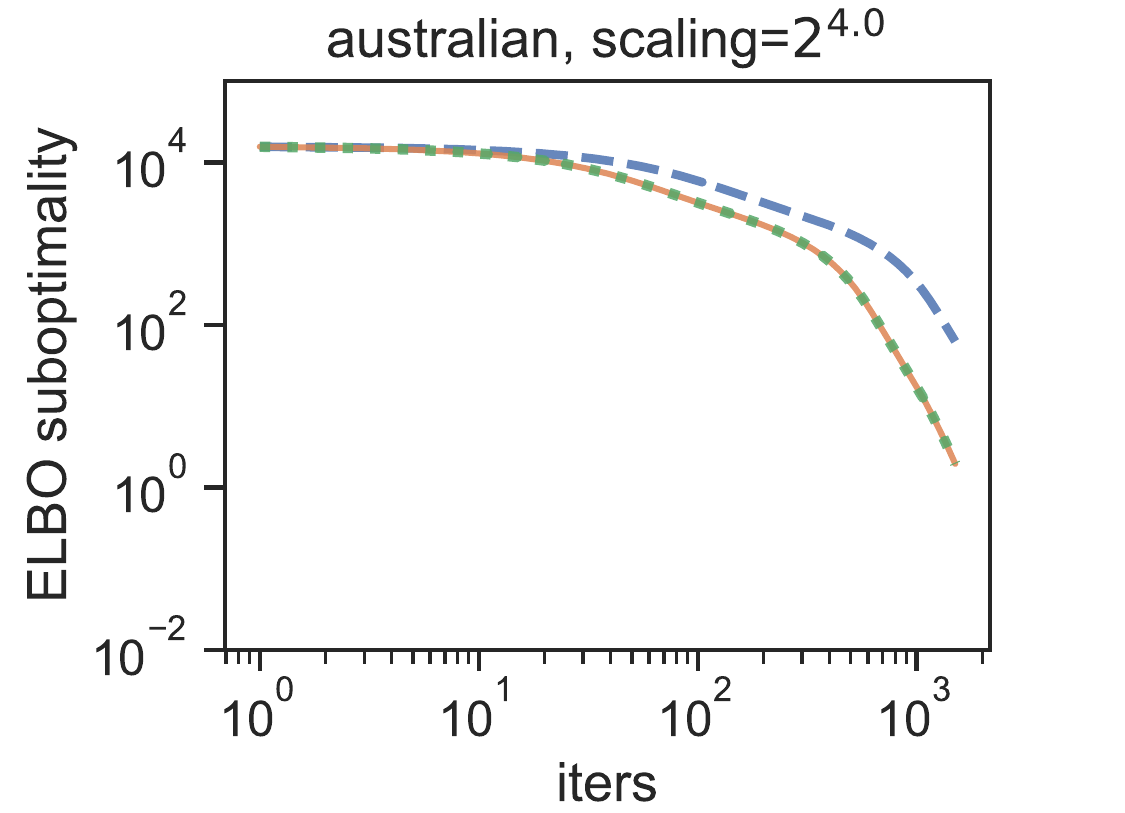}\vspace{-.25cm}
\par\end{centering}
\begin{centering}
\includegraphics[viewport=0bp 0bp 293.625bp 251.8131bp,clip,scale=0.6]{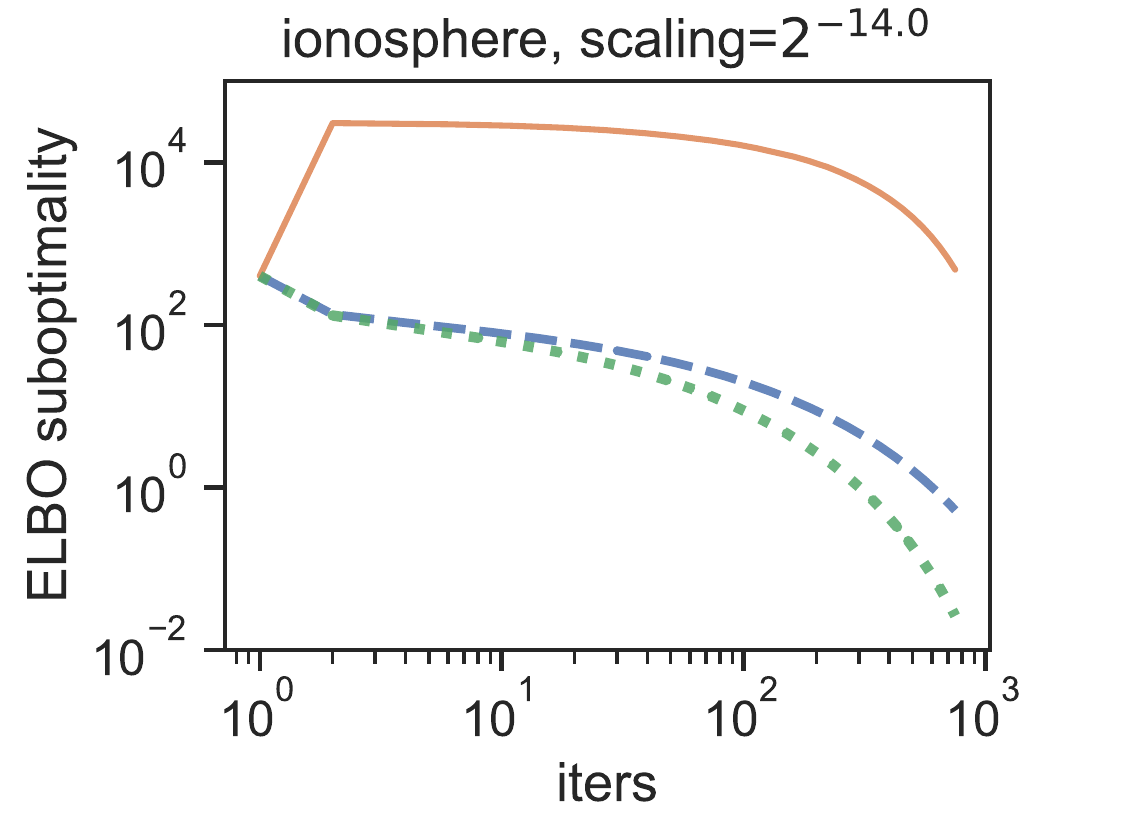}\includegraphics[viewport=60.75bp 0bp 293.625bp 251.8131bp,clip,scale=0.6]{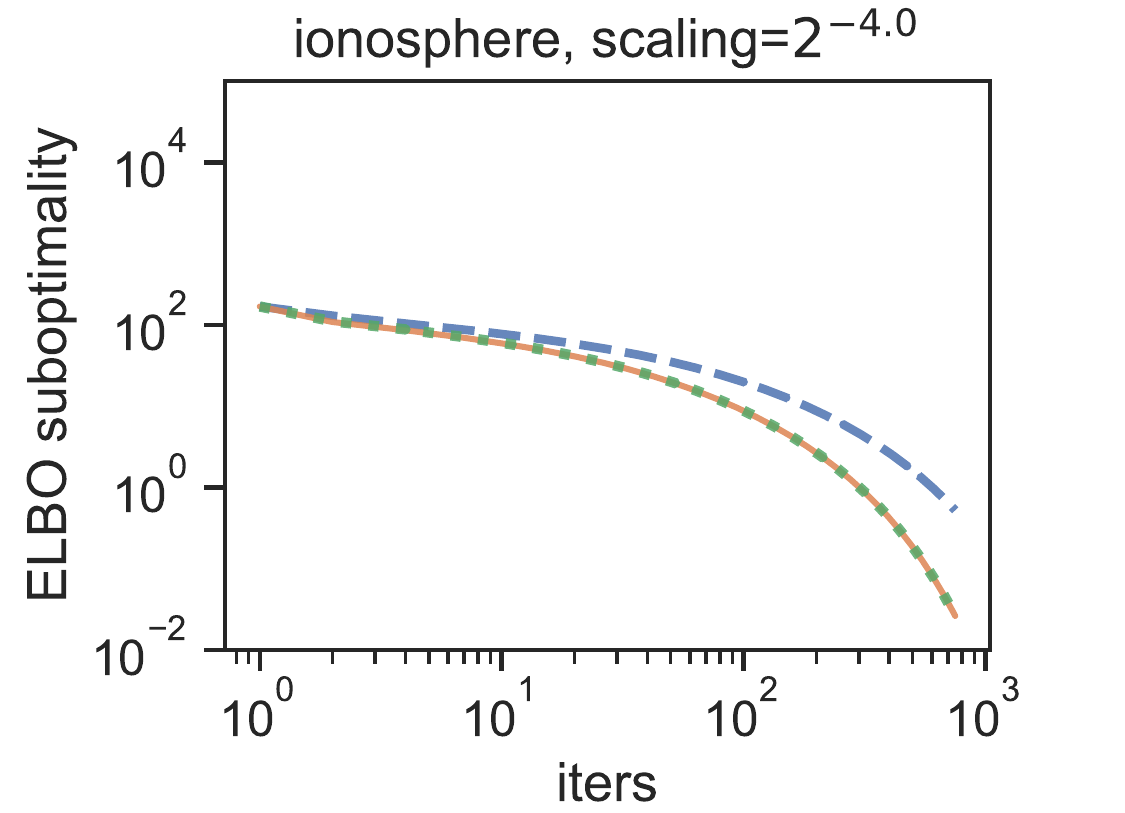}\includegraphics[viewport=60.75bp 0bp 293.625bp 251.8131bp,clip,scale=0.6]{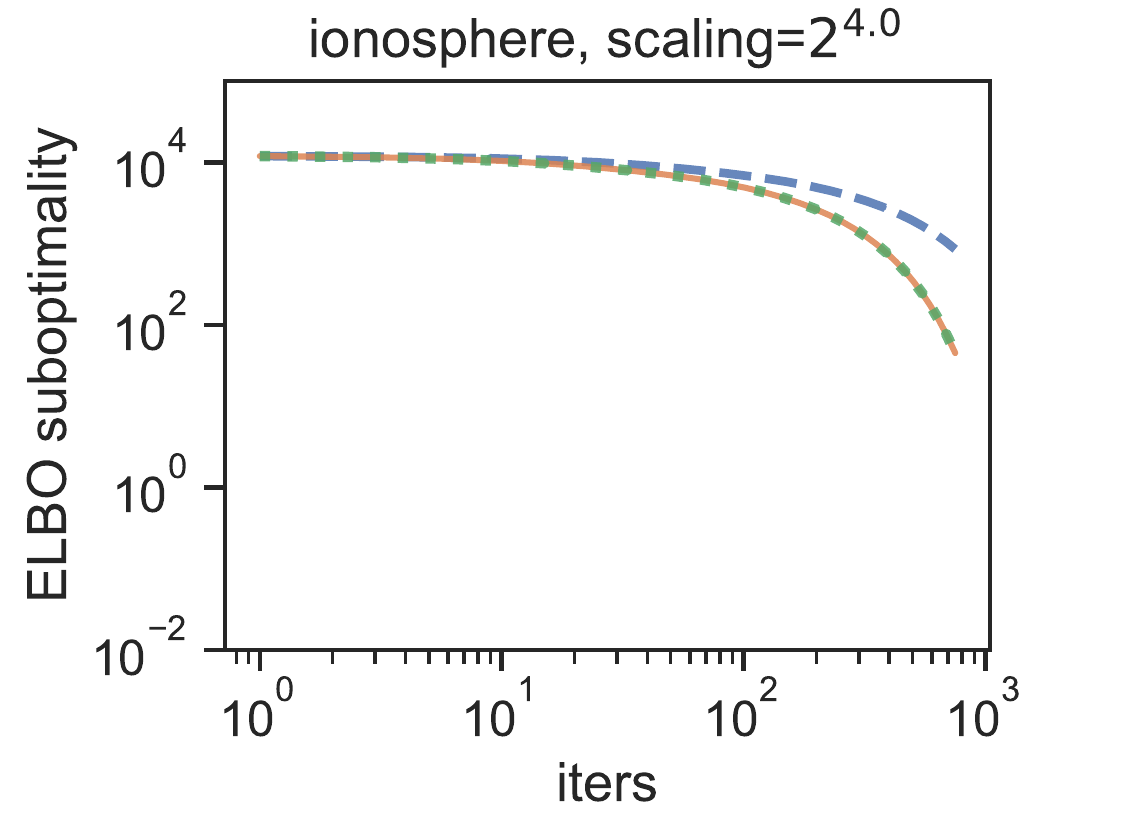}\vspace{-.25cm}
\par\end{centering}
\begin{centering}
\includegraphics[viewport=0bp 0bp 293.625bp 251.8131bp,clip,scale=0.6]{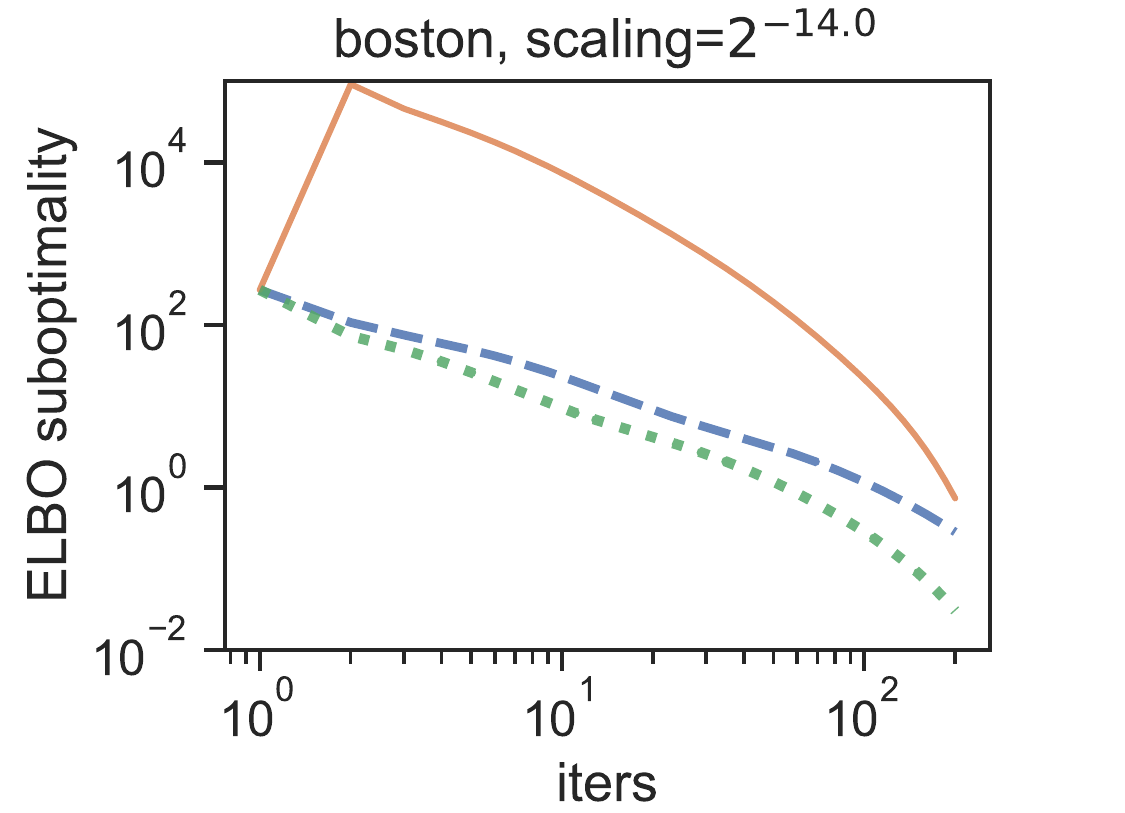}\includegraphics[viewport=60.75bp 0bp 293.625bp 251.8131bp,clip,scale=0.6]{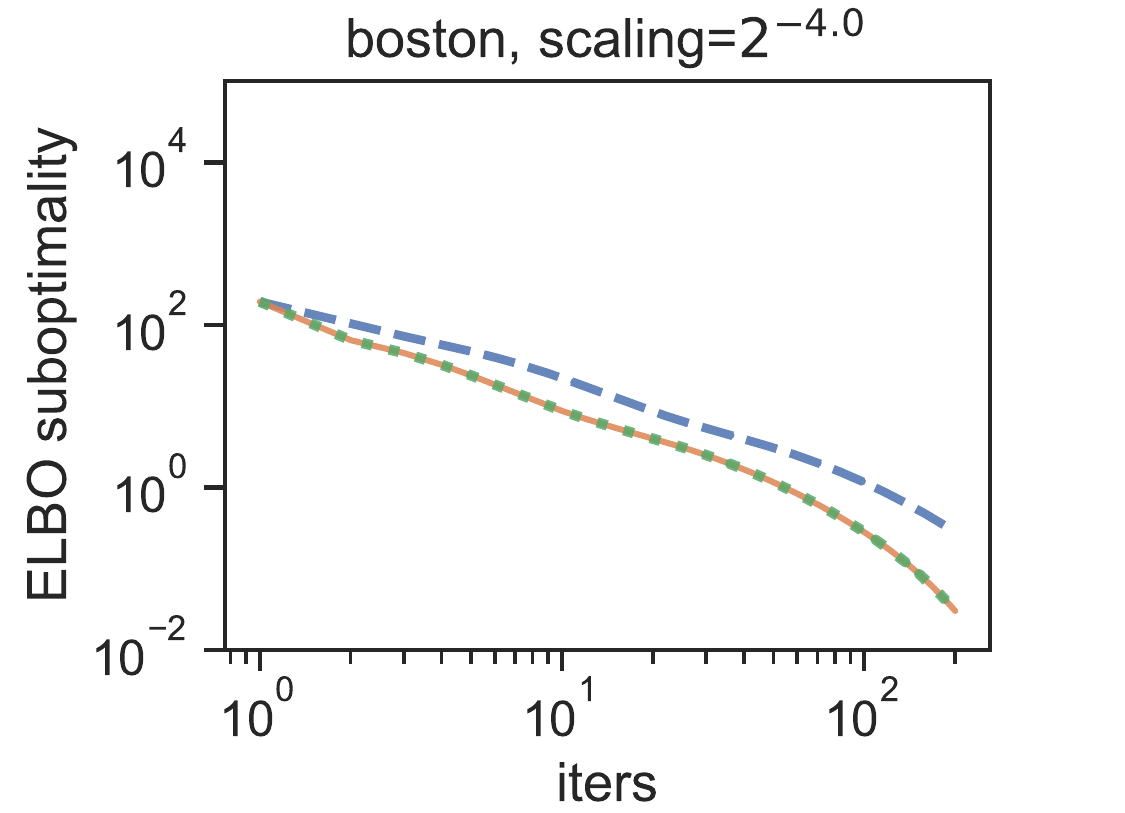}\includegraphics[viewport=60.75bp 0bp 293.625bp 251.8131bp,clip,scale=0.6]{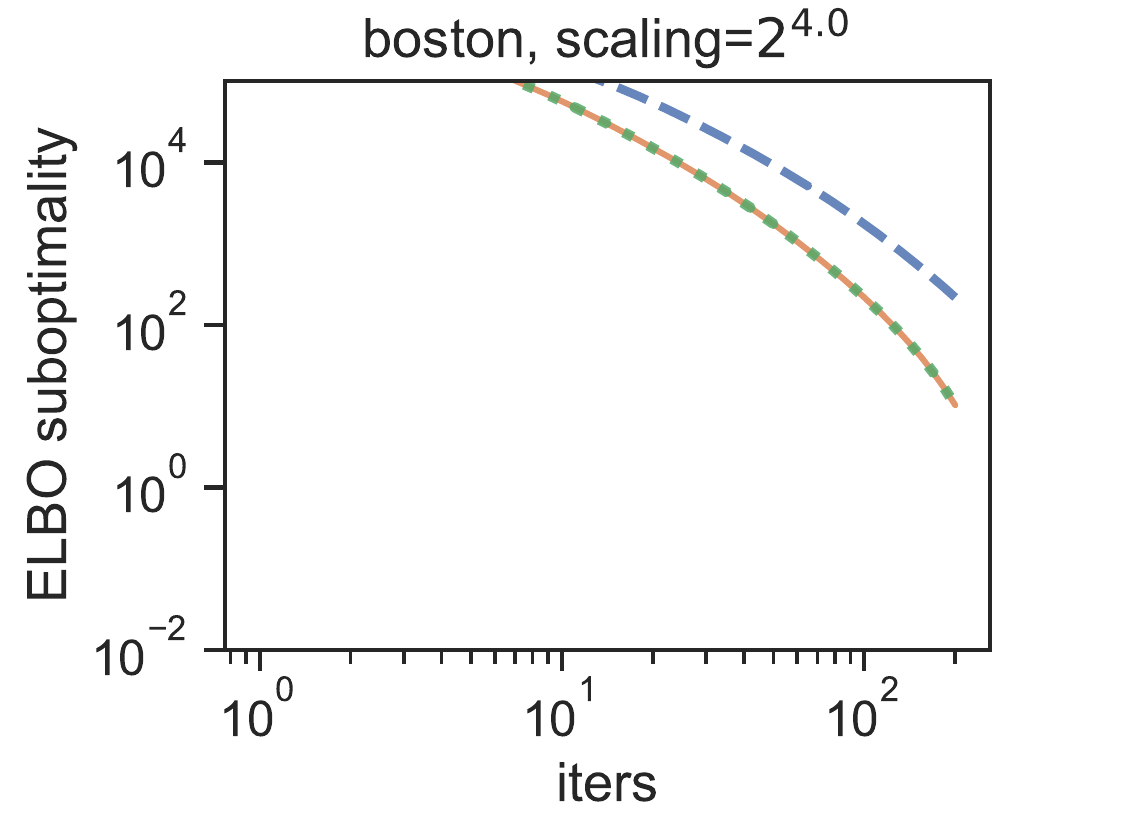}\vspace{-.25cm}
\par\end{centering}
\begin{centering}
\includegraphics[viewport=0bp 0bp 293.625bp 251.8131bp,clip,scale=0.6]{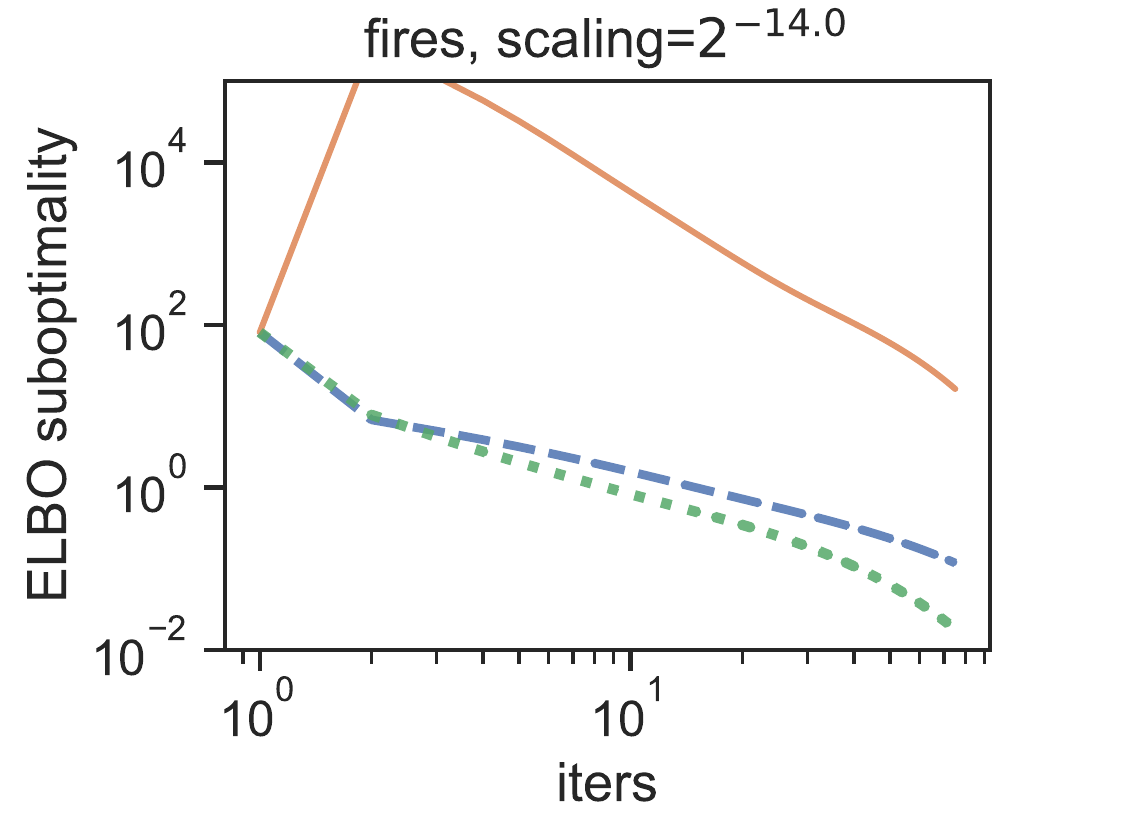}\includegraphics[viewport=60.75bp 0bp 293.625bp 251.8131bp,clip,scale=0.6]{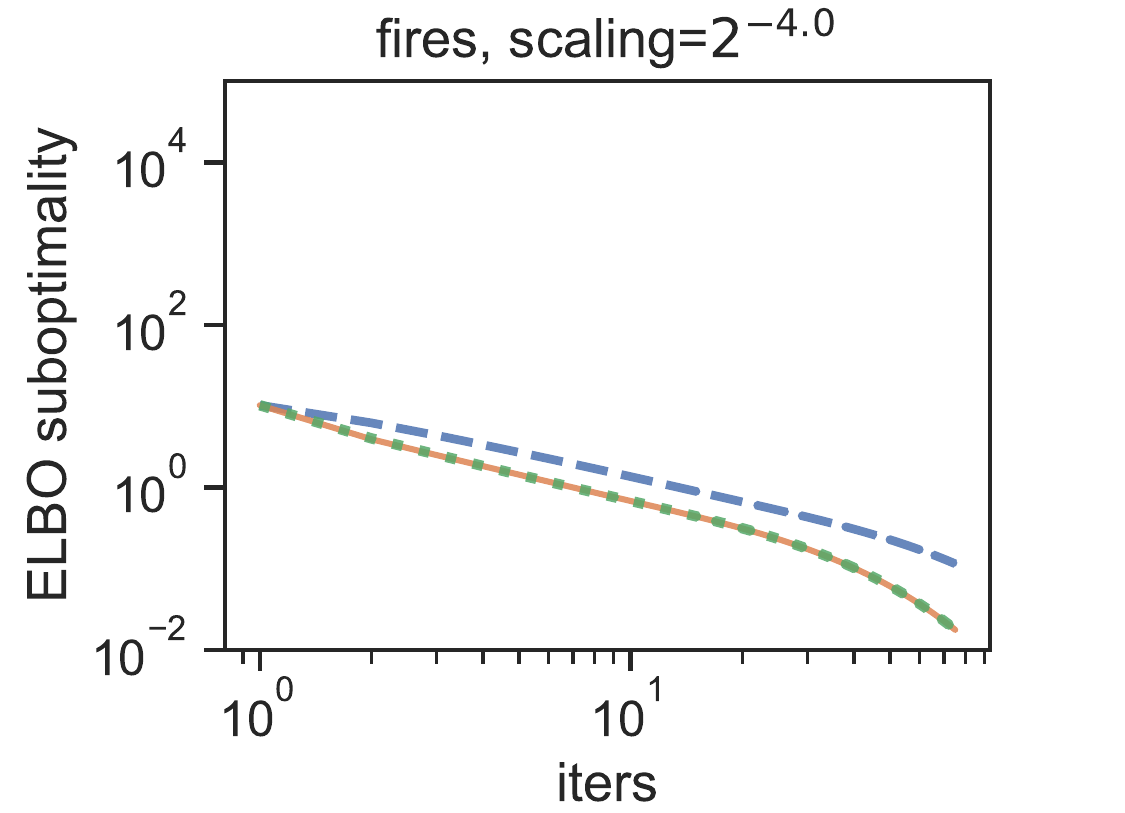}\includegraphics[viewport=60.75bp 0bp 293.625bp 251.8131bp,clip,scale=0.6]{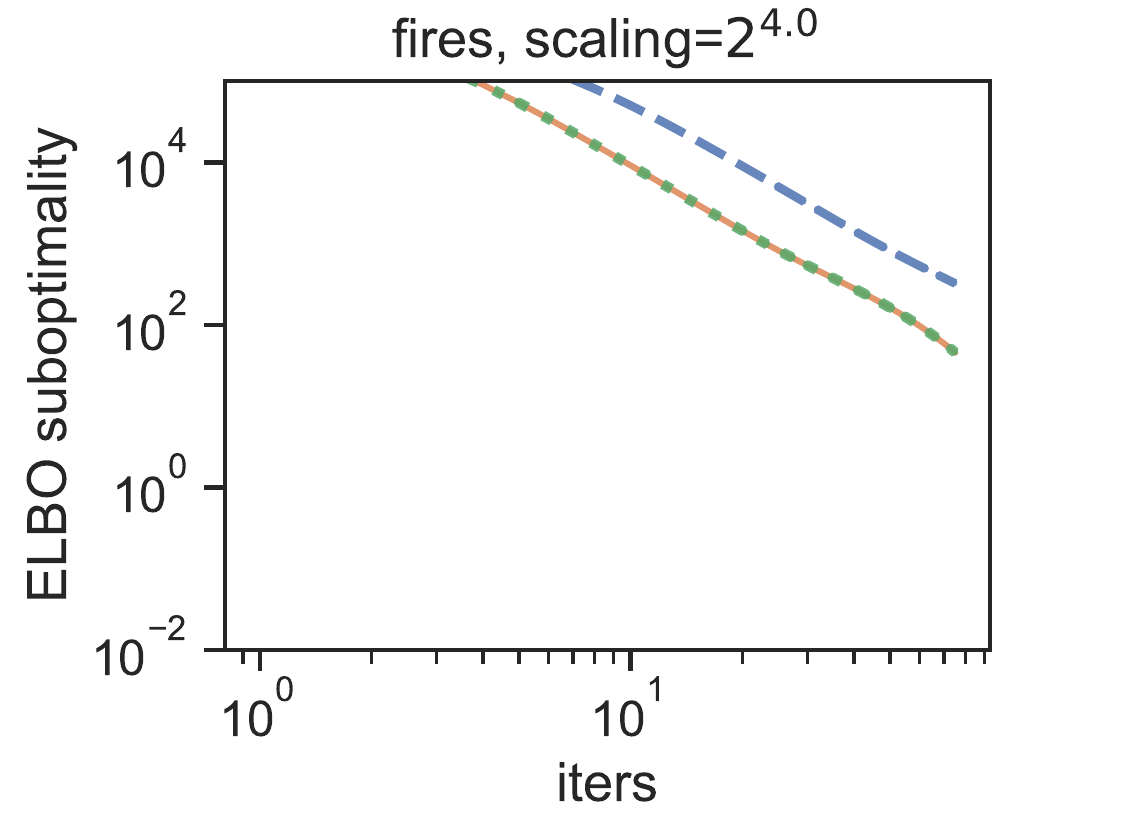}\vspace{-.25cm}
\par\end{centering}
\caption{\textbf{Naive optimization can work well, but is sensitive to initialization.}
Looseness of the objective obtained by naive gradient descent ($\gamma=1/M$),
projected gradient descent ($\gamma=1/\protect\pp{2M}$) and proximal
gradient descent ($\gamma=1/M$). Optimization starts with $\protect\b m=0$
and $C=\rho I$ where $\rho$ is a scaling factor. Initializing $C=0$
is fine for proximal or projected gradient descent, but naive gradient
descent requires careful initialization. Results for other datasets
in \ref{sec:Additional-Demonstration-Plots} (supplement).\label{fig:elbo-conv-aus}}
\end{figure*}

\begin{figure*}[t]

\begin{centering}
\includegraphics[viewport=0bp 50.6055bp 295.65bp 251.8131bp,clip,scale=0.63]{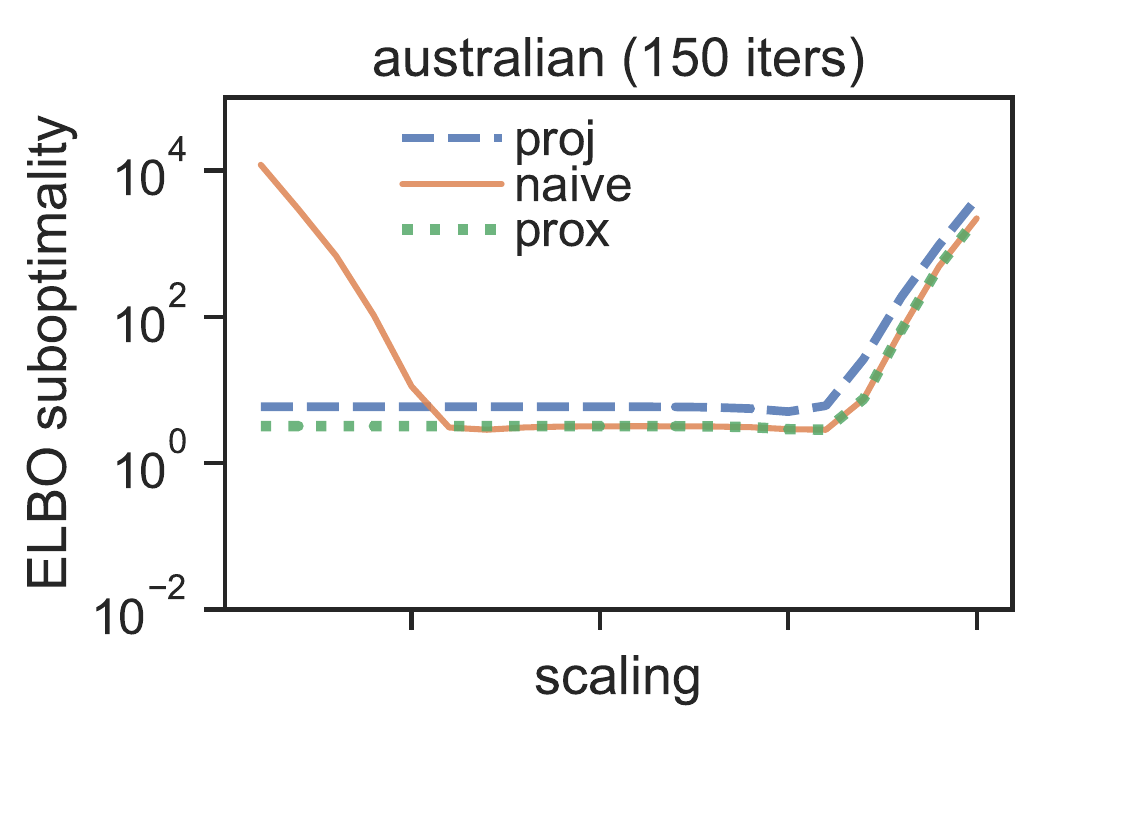}\includegraphics[viewport=60.75bp 50.6055bp 295.65bp 251.8131bp,clip,scale=0.63]{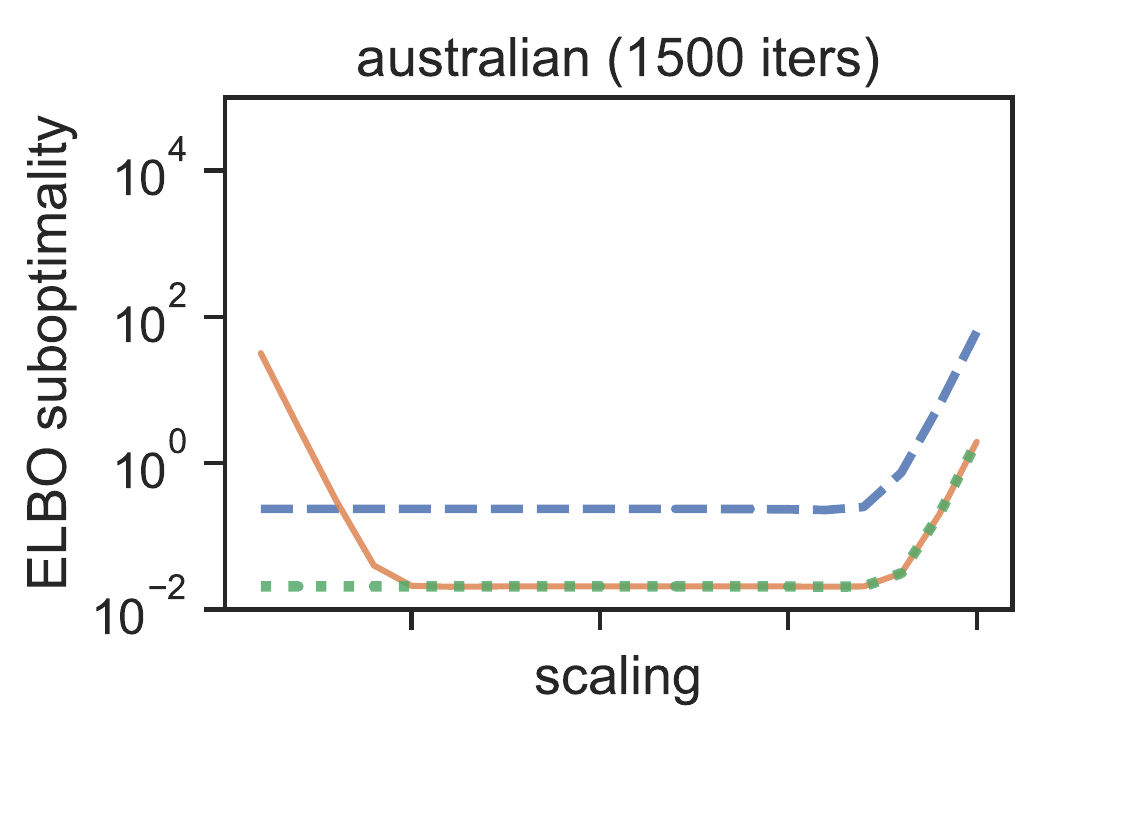}\includegraphics[viewport=60.75bp 50.6055bp 295.65bp 251.8131bp,clip,scale=0.63]{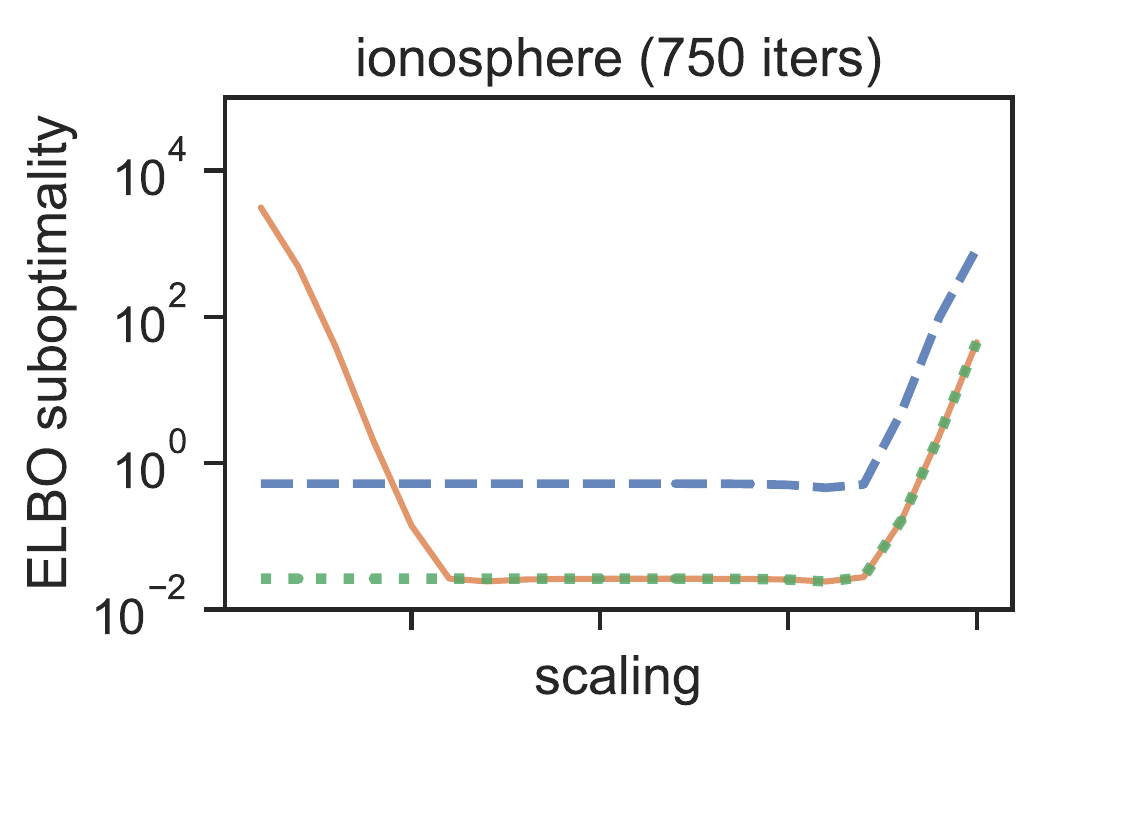}
\par\end{centering}
\begin{centering}
\includegraphics[viewport=0bp 12.14533bp 295.65bp 230.7612bp,clip,scale=0.63]{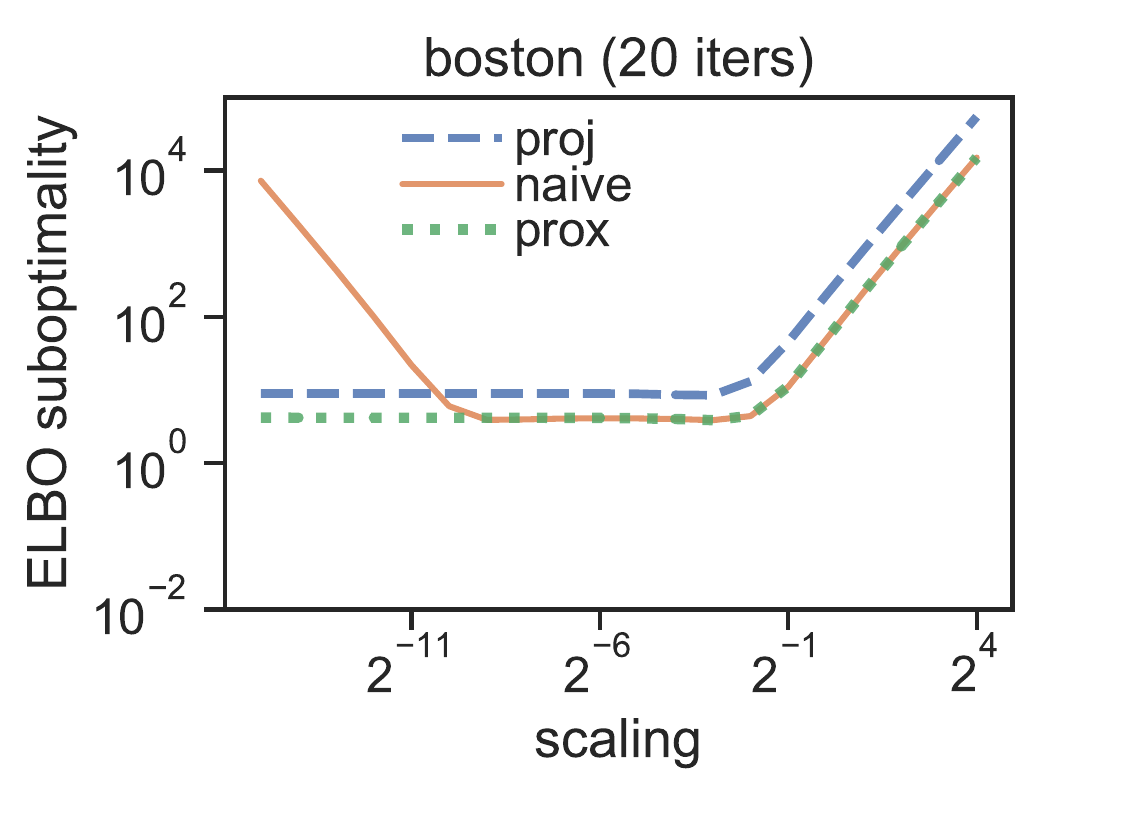}\includegraphics[viewport=60.75bp 12.14533bp 295.65bp 230.7612bp,clip,scale=0.63]{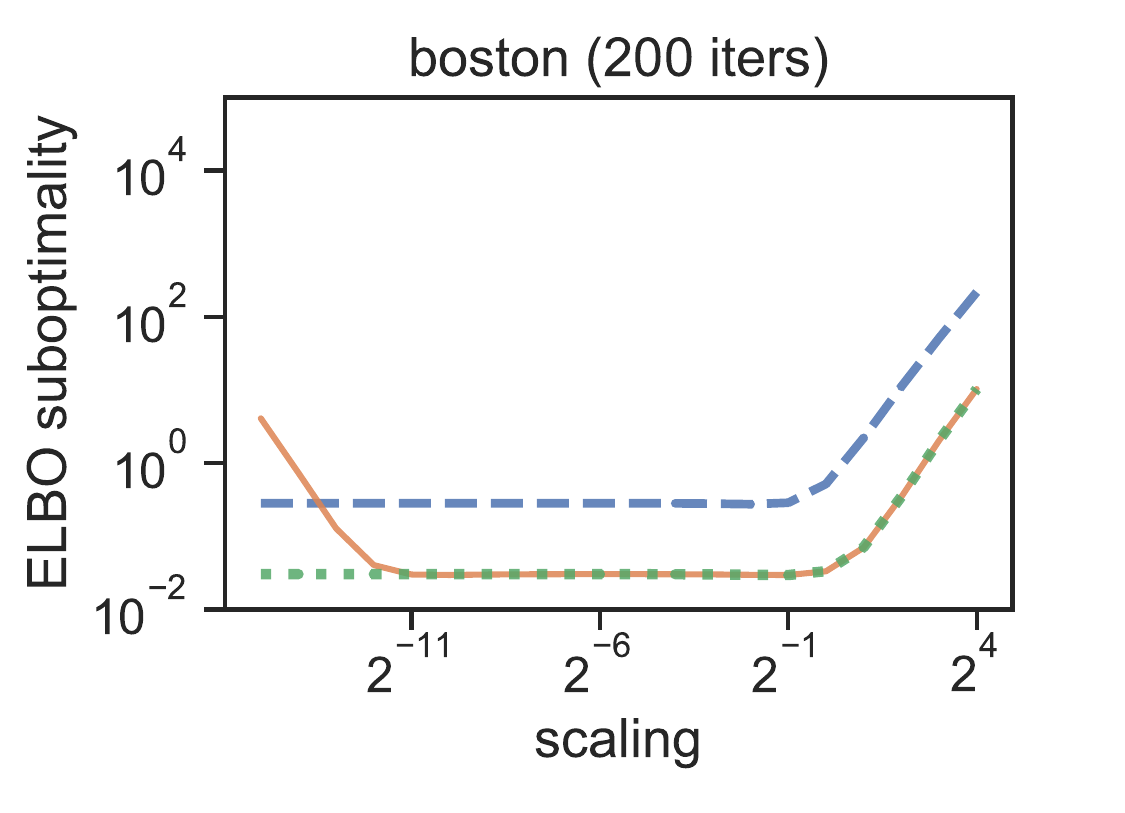}\includegraphics[viewport=60.75bp 12.14533bp 295.65bp 230.7612bp,clip,scale=0.63]{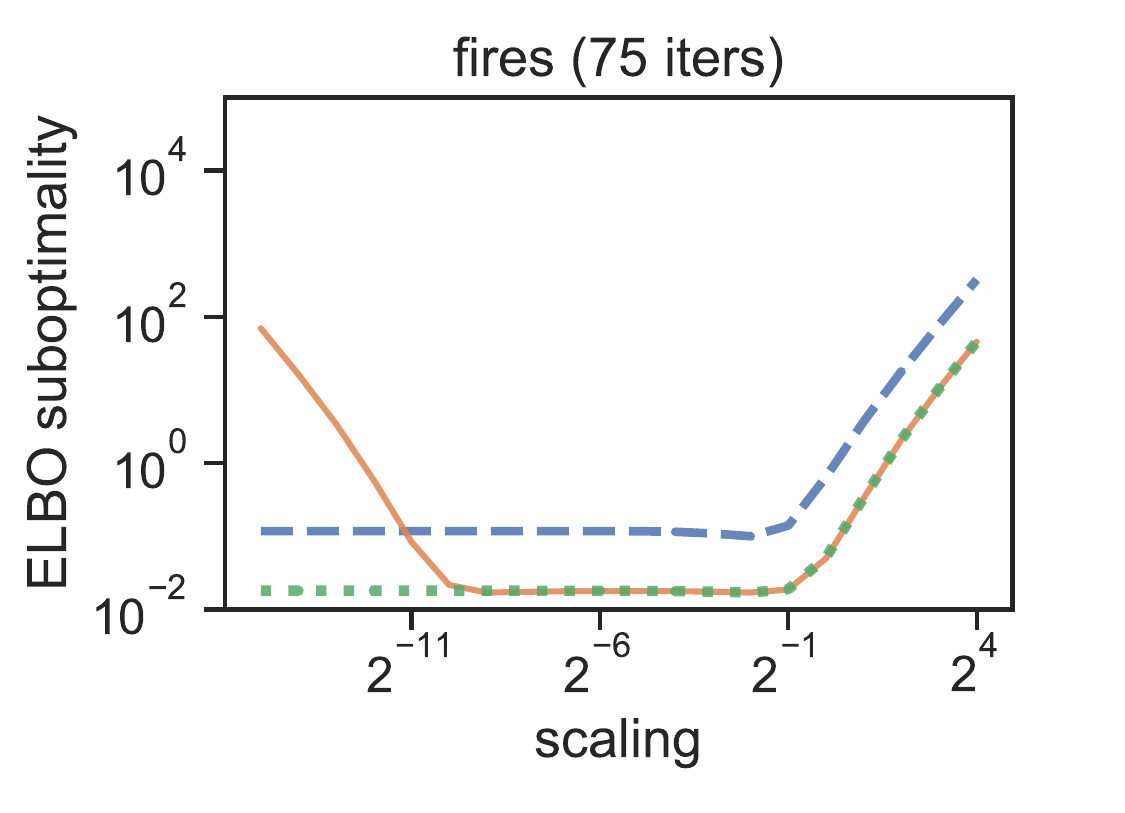}\vspace{-.25cm}
\par\end{centering}
\caption{\textbf{Naive optimization is similar to proximal for large initial
$C$, but worse for small $C$.} Results of optimizing the ELBO with
different scaling factors $\rho$ on four different datasets. The
two right columns show results after enough iterations for proximal
optimization to converge to less than $\approx10^{-1}$. The left
column shows results after $\frac{1}{10}$-th as many iterations.
Proximal optimization starting with $C\approx0$ always performs well.
Projected gradient descent requires more iterations. Naive optimization
can work well, but is not guaranteed and requires careful initialization.\label{fig:demo-final-looseness}}
\end{figure*}
 \ref{thm:gaussprox} shows that if $\w=\pp{\b m,C}$ and $C$
is triangular with a positive diagonal, then $\prox_{\gamma}\pp{\b w}=\pp{\b m,C+\Delta C},$
where $\Delta C$ is diagonal with $\Delta C_{ii}=\frac{1}{2}\pp{\pp{C_{ii}^{2}+4\gamma}^{1/2}-C_{ii}}.$
Intuitively, this has the effect of keeping the diagonal entries away
from 0: If $C_{ii}$ is very small then $\Delta C_{ii}\approx\gamma$
while if $C_{ii}$ is large, $\Delta C_{ii}\approx0.$ The proximal
scheme has two advantages over projection. First, convergence rates
depend on the smoothness constant of the linearized terms, which is
$M$ rather than $2M$.  Second the proximal operator is faster to
compute. $\prox$ takes $\Omega\pp d$ time, while $\mathrm{proj}$
takes $\Omega\pp{d^{3}}$ time, due to the need for a singular value
decomposition.

\section{Demonstration}

To avoid complications related to stochastic gradients, we consider
two settings where $l(\b w)$ and its gradient can be computed (nearly)
exactly. Take a dataset $\pp{\b x_{1},y_{1}},\cdots,\pp{\b x_{N},y_{N}}$
and let $X$ be a matrix with $\x_{n}$ on row $n$ and $\b y$ a
vector of the values $\pp{y_{1},\cdots y_{n}}$. We model $p\pp{\z,\b y|X}=p\pp{\z}\prod_{n=1}^{N}p\pp{y_{n}|\b x_{n},\z}.$
The prior $p\pp{\z}$ is a standard Gaussian. We consider both linear
regression with $p\pp{y_{n}|\b x_{n},\z}=\N\pp{y_{n}\vert\mu=\z^{\top}\b x_{n},\sigma^{2}=1}$
and binary logistic regression with $p\pp{y_{n}|\b x_{n},\z}=\sigma\pp{y_{n}\b x_{n}^{\top}\z}$.
It can be shown that $\log p(\z,\b y|X)$ is $M$-smooth with $M=1+\sigma_{\max}\pp{XX^{\top}}$
for linear regression and $M=1+\frac{1}{4}\sigma_{\max}\pp{XX^{\top}}$
for logistic regression. 

For linear regression data (\texttt{boston}, \texttt{fires}), $l$
has a closed form. For logistic regression (\texttt{australian}, \texttt{ionosphere}),
we compute $l$ via a redution to a set of pre-computed one dimensional
integrals: Observe that for all $\w,$ $\E_{\zr\sim\qw}\log p\pp{y|\x,\zr}=g\pp{y\b x^{\top}\b m,\VV{C^{\top}\b x}_{2}},$
where $g(a,b)=\E_{\r t\sim\N\pp{0,1}}\log\sigma\pp{a+b\r t}.$ By
pre-computing $g$ over a grid of inputs $\pp{a,b}$ we can quickly
evaluate $l\pp{\b w}$ and its gradient via spline interpolation.

We initialize $\b m$ to zero and $C=\rho I$ for a range of scaling
constants $\rho$. \ref{fig:elbo-conv-aus} shows example results
on two datasets. For projected or proximal gradient descent, simply
initializing $C=0$ is fine. For naive gradient descent, initialization
is subtle, since too small a $\rho$ leads to an enormous entropy
gradient (and thus ``jumps''), while for large $\rho$, all algorithms
converge slowly.

\ref{fig:demo-final-looseness} systematically varies $\rho$ on various
datasets. There are two seemingly strange behaviors for naive gradient
descent. First, it performs very similarly to proximal gradient descent
for large $\rho$. To understand this, note that when $C$ is large,
the entropy is locally nearly linear, and so a proximal step is similar
to a naive step. Second, there is a near-symmetry between small and
large $\rho$. Here, observe that if naive gradient descent is initialized
with \emph{small} $\rho$, the huge gradient of the entropy term will
send the parameters to a correspondingly \emph{large} $C$ in the
second iteration. In these examples, a carefully chosen $\rho$ performs
well, though this may be hard to find and there is no guarantee in
general.

These results confirm the theory developed above. First, we see that
proximal gradient descent always converges with a step-size of $\gamma=1/M$.
Thus suggests that $M$ as derived in \ref{thm:lsmoothness} is correct.
Second, projected gradient descent always converges with a step-size
of $1/\pp{2M}.$ This suggests that \ref{thm:projectedsummary} is
correct to assert that the optimal parameters $\w^{*}\mathcal{\in\mathcal{W}}_{M}$
and that the ELBO is $2M$-smooth over $\mathcal{W}_{M}$. Finally,
naive gradient ``descent'' truly can ascend when the parameters
$\w$ start in the region where $h\pp{\w}$ is non-smooth, but behaves
similarly to proximal gradient descent otherwise, confirming the discussion
in \ref{sec:Convergence-Considerations}.

\section{Discussion}

The primary contribution of this paper is to show that for VI with
location-scale families, smoothness of $\log p\pp{\z,\x}$ implies
smoothness of the free energy. This fills a theoretical gap relevant
to many existing works \citep{Khan_2015_KullbackLeiblerproximalvariational,Khan_2016_FasterStochasticVariational,Khan_2017_ConjugateComputationVariationalInference,Regier_2017_FastBlackboxVariational,Fan_2016_TriplyStochasticVariational,Buchholz_2018_QuasiMonteCarloVariational,Mohamad_2018_AsynchronousStochasticVariational,Alquier_2017_Concentrationtemperedposteriors}.
We also showed that result gives parameter-space guarantees on the
location of the optimal parameters. As a minor contribution, we also
give analogous guarantees for strong-convexity. Convergence guarantees
for gradient-based optimization require \emph{either} smoothness or
convexity. Thus, at a very high level, this paper shows that if $\log p\pp{\z,\x}$
has the structure needed to guarantee finding $\z^{*}=\argmax\log p(\z,\x)$,
then it \emph{also} has the structure to guarantee that VI with a
location-scale family will converge.

While motivated by VI, the main results for smoothness (\ref{thm:lsmoothness})
and (strong) convexity (\ref{thm:convexity}) are general properties
of expectations parameterized by location-scale families, and so may
be of independent interest.

There are several issues to consider when gauging the immediate practical
impact of this work. Most importantly, the smoothness guarantee in
this paper was already \emph{true}, even if was not \emph{known}.
Thus, real-world black-box VI methods already benefit from it. Second,
$\nabla l\pp{\w}$ typically must be \emph{estimated}, and convergence
guarantees need bounds on the fluctuations of the estimator. Finding
better gradient estimators (and bounds) is an active research topic
\citep{Domke_2019_ProvableGradientVariancea}. Finally, the theory
for projected and proximal gradient optimization is still evolving,
particularly for non-convex objectives. It seems to be an open question
if the ``minibatches'' of gradient estimates that current bounds
\citep{Ghadimi_2016_Minibatchstochasticapproximation} use are truly
required.

A result conceptually related to this paper's smoothness guarantee
is used in variational boosting \citep{Guo_2016_BoostingVariationalInference,Locatello_2018_BoostingVariationalInference}:
The functional gradient for non-parametric $q$ is smooth if $q$
is bounded below by a positive constant. While similar in spirit,
this does not address traditional parametric VI.

\bibliographystyle{plainnat}
\bibliography{/Users/j/Dropbox/Papers/Bibliography/justindomke_zotero_betterbibtex2}

\begin{thebibliography}{42}
\providecommand{\natexlab}[1]{#1}
\providecommand{\url}[1]{\texttt{#1}}
\expandafter\ifx\csname urlstyle\endcsname\relax
  \providecommand{\doi}[1]{doi: #1}\else
  \providecommand{\doi}{doi: \begingroup \urlstyle{rm}\Url}\fi

\bibitem[Alquier and
  Ridgway(2017)]{Alquier_2017_Concentrationtemperedposteriors}
Pierre Alquier and James Ridgway.
\newblock Concentration of tempered posteriors and of their variational
  approximations.
\newblock \emph{arXiv:1706.09293 [cs, math, stat]}, 2017.

\bibitem[Beck and
  Teboulle(2009)]{Beck_2009_Gradientbasedalgorithmsapplications}
Amir Beck and Marc Teboulle.
\newblock Gradient-based algorithms with applications to signal recovery.
\newblock \emph{Convex optimization in signal processing and communications},
  pages 42--88, 2009.

\bibitem[Blei et~al.(2017)Blei, Kucukelbir, and
  McAuliffe]{Blei_2017_VariationalInferenceReview}
David~M. Blei, Alp Kucukelbir, and Jon~D. McAuliffe.
\newblock Variational {{Inference}}: {{A Review}} for {{Statisticians}}.
\newblock \emph{Journal of the American Statistical Association}, 112\penalty0
  (518):\penalty0 859--877, 2017.

\bibitem[Bottou et~al.(2016)Bottou, Curtis, and
  Nocedal]{Bottou_2016_OptimizationMethodsLargeScale}
L{\'e}on Bottou, Frank~E. Curtis, and Jorge Nocedal.
\newblock Optimization {{Methods}} for {{Large}}-{{Scale Machine Learning}}.
\newblock \emph{arXiv:1606.04838 [cs, math, stat]}, 2016.

\bibitem[Bubeck(2015)]{Bubeck_2015_ConvexOptimizationAlgorithms}
S{\'e}bastien Bubeck.
\newblock Convex {{Optimization}}: {{Algorithms}} and {{Complexity}}.
\newblock \emph{Foundations and Trends in Machine Learning}, 8\penalty0
  (3-4):\penalty0 231--358, 2015.

\bibitem[Buchholz et~al.(2018)Buchholz, Wenzel, and
  Mandt]{Buchholz_2018_QuasiMonteCarloVariational}
Alexander Buchholz, Florian Wenzel, and Stephan Mandt.
\newblock Quasi-{{Monte Carlo Variational Inference}}.
\newblock In \emph{{{ICML}}}, 2018.

\bibitem[Challis and
  Barber(2013)]{Challis_2013_GaussianKullbackLeiblerApproximate}
Edward Challis and David Barber.
\newblock Gaussian {{Kullback}}-{{Leibler Approximate Inference}}.
\newblock \emph{Journal of Machine Learning Research}, 14:\penalty0 2239--2286,
  2013.

\bibitem[Cover and Thomas(2006)]{Cover_2006_Elementsinformationtheory}
T.~M. Cover and Joy~A. Thomas.
\newblock \emph{Elements of Information Theory}.
\newblock {Wiley-Interscience}, {Hoboken, N.J}, 2nd ed edition, 2006.

\bibitem[Domke(2019)]{Domke_2019_ProvableGradientVariancea}
Justin Domke.
\newblock Provable {{Gradient Variance Guarantees}} for {{Black}}-{{Box
  Variational Inference}}.
\newblock In \emph{{{NeurIPS}}}, 2019.

\bibitem[Fan et~al.(2016)Fan, Zhang, Henao, and
  Heller]{Fan_2016_TriplyStochasticVariational}
K.~Fan, Y.~Zhang, R.~Henao, and K.~Heller.
\newblock Triply {{Stochastic Variational Inference}} for {{Non}}-linear {{Beta
  Process Factor Analysis}}.
\newblock In \emph{2016 {{IEEE}} 16th {{International Conference}} on {{Data
  Mining}} ({{ICDM}})}, pages 121--130, 2016.

\bibitem[Fan et~al.(2015)Fan, Wang, Beck, Kwok, and
  Heller]{Fan_2015_FastSecondOrderStochastic}
Kai Fan, Ziteng Wang, Jeff Beck, James Kwok, and Katherine Heller.
\newblock Fast {{Second}}-{{Order Stochastic Backpropagation}} for
  {{Variational Inference}}.
\newblock In \emph{{{NeurIPS}}}, 2015.

\bibitem[Ge et~al.(2015)Ge, Huang, Jin, and Yuan]{Ge_2015_EscapingSaddlePoints}
Rong Ge, Furong Huang, Chi Jin, and Yang Yuan.
\newblock Escaping {{From Saddle Points}} \textendash{} {{Online Stochastic
  Gradient}} for {{Tensor Decomposition}}.
\newblock In \emph{{{COLT}}}, page~46, 2015.

\bibitem[Geyer(2011)]{Geyer_2011_Statistics5101Lecture}
Charles~J Geyer.
\newblock Statistics 5101 {{Lecture Slides}}, {{Deck}} 5.
\newblock http://www.stat.umn.edu/geyer/f11/5101/slides/s5.pdf, 2011.

\bibitem[Ghadimi and Lan(2013)]{Ghadimi_2013_StochasticFirstZerothOrder}
S.~Ghadimi and G.~Lan.
\newblock Stochastic {{First}}- and {{Zeroth}}-{{Order Methods}} for
  {{Nonconvex Stochastic Programming}}.
\newblock \emph{SIAM J. Optim.}, 23\penalty0 (4):\penalty0 2341--2368, 2013.

\bibitem[Ghadimi and Lan(2012)]{Ghadimi_2012_OptimalStochasticApproximation}
Saeed Ghadimi and Guanghui Lan.
\newblock Optimal {{Stochastic Approximation Algorithms}} for {{Strongly Convex
  Stochastic Composite Optimization I}}: {{A Generic Algorithmic Framework}}.
\newblock \emph{SIAM Journal on Optimization}, 22\penalty0 (4):\penalty0
  1469--1492, 2012.

\bibitem[Ghadimi et~al.(2016)Ghadimi, Lan, and
  Zhang]{Ghadimi_2016_Minibatchstochasticapproximation}
Saeed Ghadimi, Guanghui Lan, and Hongchao Zhang.
\newblock Mini-batch stochastic approximation methods for nonconvex stochastic
  composite optimization.
\newblock \emph{Mathematical Programming}, 155\penalty0 (1-2):\penalty0
  267--305, 2016.

\bibitem[Ghahramani and
  Beal(2001)]{Ghahramani_2001_PropagationAlgorithmsVariational}
Zoubin Ghahramani and Matthew Beal.
\newblock Propagation {{Algorithms}} for {{Variational Bayesian Learning}}.
\newblock In \emph{{{NeurIPS}}}, 2001.

\bibitem[Guo et~al.(2016)Guo, Wang, Fan, Broderick, and
  Dunson]{Guo_2016_BoostingVariationalInference}
Fangjian Guo, Xiangyu Wang, Kai Fan, Tamara Broderick, and David~B. Dunson.
\newblock Boosting {{Variational Inference}}.
\newblock \emph{arXiv:1611.05559 [cs, stat]}, 2016.

\bibitem[Hasegawa and
  Karamakar()]{Hasegawa__GeneralizationsCauchySchwarzProbability}
R~Hasegawa and B~Karamakar.
\newblock Generalizations of {{Cauchy}}-{{Schwarz}} in {{Probability Theory}}.

\bibitem[Khan et~al.(2015)Khan, Baqu{\'e}, Fleuret, and
  Fua]{Khan_2015_KullbackLeiblerproximalvariational}
Mohammad~E. Khan, Pierre Baqu{\'e}, Fran{\c c}ois Fleuret, and Pascal Fua.
\newblock Kullback-{{Leibler}} proximal variational inference.
\newblock In \emph{{{NeurIPS}}}, pages 3402--3410, 2015.

\bibitem[Khan and
  Lin(2017)]{Khan_2017_ConjugateComputationVariationalInference}
Mohammad~Emtiyaz Khan and Wu~Lin.
\newblock Conjugate-{{Computation Variational Inference}}: {{Converting
  Variational Inference}} in {{Non}}-{{Conjugate Models}} to {{Inferences}} in
  {{Conjugate Models}}.
\newblock In \emph{{{AISTATS}}}, volume~54 of \emph{Proceedings of {{Machine
  Learning Research}}}, pages 878--887. {PMLR}, 2017.

\bibitem[Khan et~al.(2016)Khan, Babanezhad, Lin, Schmidt, and
  Sugiyama]{Khan_2016_FasterStochasticVariational}
Mohammad~Emtiyaz Khan, Reza Babanezhad, Wu~Lin, Mark Schmidt, and Masashi
  Sugiyama.
\newblock Faster {{Stochastic Variational Inference}} using
  {{Proximal}}-{{Gradient Methods}} with {{General Divergence Functions}}.
\newblock In \emph{{{UAI}}}, 2016.

\bibitem[Kingma and Welling(2014)]{Kingma_2014_Autoencodingvariationalbayes}
Diederik~P. Kingma and Max Welling.
\newblock Auto-encoding variational bayes.
\newblock In \emph{{{ICLR}}}, 2014.

\bibitem[Kucukelbir et~al.(2017)Kucukelbir, Tran, Ranganath, Gelman, and
  Blei]{Kucukelbir_2017_AutomaticDifferentiationVariational}
Alp Kucukelbir, Dustin Tran, Rajesh Ranganath, Andrew Gelman, and David~M.
  Blei.
\newblock Automatic {{Differentiation Variational Inference}}.
\newblock \emph{Journal of Machine Learning Research}, 18\penalty0
  (14):\penalty0 1--45, 2017.

\bibitem[Lee et~al.(2016)Lee, Simchowitz, Jordan, and
  Recht]{Lee_2016_GradientDescentOnly}
Jason~D Lee, Max Simchowitz, Michael~I Jordan, and Benjamin Recht.
\newblock Gradient {{Descent Only Converges}} to {{Minimizers}}.
\newblock In \emph{{{COLT}}}, page~12, 2016.

\bibitem[Locatello et~al.(2018)Locatello, Khanna, Ghosh, and
  Ratsch]{Locatello_2018_BoostingVariationalInference}
Francesco Locatello, Rajiv Khanna, Joydeep Ghosh, and Gunnar Ratsch.
\newblock Boosting {{Variational Inference}}: An {{Optimization Perspective}}.
\newblock In \emph{{{AISTATS}}}, pages 464--472, 2018.

\bibitem[Mohamad et~al.(2018)Mohamad, Bouchachia, and
  {Sayed-Mouchaweh}]{Mohamad_2018_AsynchronousStochasticVariational}
Saad Mohamad, Abdelhamid Bouchachia, and Moamar {Sayed-Mouchaweh}.
\newblock Asynchronous {{Stochastic Variational Inference}}.
\newblock \emph{arXiv:1801.04289 [cs, stat]}, 2018.

\bibitem[Nesterov(2014)]{Nesterov_2014_Introductorylecturesconvex}
Yurii Nesterov.
\newblock \emph{Introductory Lectures on Convex Optimization: A Basic Course.}
\newblock {Springer}, {Place of publication not identified}, 2014.
\newblock OCLC: 878109549.

\bibitem[Parikh(2014)]{Parikh_2014_ProximalAlgorithms}
Neal Parikh.
\newblock Proximal {{Algorithms}}.
\newblock \emph{Foundations and Trends\textregistered{} in Optimization},
  1\penalty0 (3):\penalty0 127--239, 2014.

\bibitem[Rakhlin et~al.(2012)Rakhlin, Shamir, and
  Sridharan]{Rakhlin_2012_Makinggradientdescent}
Alexander Rakhlin, Ohad Shamir, and Karthik Sridharan.
\newblock Making gradient descent optimal for strongly convex stochastic
  optimization.
\newblock In \emph{{{ICML}}}, pages 449--456, 2012.

\bibitem[Ranganath et~al.(2014)Ranganath, Gerrish, and
  Blei]{Ranganath_2014_BlackBoxVariational}
Rajesh Ranganath, Sean Gerrish, and David~M. Blei.
\newblock Black {{Box Variational Inference}}.
\newblock In \emph{{{AISTATS}}}, 2014.

\bibitem[Regier et~al.(2017{\natexlab{a}})Regier, Jordan, and
  McAuliffe]{Regier_2017_FastBlackboxVariational}
Jeffrey Regier, Michael~I Jordan, and Jon McAuliffe.
\newblock Fast {{Black}}-box {{Variational Inference}} through {{Stochastic
  Trust}}-{{Region Optimization}}.
\newblock In \emph{{{NeurIPS}}}, pages 2399--2408. {Curran Associates, Inc.},
  2017{\natexlab{a}}.

\bibitem[Regier et~al.(2017{\natexlab{b}})Regier, Jordan, and
  McAuliffe]{Regier_2017_FastBlackboxVariationala}
Jeffrey Regier, Michael~I Jordan, and Jon McAuliffe.
\newblock Fast {{Black}}-box {{Variational Inference}} through {{Stochastic
  Trust}}-{{Region Optimization}}.
\newblock In \emph{{{NeurIPS}}}, page~10, 2017{\natexlab{b}}.

\bibitem[Rezende et~al.(2014)Rezende, Mohamed, and
  Wierstra]{Rezende_2014_StochasticBackpropagationApproximate}
Danilo~Jimenez Rezende, Shakir Mohamed, and Daan Wierstra.
\newblock Stochastic {{Backpropagation}} and {{Approximate Inference}} in
  {{Deep Generative Models}}.
\newblock In \emph{{{ICML}}}, 2014.

\bibitem[Rooin and Bayat(2012)]{Rooin_2012_EquivalencyCauchySchwarzBessel}
Jamal Rooin and Morteza Bayat.
\newblock Equivalency of {{Cauchy}}-{{Schwarz}} and {{Bessel Inequalities}}.
\newblock \emph{The Mathematical Intelligencer}, 34\penalty0 (4):\penalty0
  2--3, 2012.

\bibitem[Salimans and
  Knowles(2013)]{Salimans_2013_FixedFormVariationalPosterior}
Tim Salimans and David~A. Knowles.
\newblock Fixed-{{Form Variational Posterior Approximation}} through
  {{Stochastic Linear Regression}}.
\newblock \emph{Bayesian Anal.}, 8\penalty0 (4):\penalty0 837--882, 2013.

\bibitem[Titsias and
  {L{\'a}zaro-gredilla}(2014)]{Titsias_2014_DoublyStochasticVariational}
Michalis Titsias and Miguel {L{\'a}zaro-gredilla}.
\newblock Doubly {{Stochastic Variational Bayes}} for non-{{Conjugate
  Inference}}.
\newblock In \emph{{{ICML}}}, 2014.

\bibitem[Wingate and Weber(2013)]{Wingate_2013_AutomatedVariationalInference}
David Wingate and Theophane Weber.
\newblock Automated {{Variational Inference}} in {{Probabilistic Programming}}.
\newblock \emph{arXiv:1301.1299 [cs, stat]}, 2013.

\bibitem[Winn and Bishop(2005)]{Winn_2005_VariationalMessagePassing}
John Winn and Christopher~M Bishop.
\newblock Variational {{Message Passing}}.
\newblock \emph{Journal of Machine Learning Research}, 6:\penalty0 661--694,
  2005.

\bibitem[Xu et~al.(2018)Xu, Quiroz, Kohn, and
  Sisson]{Xu_2018_variancereductionproperties}
Ming Xu, Matias Quiroz, Robert Kohn, and Scott~A. Sisson.
\newblock On some variance reduction properties of the reparameterization
  trick.
\newblock \emph{arXiv:1809.10330 [cs, stat]}, 2018.

\bibitem[Yang et~al.(2016)Yang, Lin, and
  Li]{Yang_2016_UnifiedConvergenceAnalysis}
Tianbao Yang, Qihang Lin, and Zhe Li.
\newblock Unified {{Convergence Analysis}} of {{Stochastic Momentum Methods}}
  for {{Convex}} and {{Non}}-convex {{Optimization}}.
\newblock \emph{arXiv:1604.03257 [math, stat]}, 2016.

\bibitem[Yao et~al.(2018)Yao, Vehtari, Simpson, and Gelman]{Yao_2018_YesDidIt}
Yuling Yao, Aki Vehtari, Daniel Simpson, and Andrew Gelman.
\newblock Yes, but {{Did It Work}}?: {{Evaluating Variational Inference}}.
\newblock In \emph{{{ICML}}}, 2018.

\end{thebibliography}
\clearpage{}

\section{Additional Demonstration Plots\label{sec:Additional-Demonstration-Plots}}

\begin{figure*}
\begin{centering}
\includegraphics[viewport=0bp 44.5329bp 293.625bp 251.8131bp,clip,scale=0.6]{GLMs/figures_individual_leg/australian_1}\includegraphics[viewport=60.75bp 44.5329bp 293.625bp 251.8131bp,clip,scale=0.6]{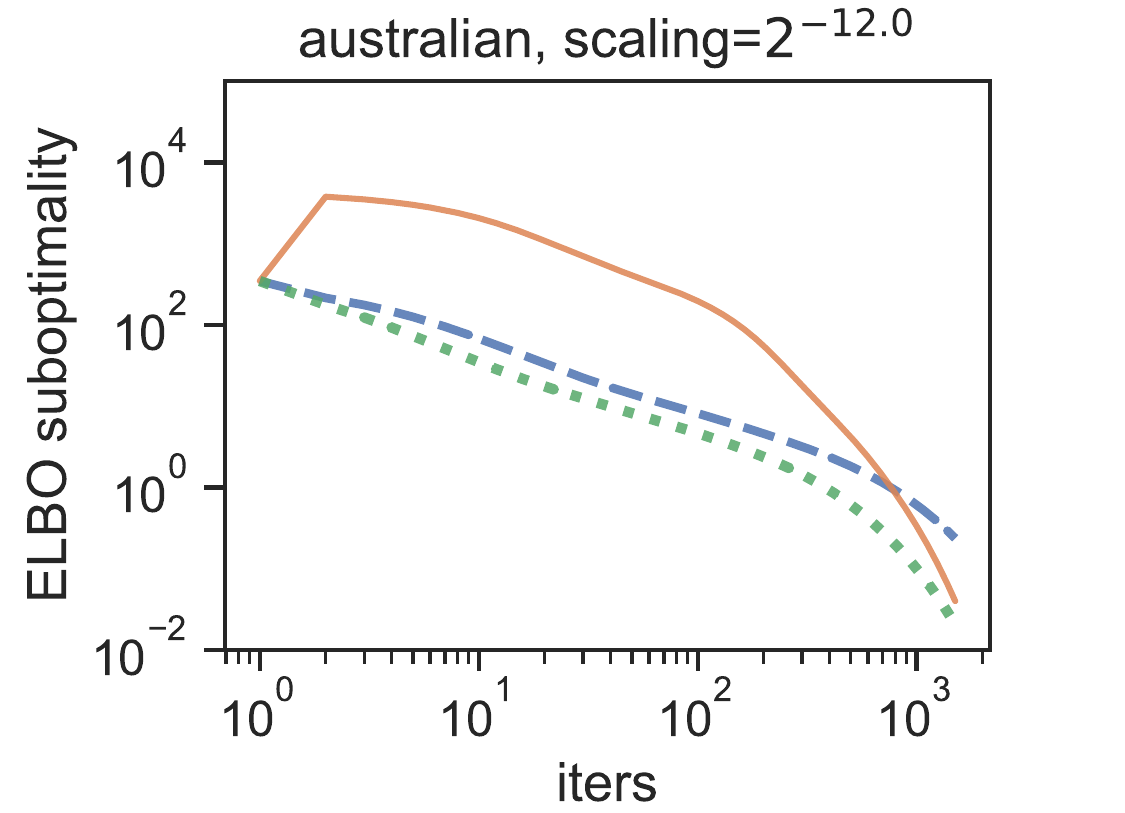}\includegraphics[viewport=60.75bp 44.5329bp 293.625bp 251.8131bp,clip,scale=0.6]{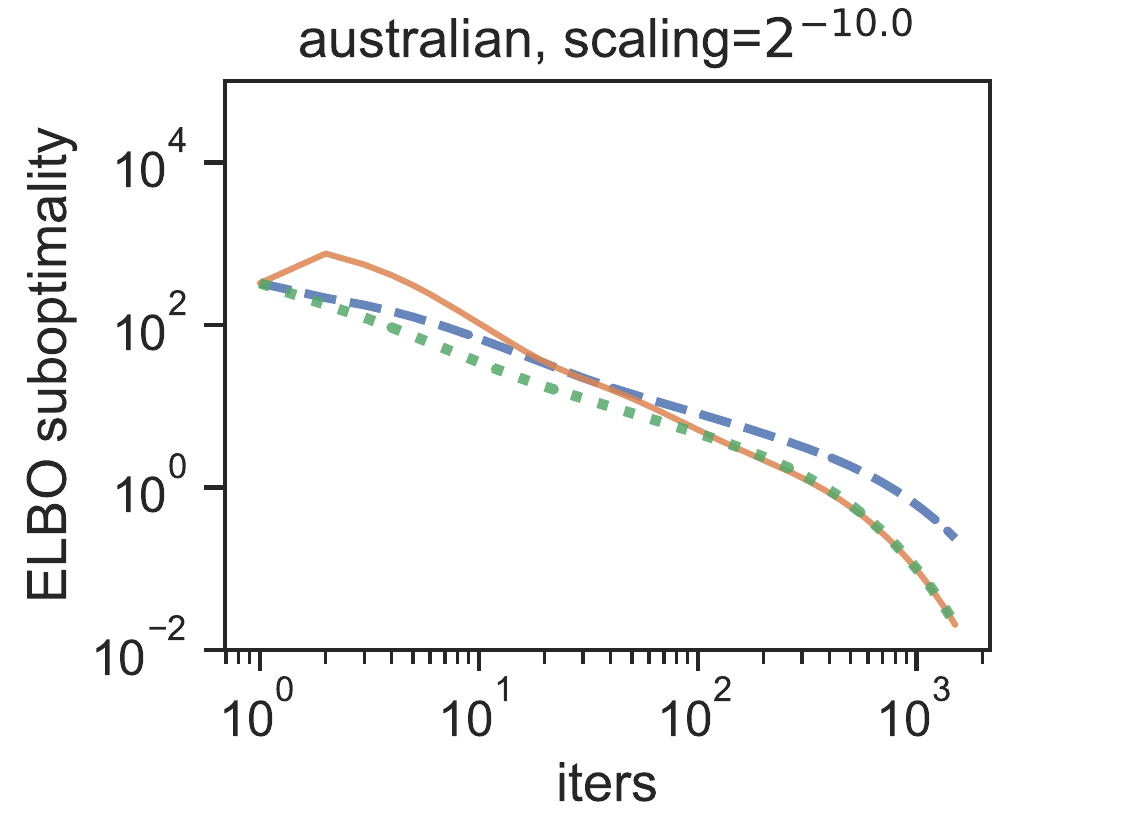}
\par\end{centering}
\begin{centering}
\includegraphics[viewport=0bp 44.5329bp 293.625bp 251.8131bp,clip,scale=0.6]{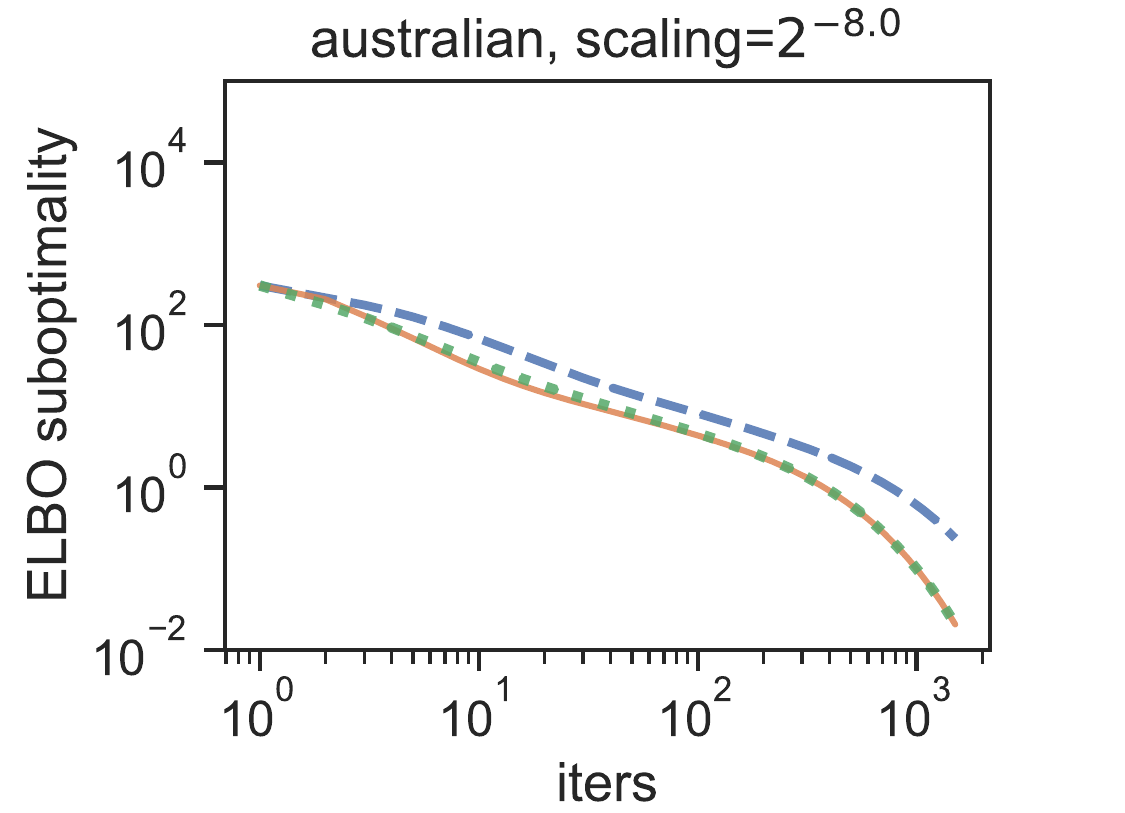}\includegraphics[viewport=60.75bp 44.5329bp 293.625bp 251.8131bp,clip,scale=0.6]{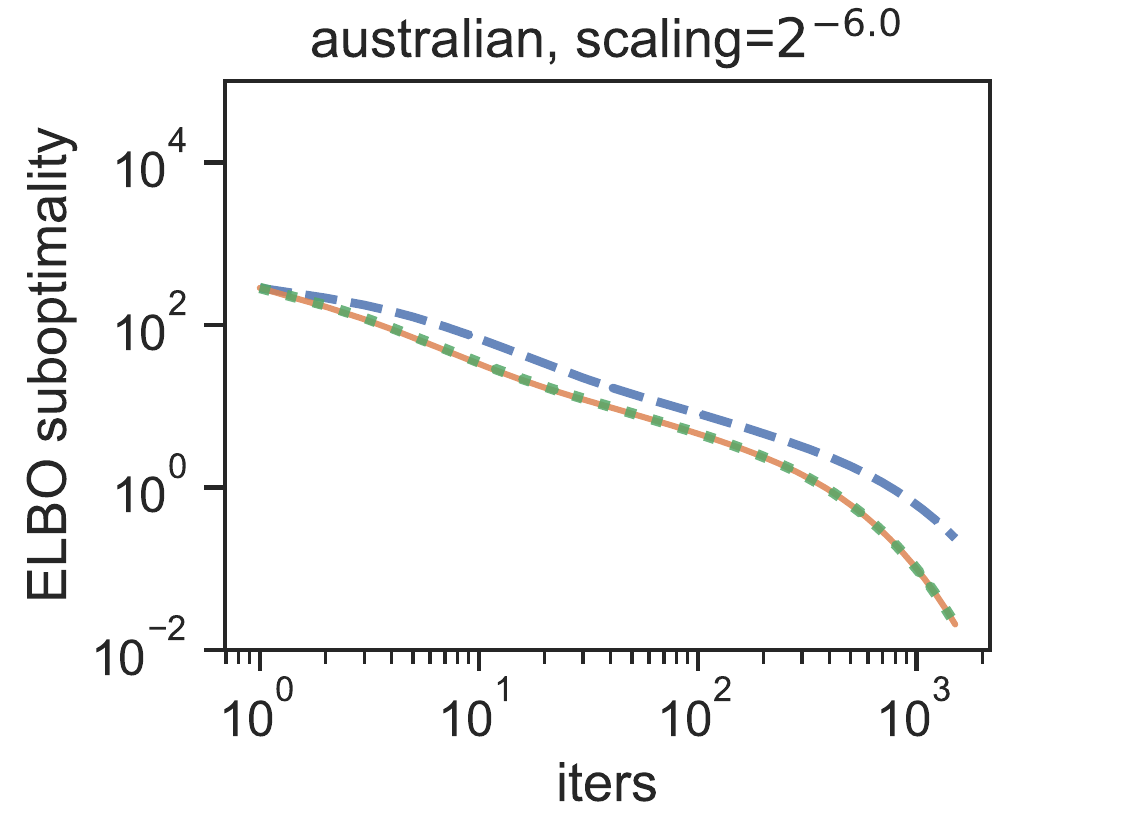}\includegraphics[viewport=60.75bp 44.5329bp 293.625bp 251.8131bp,clip,scale=0.6]{GLMs/figures_individual/australian_11}
\par\end{centering}
\begin{centering}
\includegraphics[viewport=0bp 0bp 293.625bp 251.8131bp,clip,scale=0.6]{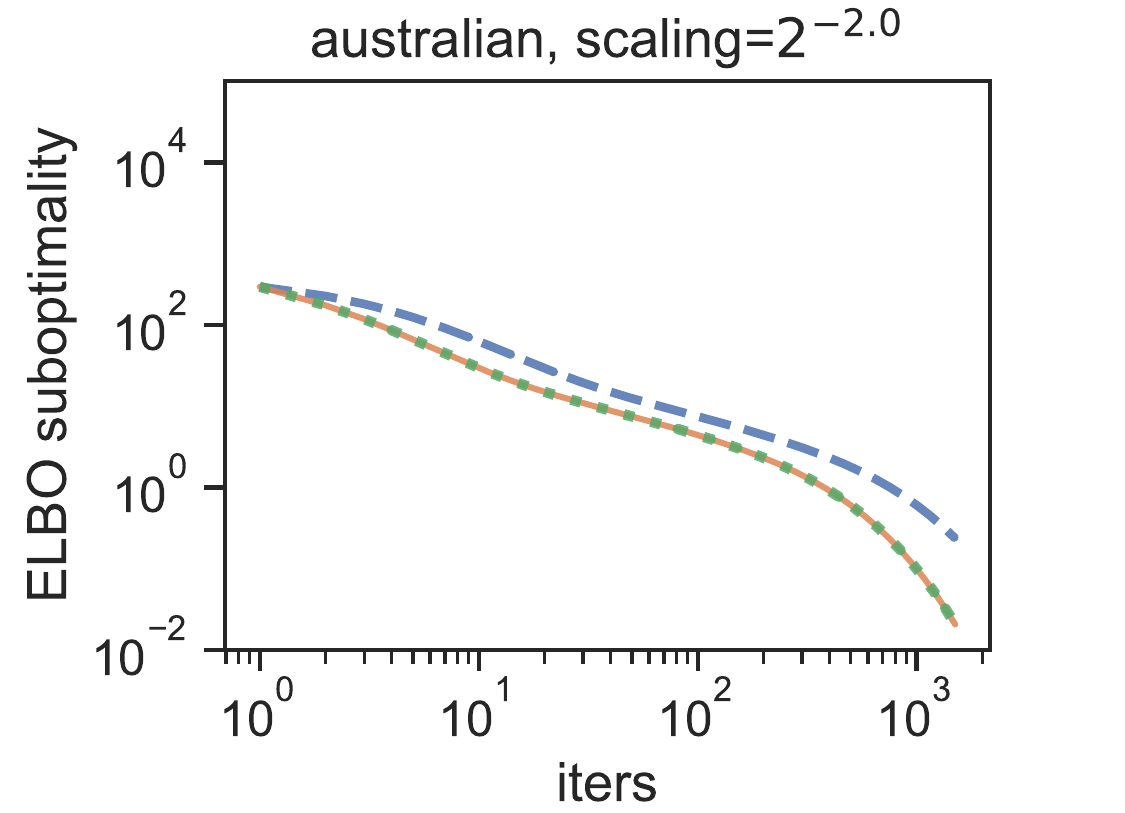}\includegraphics[viewport=60.75bp 0bp 293.625bp 251.8131bp,clip,scale=0.6]{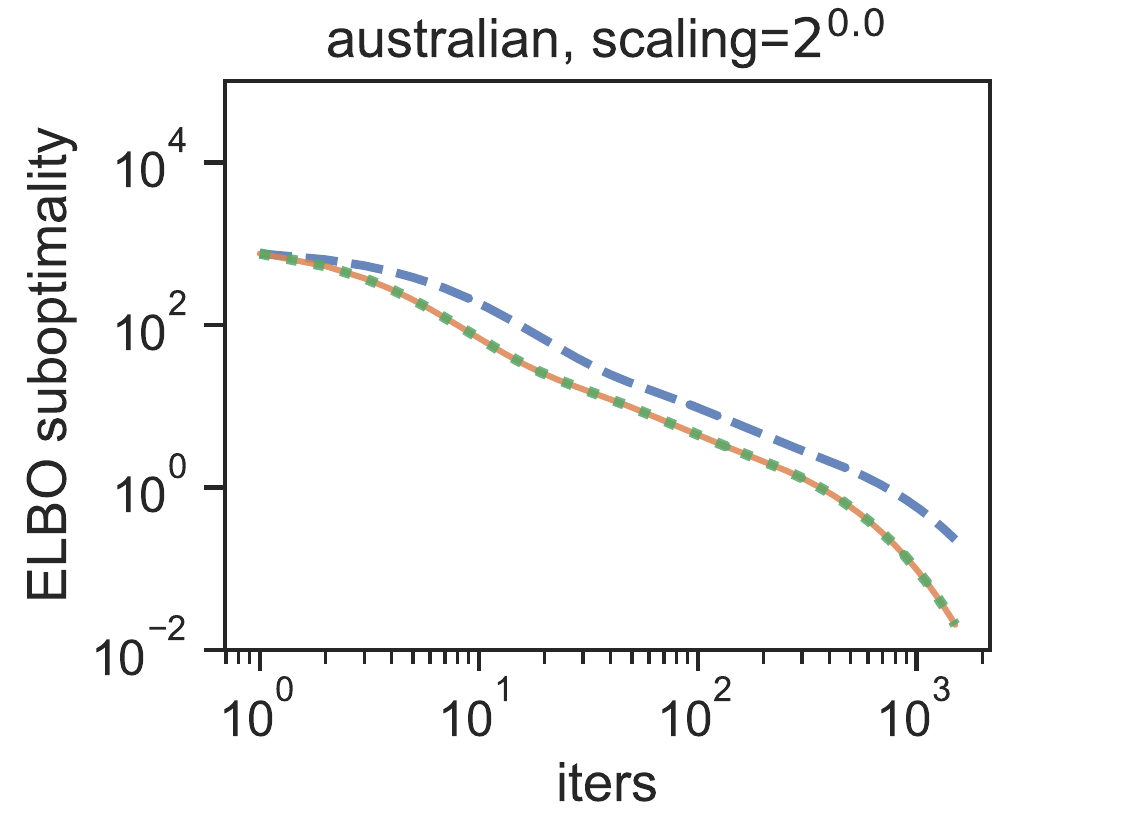}\includegraphics[viewport=60.75bp 0bp 293.625bp 251.8131bp,clip,scale=0.6]{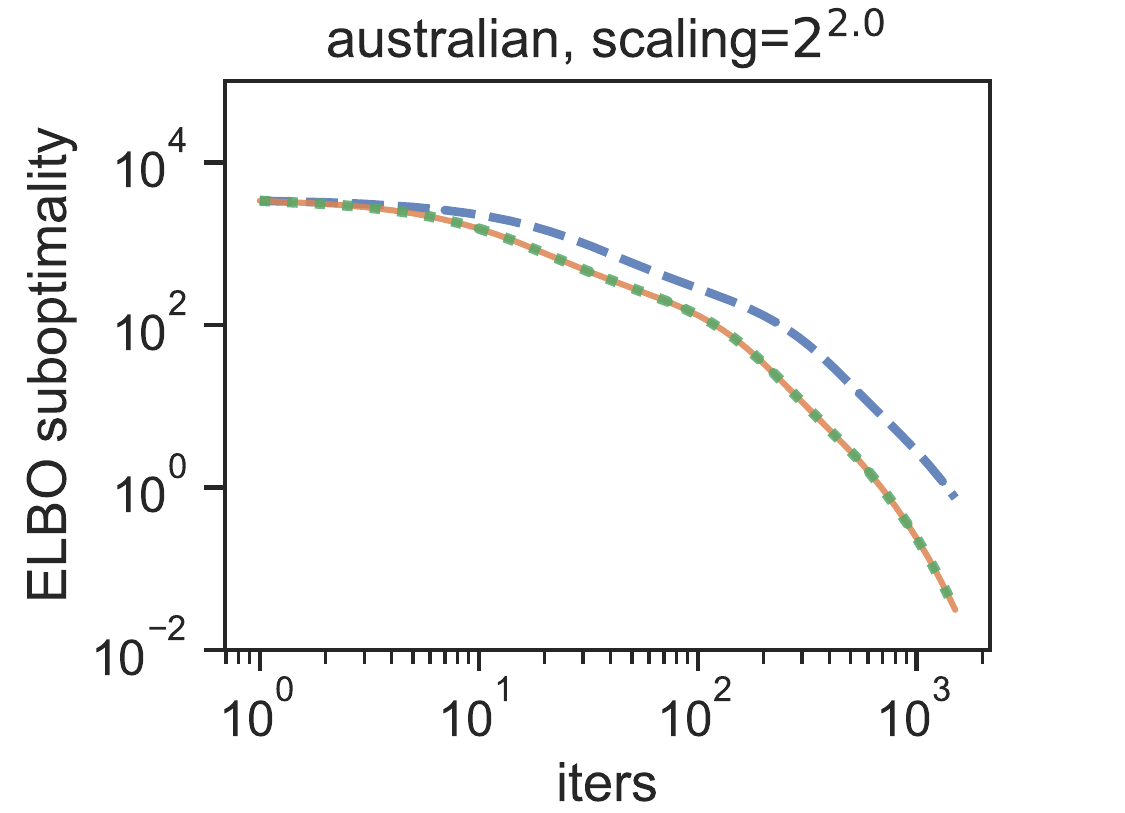}
\par\end{centering}
\begin{centering}
\includegraphics[viewport=0bp 12.14533bp 303.75bp 251.8131bp,clip,scale=0.78]{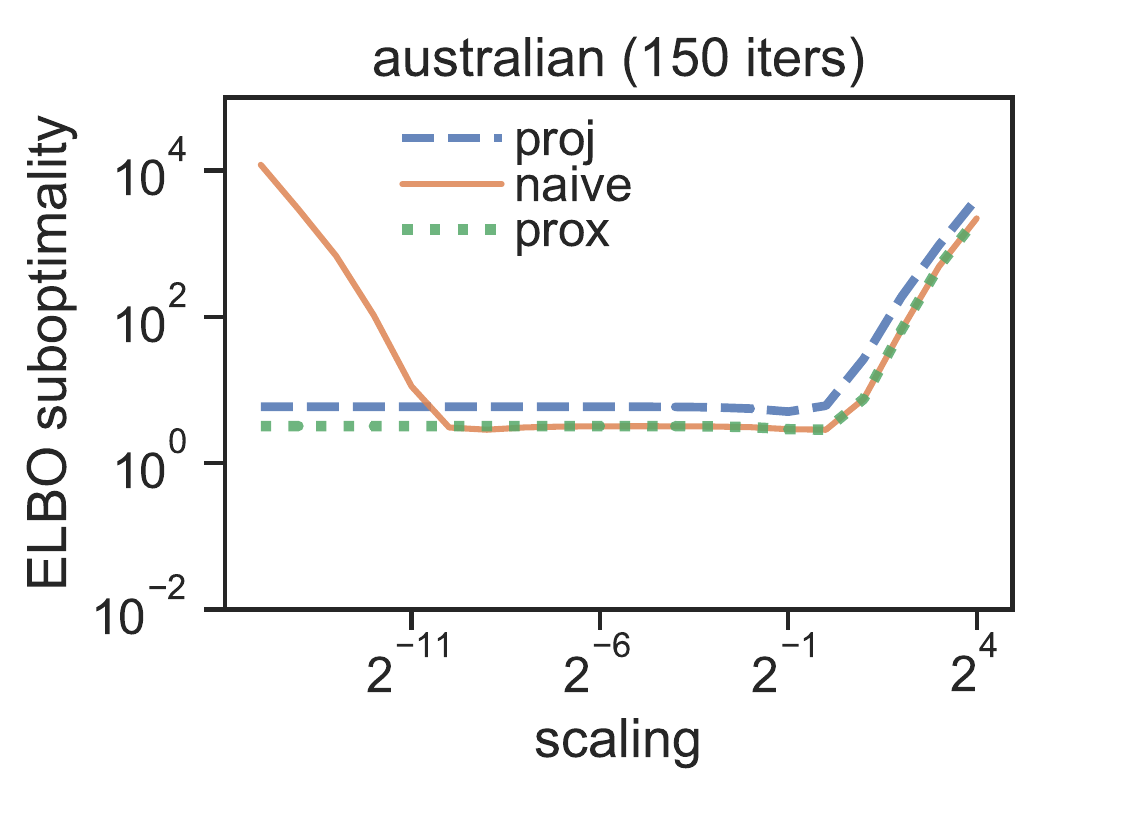}\includegraphics[viewport=60.75bp 12.14533bp 303.75bp 251.8131bp,clip,scale=0.78]{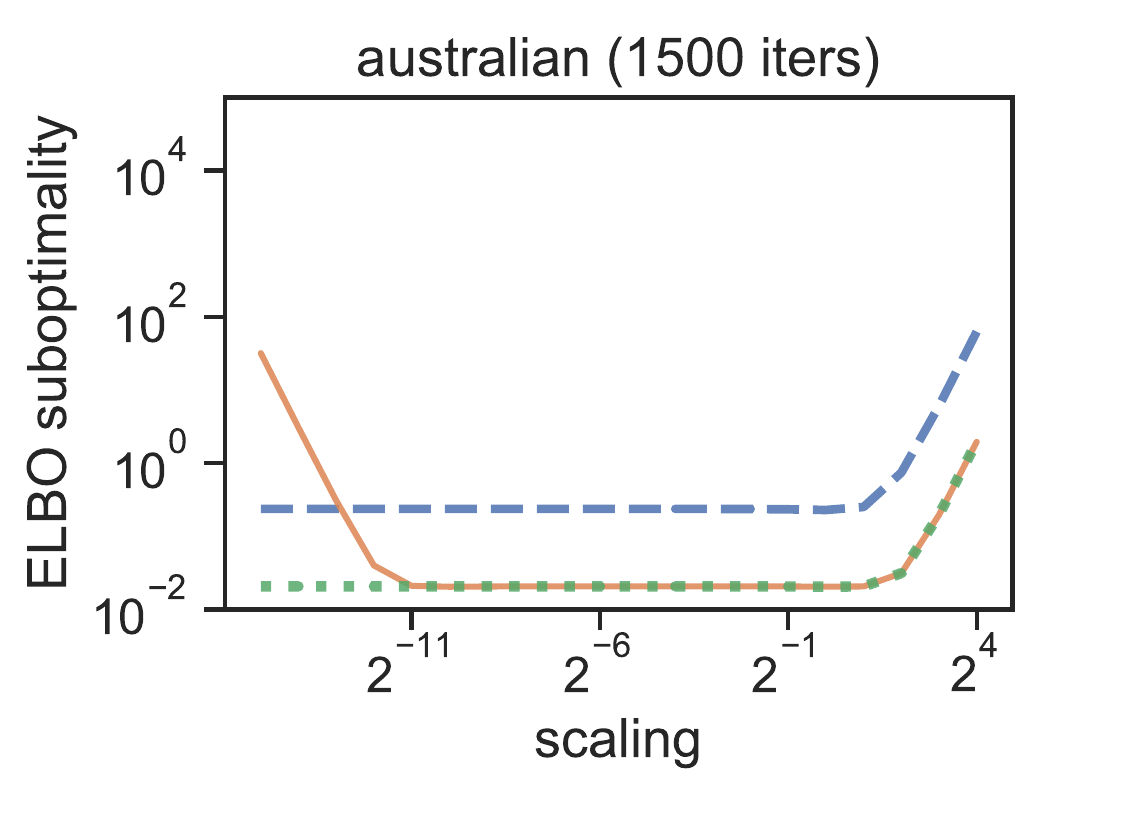}
\par\end{centering}
\caption{Looseness of the objective obtained by naive gradient descent ($\gamma=1/M$),
projected gradient descent ($\gamma=1/\protect\pp{2M}$) and proximal
gradient descent ($\gamma=1/M$). Optimization starts with $\protect\b m=0$
and $C=\rho I$ where $\rho$ is a scaling factor.\label{fig:bigfig-australian}}
\end{figure*}

\begin{figure*}
\begin{centering}
\includegraphics[viewport=0bp 44.5329bp 293.625bp 251.8131bp,clip,scale=0.6]{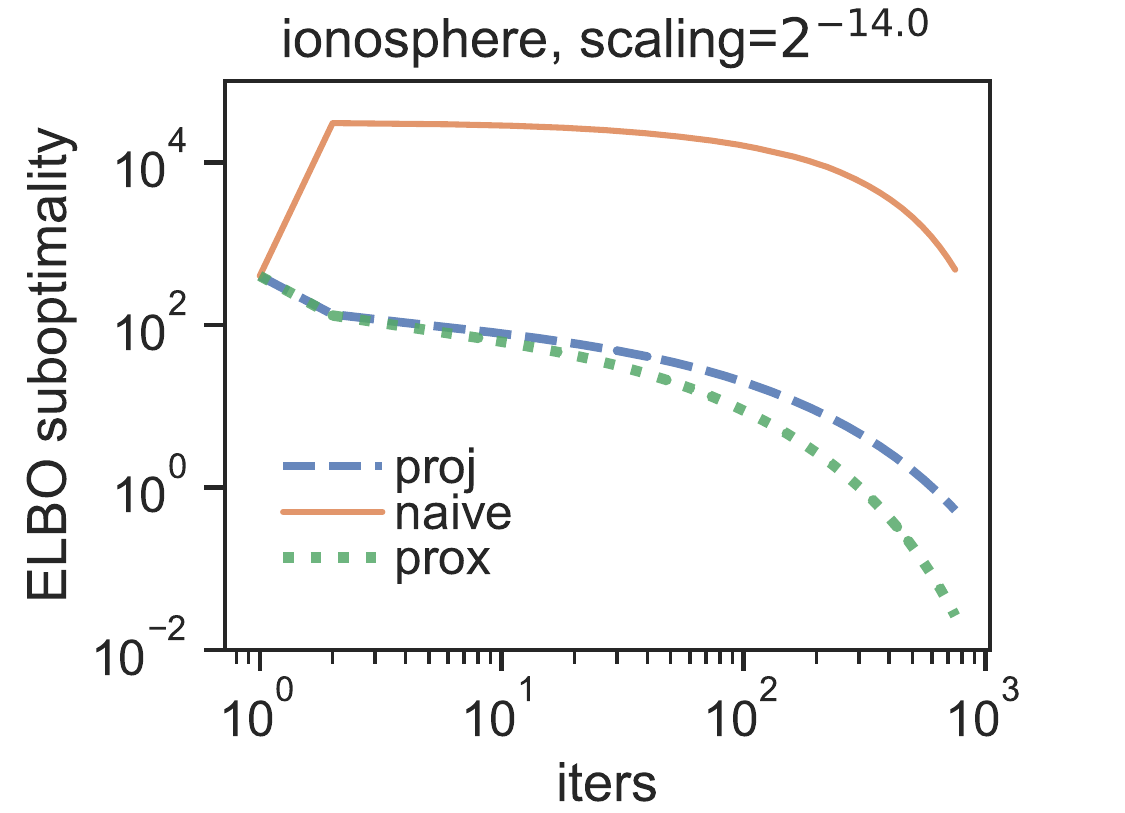}\includegraphics[viewport=60.75bp 44.5329bp 293.625bp 251.8131bp,clip,scale=0.6]{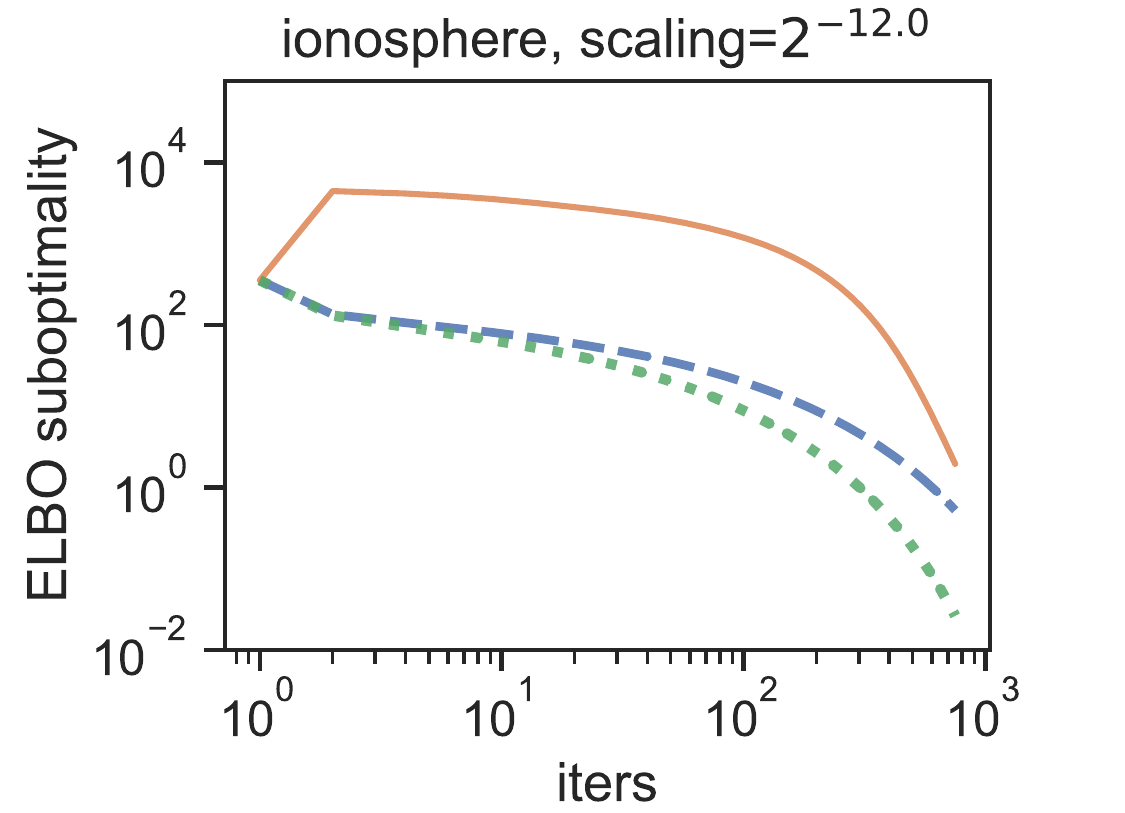}\includegraphics[viewport=60.75bp 44.5329bp 293.625bp 251.8131bp,clip,scale=0.6]{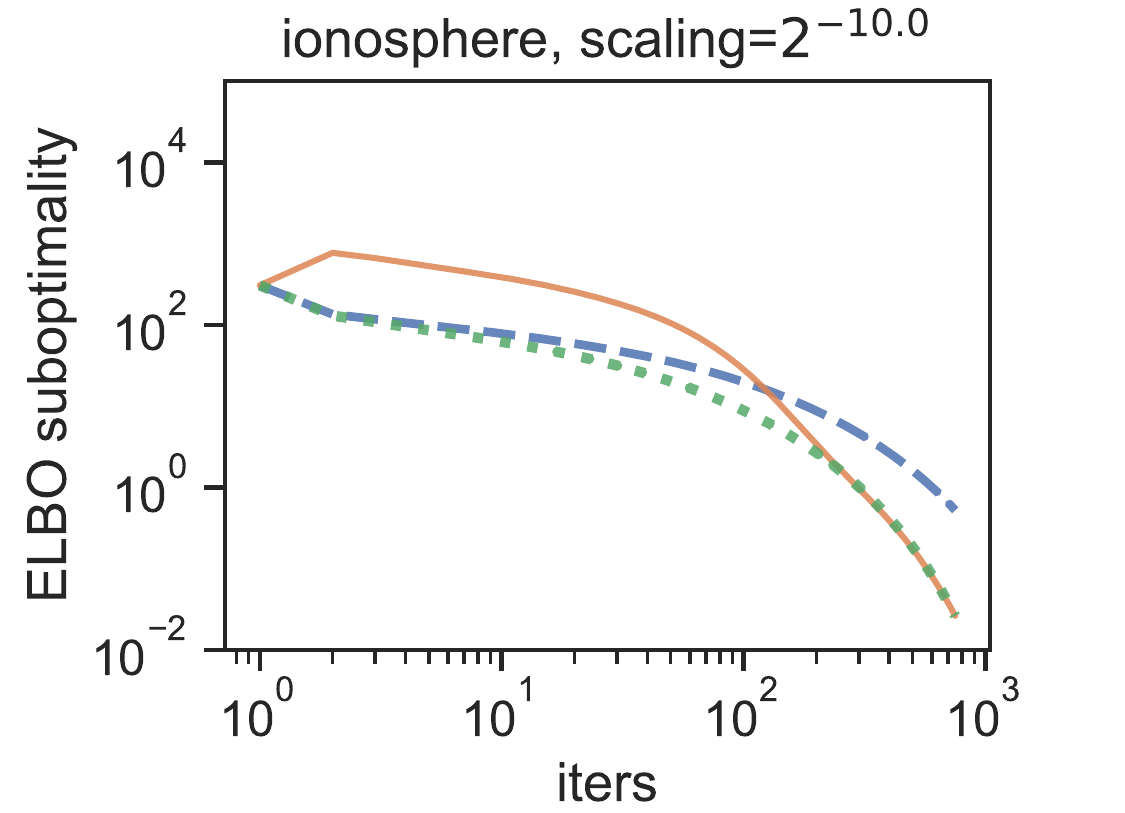}
\par\end{centering}
\begin{centering}
\includegraphics[viewport=0bp 44.5329bp 293.625bp 251.8131bp,clip,scale=0.6]{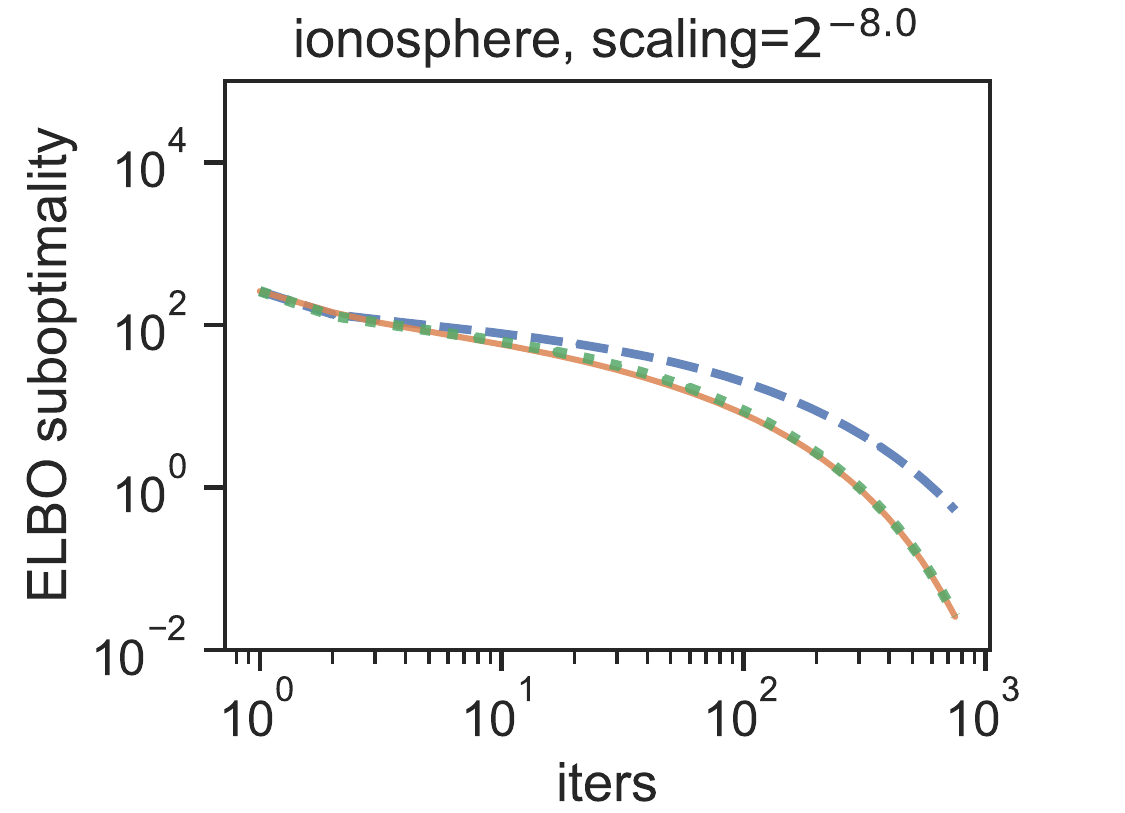}\includegraphics[viewport=60.75bp 44.5329bp 293.625bp 251.8131bp,clip,scale=0.6]{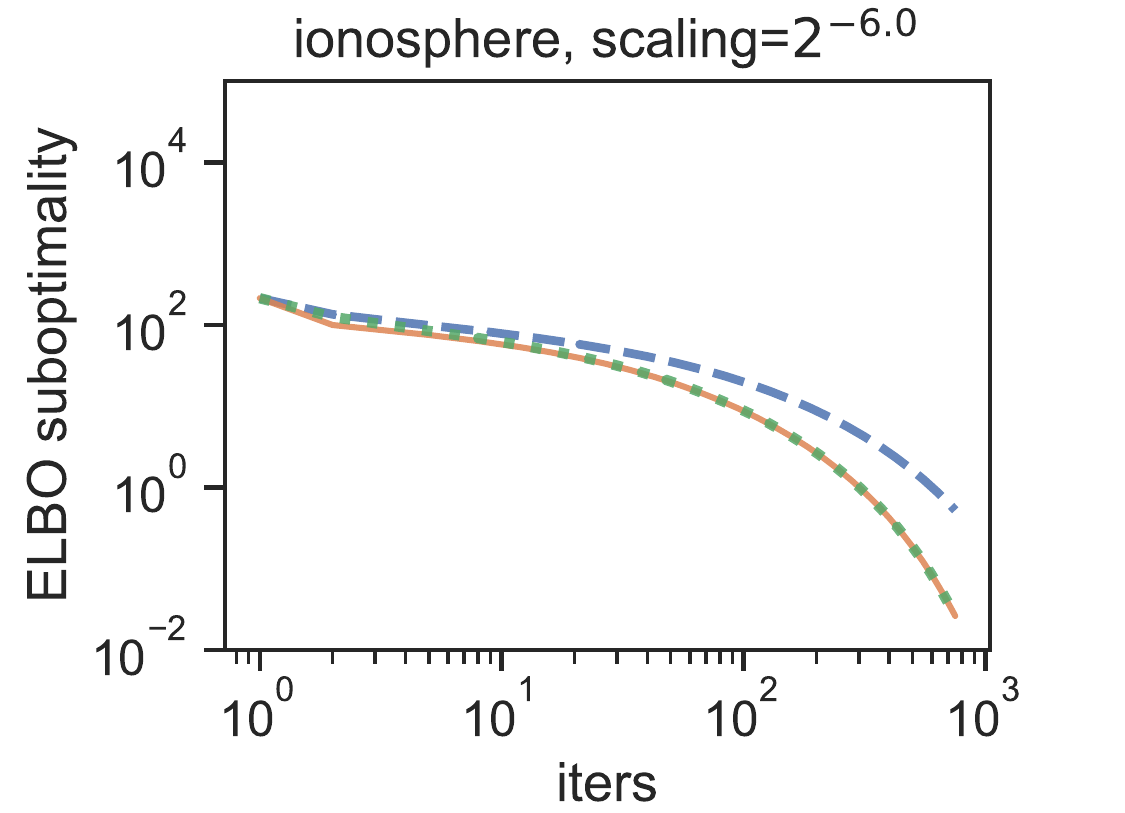}\includegraphics[viewport=60.75bp 44.5329bp 293.625bp 251.8131bp,clip,scale=0.6]{GLMs/figures_individual/ionosphere_11}
\par\end{centering}
\begin{centering}
\includegraphics[viewport=0bp 0bp 293.625bp 251.8131bp,clip,scale=0.6]{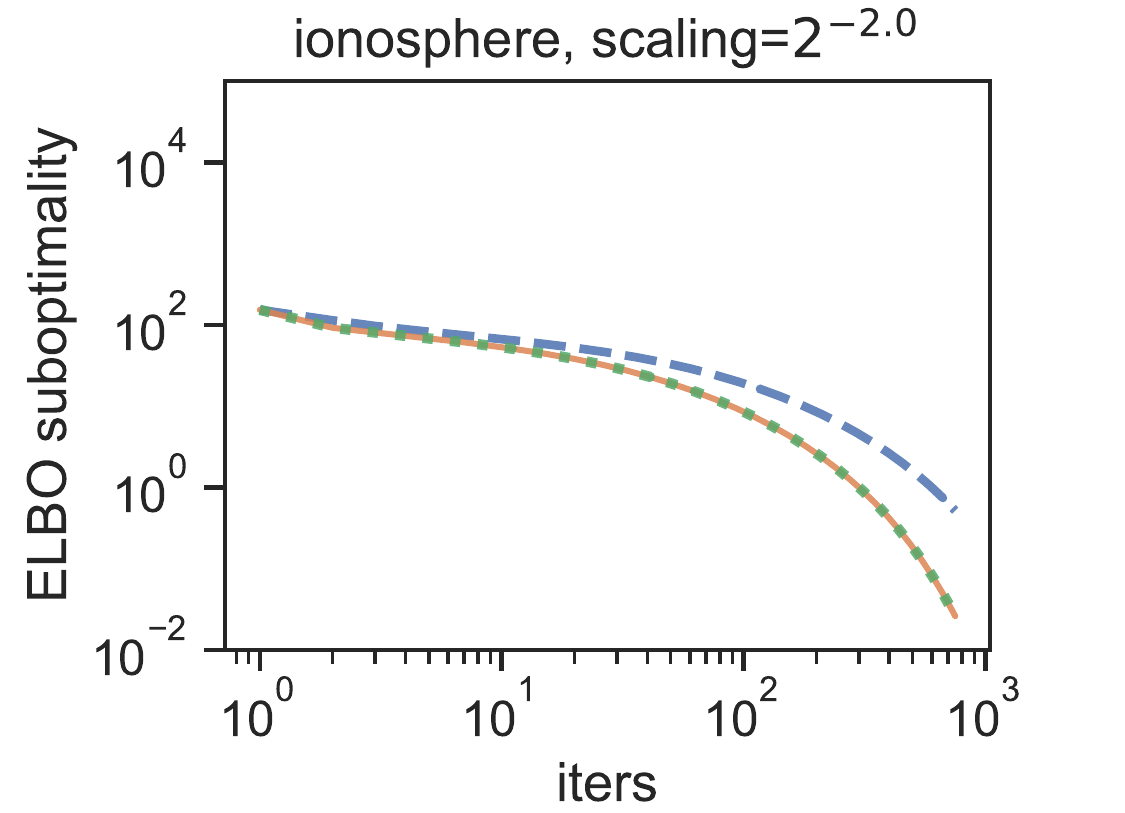}\includegraphics[viewport=60.75bp 0bp 293.625bp 251.8131bp,clip,scale=0.6]{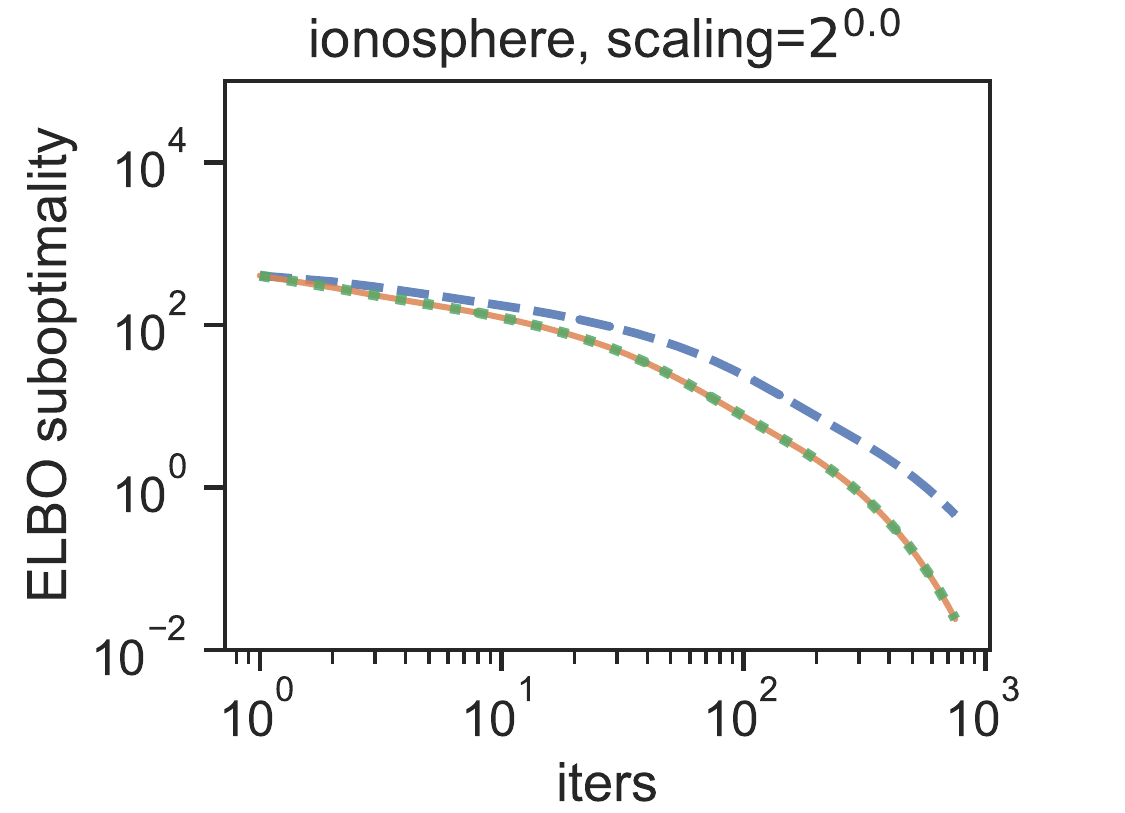}\includegraphics[viewport=60.75bp 0bp 293.625bp 251.8131bp,clip,scale=0.6]{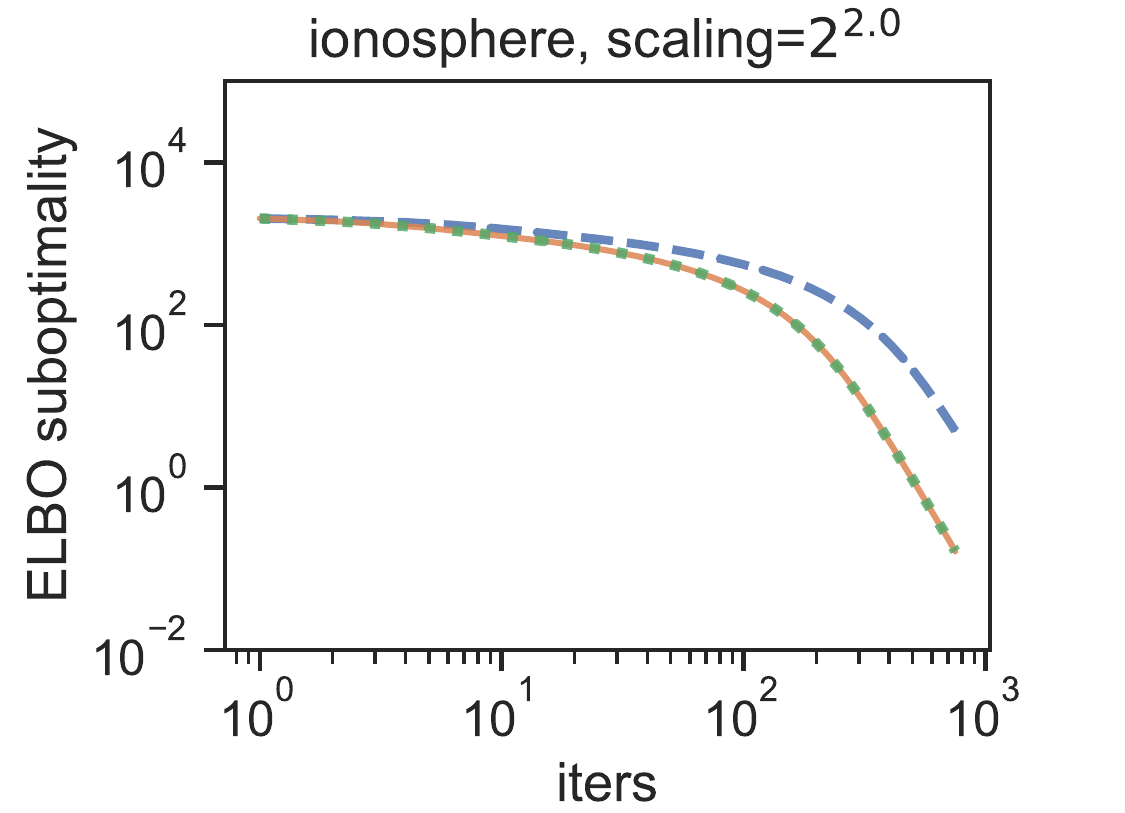}
\par\end{centering}
\begin{centering}
\includegraphics[viewport=0bp 12.14533bp 303.75bp 251.8131bp,clip,scale=0.78]{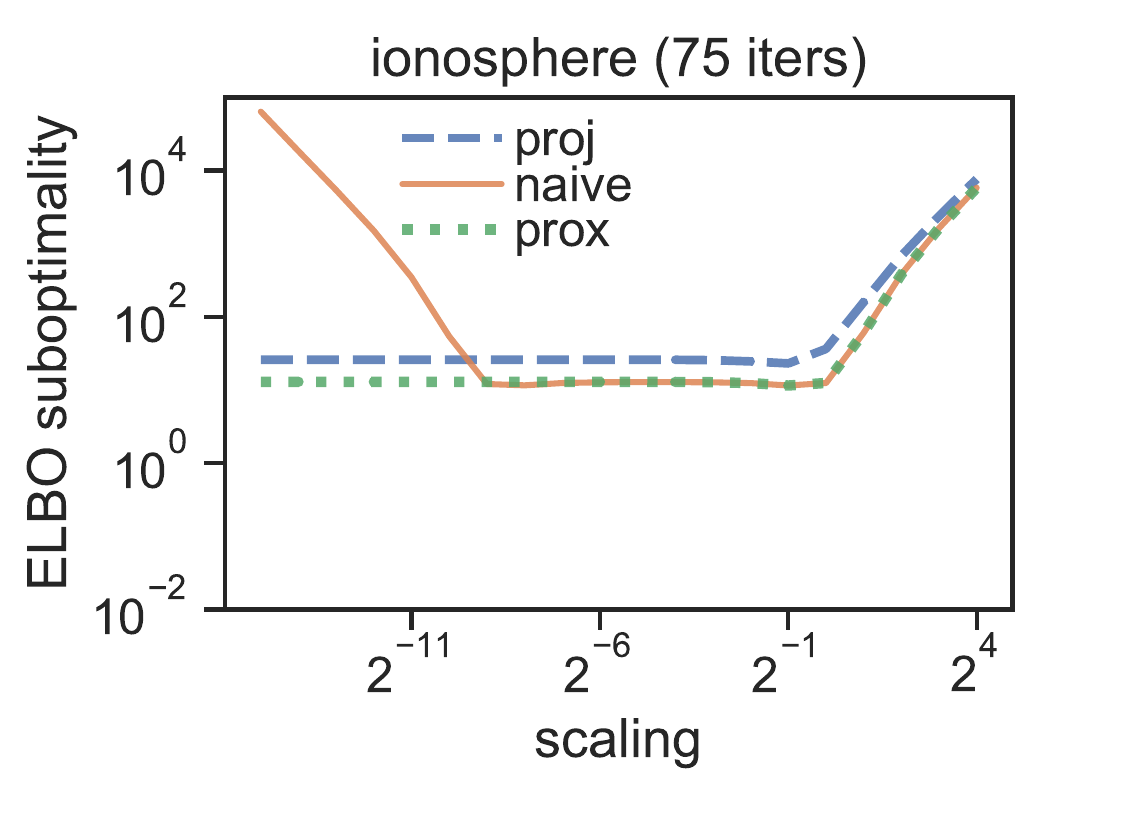}\includegraphics[viewport=60.75bp 12.14533bp 303.75bp 251.8131bp,clip,scale=0.78]{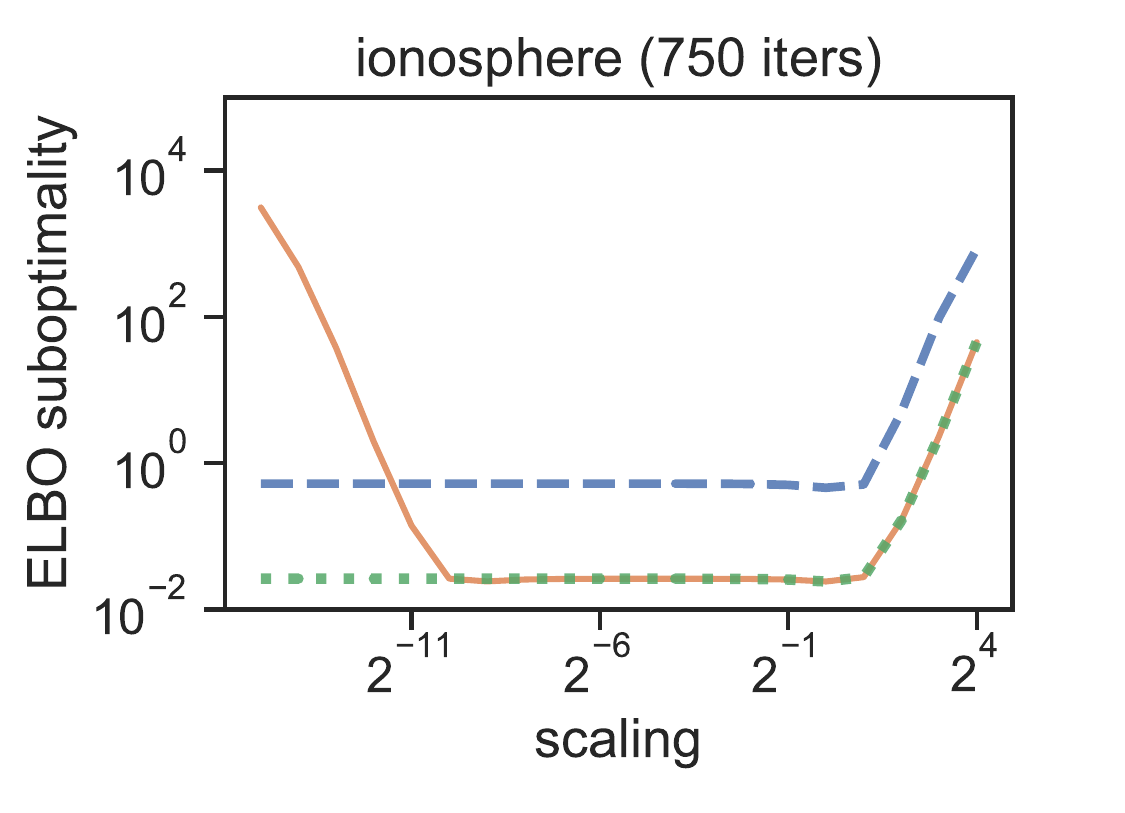}
\par\end{centering}
\caption{Looseness of the objective obtained by naive gradient descent ($\gamma=1/M$),
projected gradient descent ($\gamma=1/\protect\pp{2M}$) and proximal
gradient descent ($\gamma=1/M$). Optimization starts with $\protect\b m=0$
and $C=\rho I$ where $\rho$ is a scaling factor.\label{fig:bigfig-ionospheree}}
\end{figure*}

\begin{figure*}
\begin{centering}
\includegraphics[viewport=0bp 44.5329bp 293.625bp 251.8131bp,clip,scale=0.6]{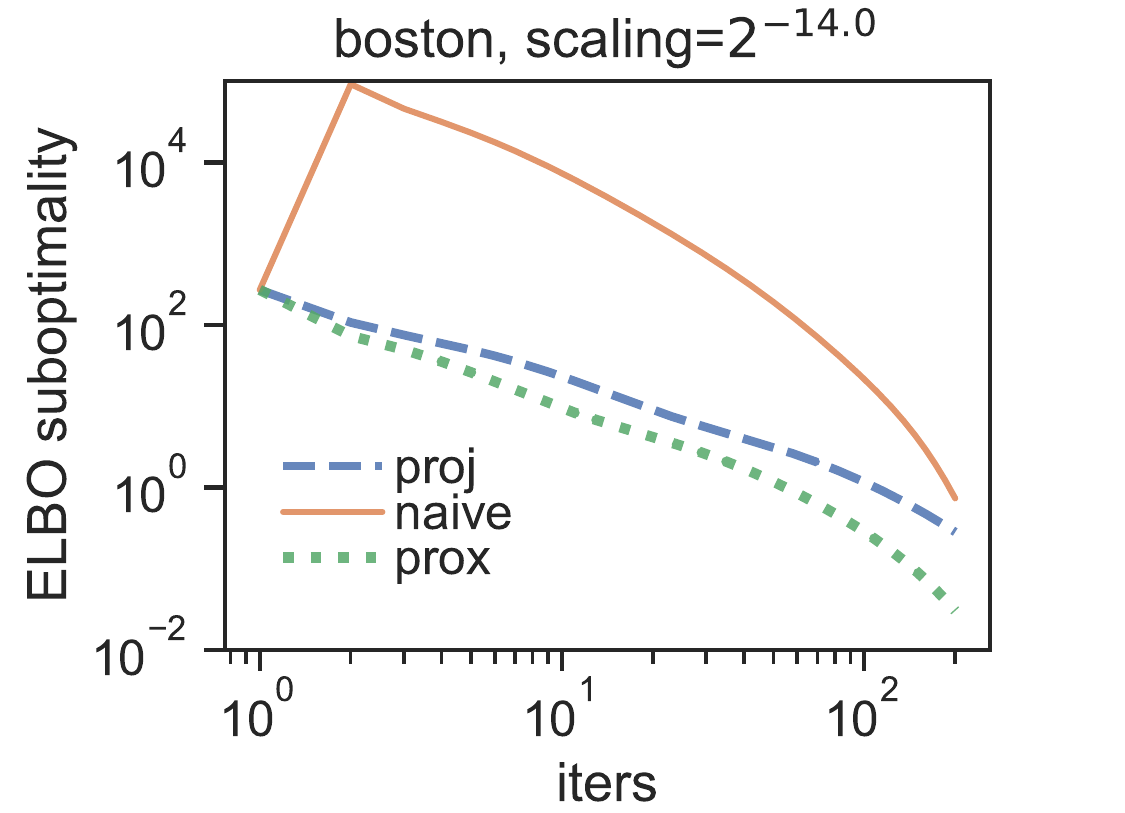}\includegraphics[viewport=60.75bp 44.5329bp 293.625bp 251.8131bp,clip,scale=0.6]{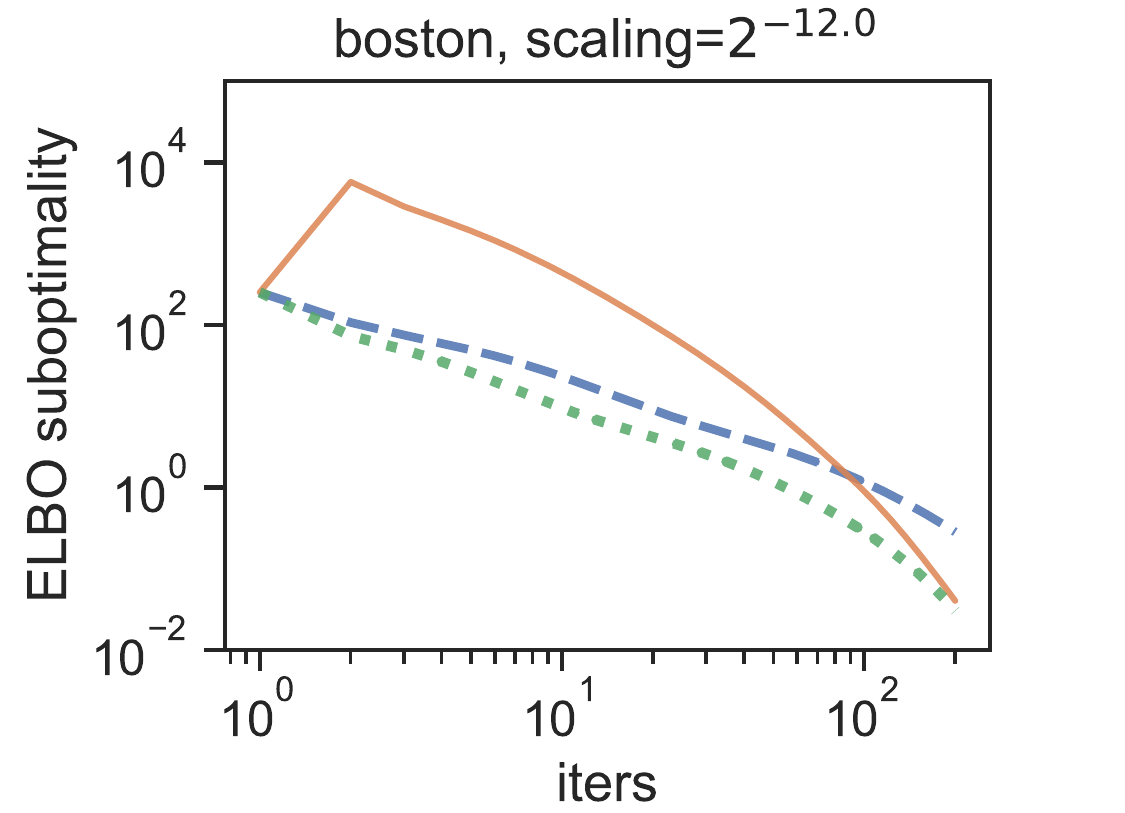}\includegraphics[viewport=60.75bp 44.5329bp 293.625bp 251.8131bp,clip,scale=0.6]{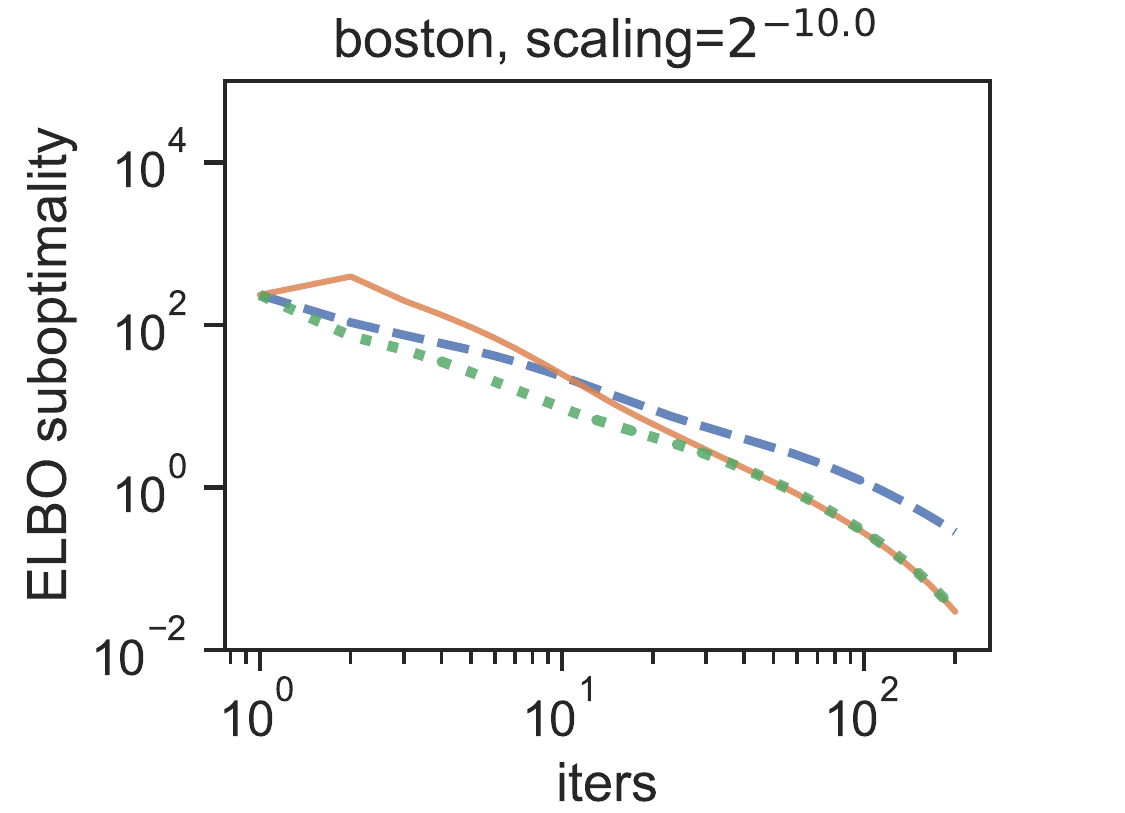}
\par\end{centering}
\begin{centering}
\includegraphics[viewport=0bp 44.5329bp 293.625bp 251.8131bp,clip,scale=0.6]{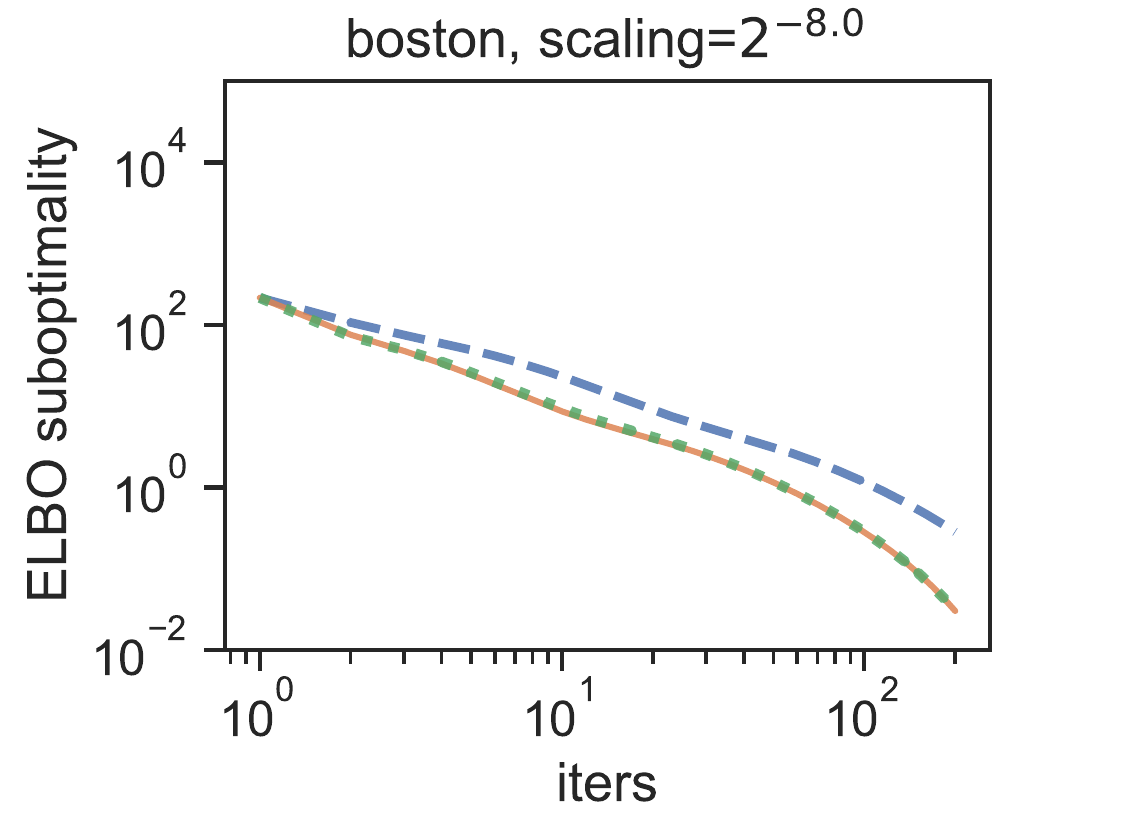}\includegraphics[viewport=60.75bp 44.5329bp 293.625bp 251.8131bp,clip,scale=0.6]{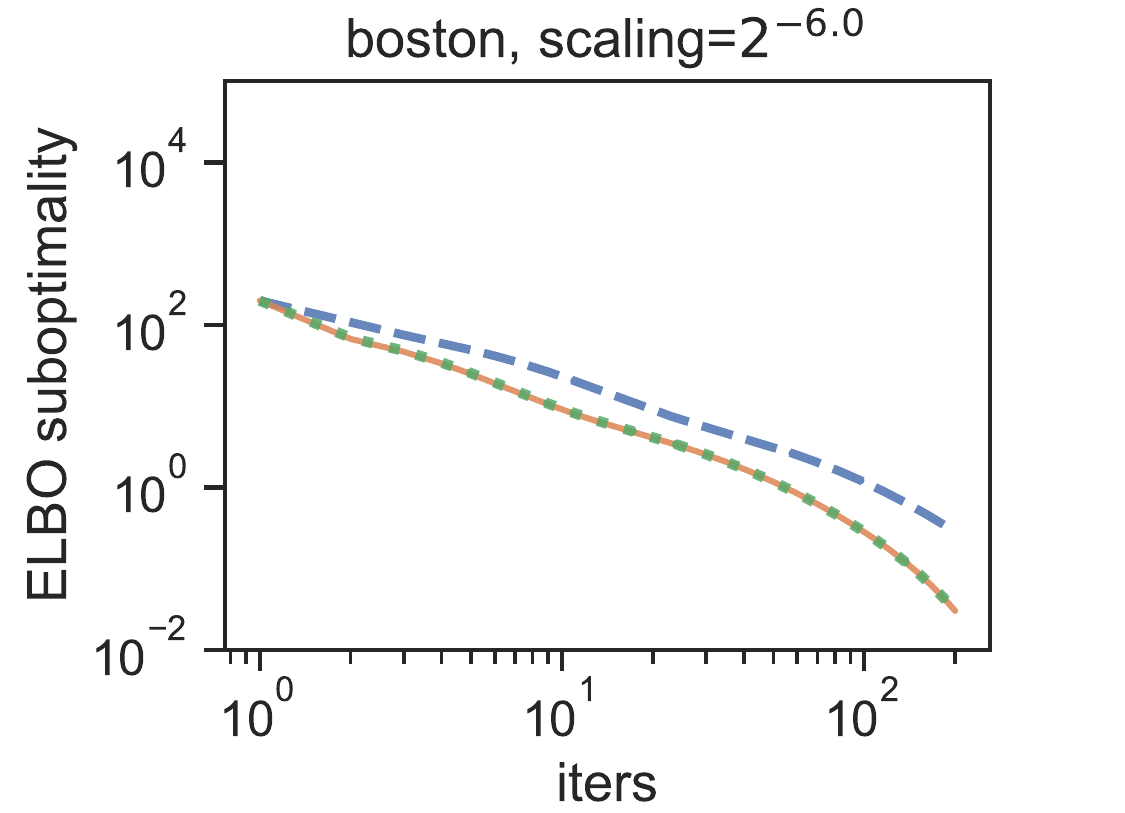}\includegraphics[viewport=60.75bp 44.5329bp 293.625bp 251.8131bp,clip,scale=0.6]{GLMs/figures_individual/boston_11}
\par\end{centering}
\begin{centering}
\includegraphics[viewport=0bp 0bp 293.625bp 251.8131bp,clip,scale=0.6]{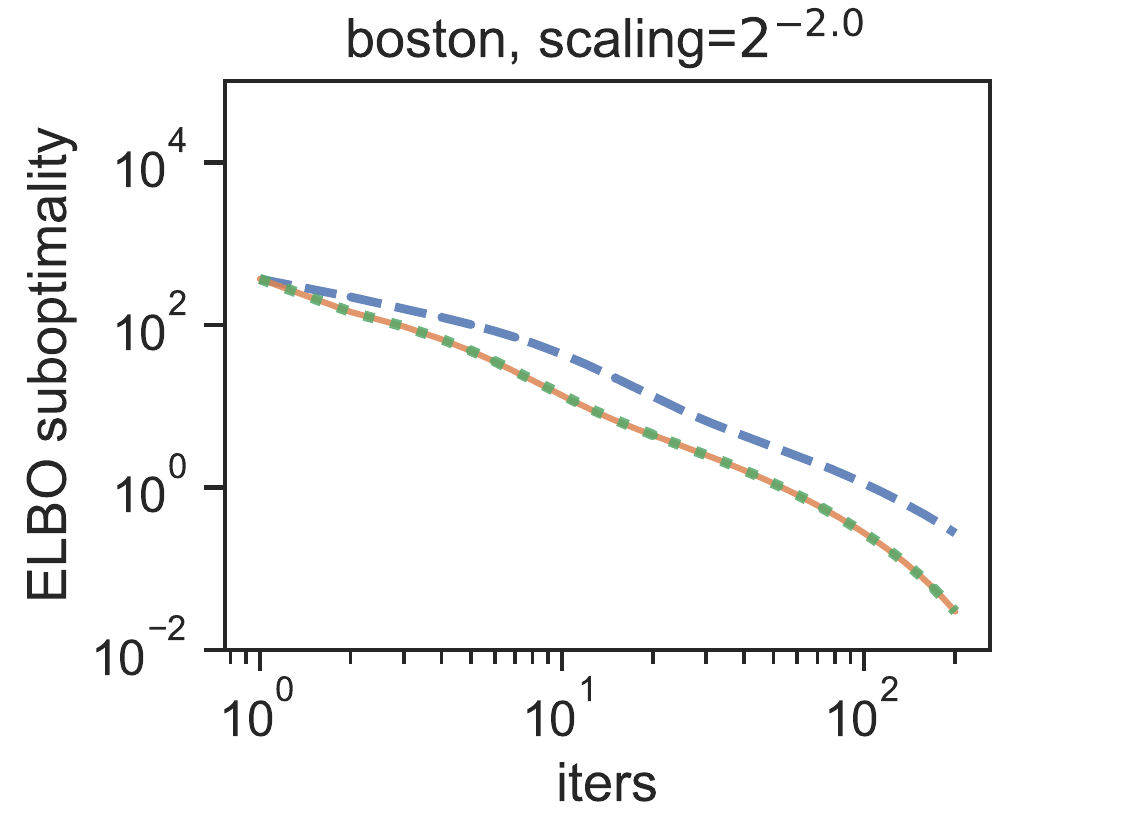}\includegraphics[viewport=60.75bp 0bp 293.625bp 251.8131bp,clip,scale=0.6]{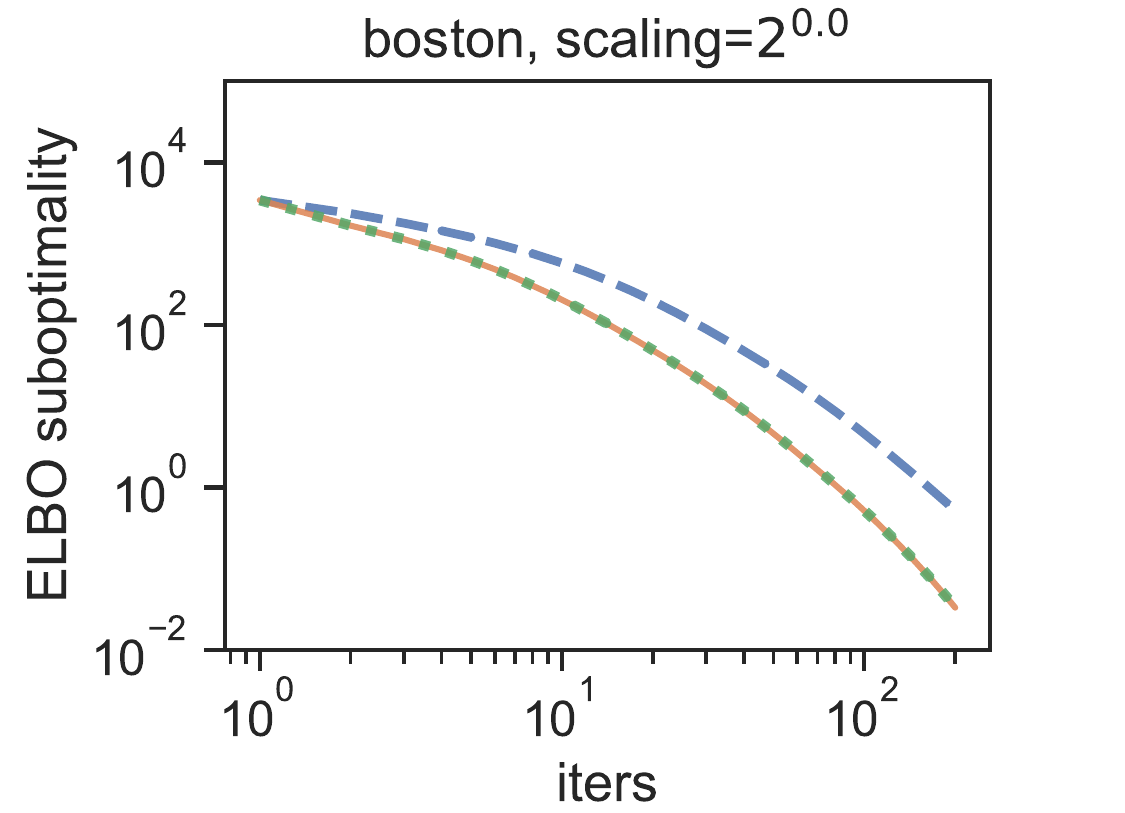}\includegraphics[viewport=60.75bp 0bp 293.625bp 251.8131bp,clip,scale=0.6]{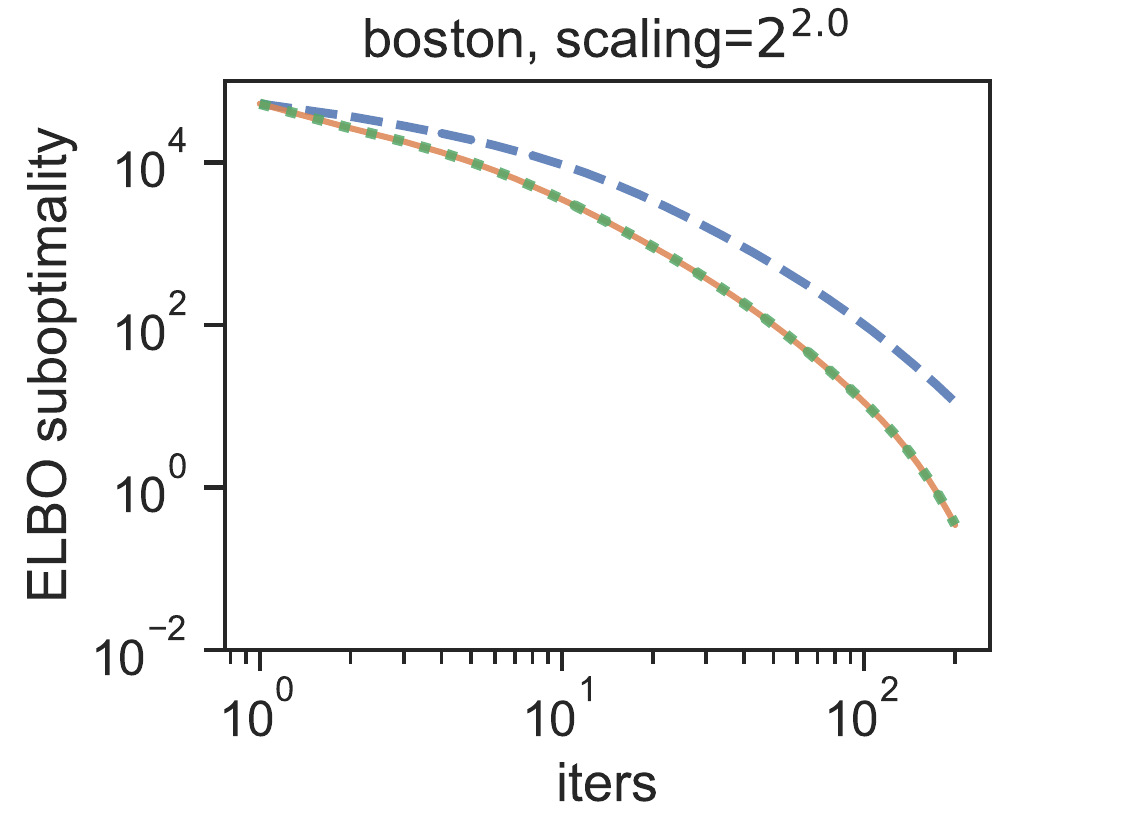}
\par\end{centering}
\begin{centering}
\includegraphics[viewport=0bp 12.14533bp 303.75bp 251.8131bp,clip,scale=0.78]{GLMs/figures_final/boston_10}\includegraphics[viewport=60.75bp 12.14533bp 303.75bp 251.8131bp,clip,scale=0.78]{GLMs/figures_final/boston_1}
\par\end{centering}
\caption{Looseness of the objective obtained by naive gradient descent ($\gamma=1/M$),
projected gradient descent ($\gamma=1/\protect\pp{2M}$) and proximal
gradient descent ($\gamma=1/M$). Optimization starts with $\protect\b m=0$
and $C=\rho I$ where $\rho$ is a scaling factor.\label{fig:bigfig-boston}}
\end{figure*}

\begin{figure*}
\begin{centering}
\includegraphics[viewport=0bp 44.5329bp 293.625bp 251.8131bp,clip,scale=0.6]{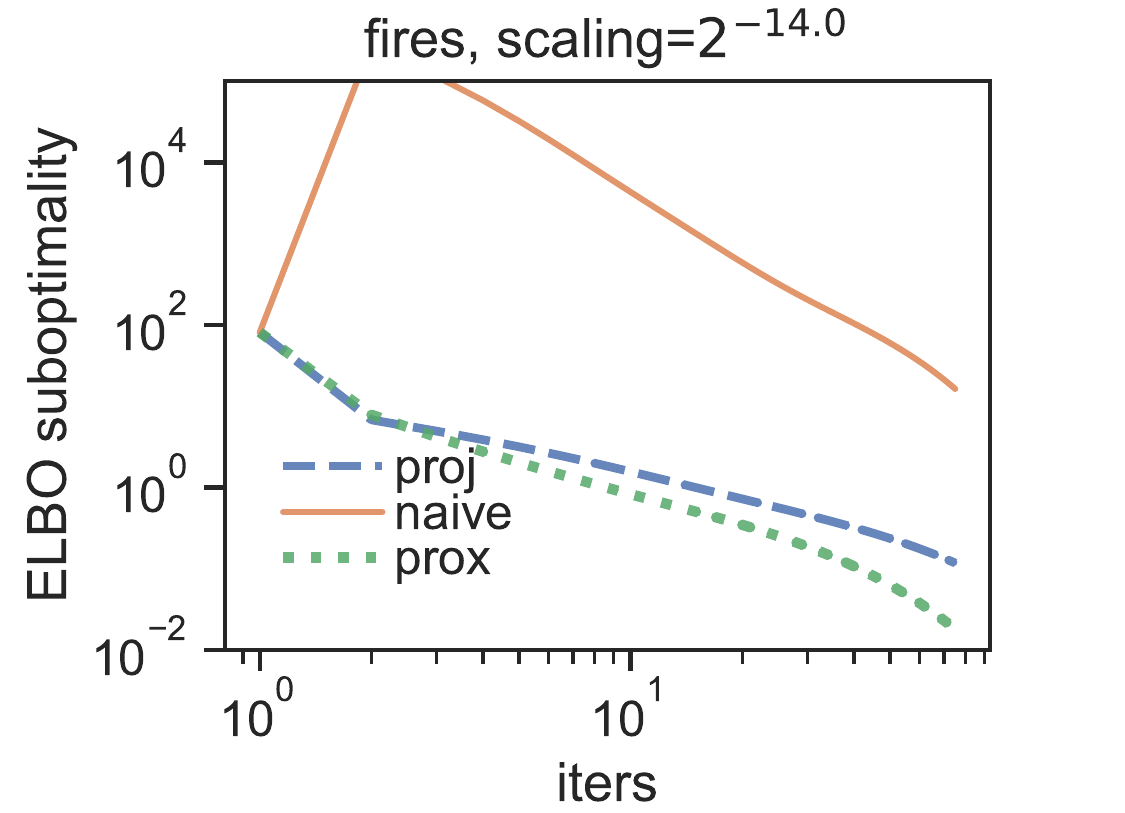}\includegraphics[viewport=60.75bp 44.5329bp 293.625bp 251.8131bp,clip,scale=0.6]{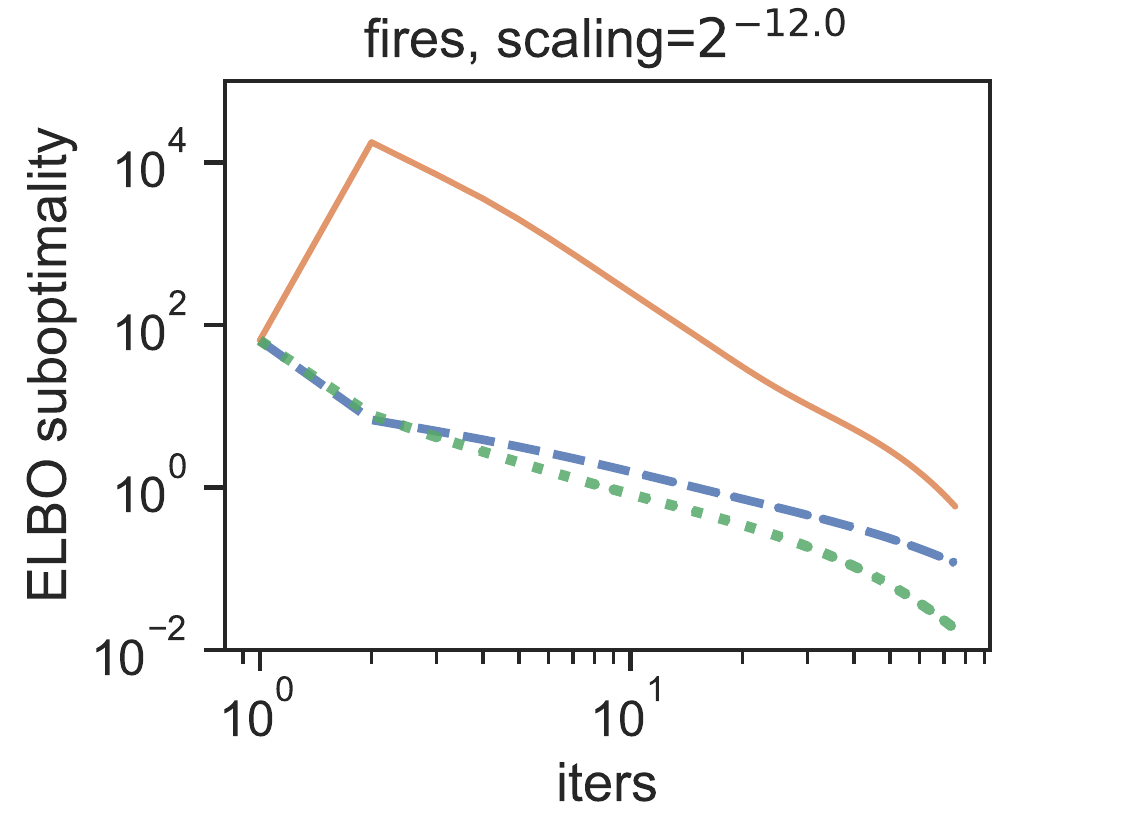}\includegraphics[viewport=60.75bp 44.5329bp 293.625bp 251.8131bp,clip,scale=0.6]{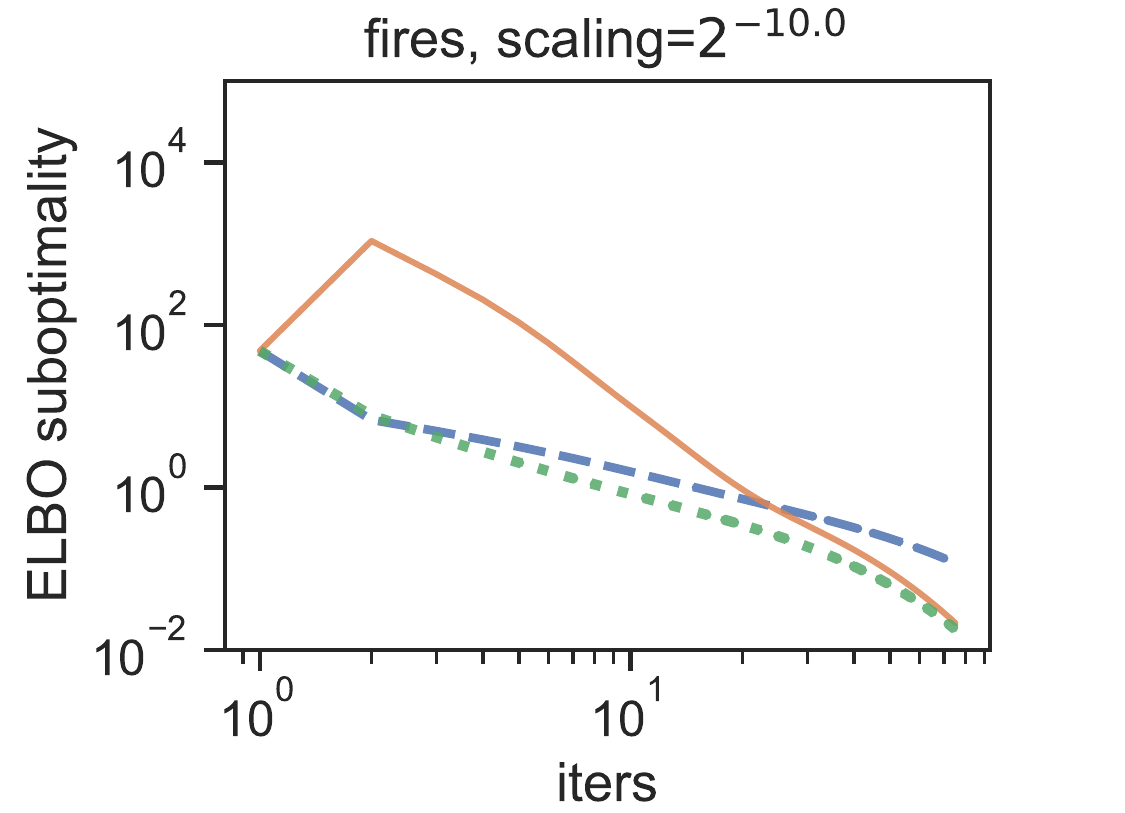}
\par\end{centering}
\begin{centering}
\includegraphics[viewport=0bp 44.5329bp 293.625bp 251.8131bp,clip,scale=0.6]{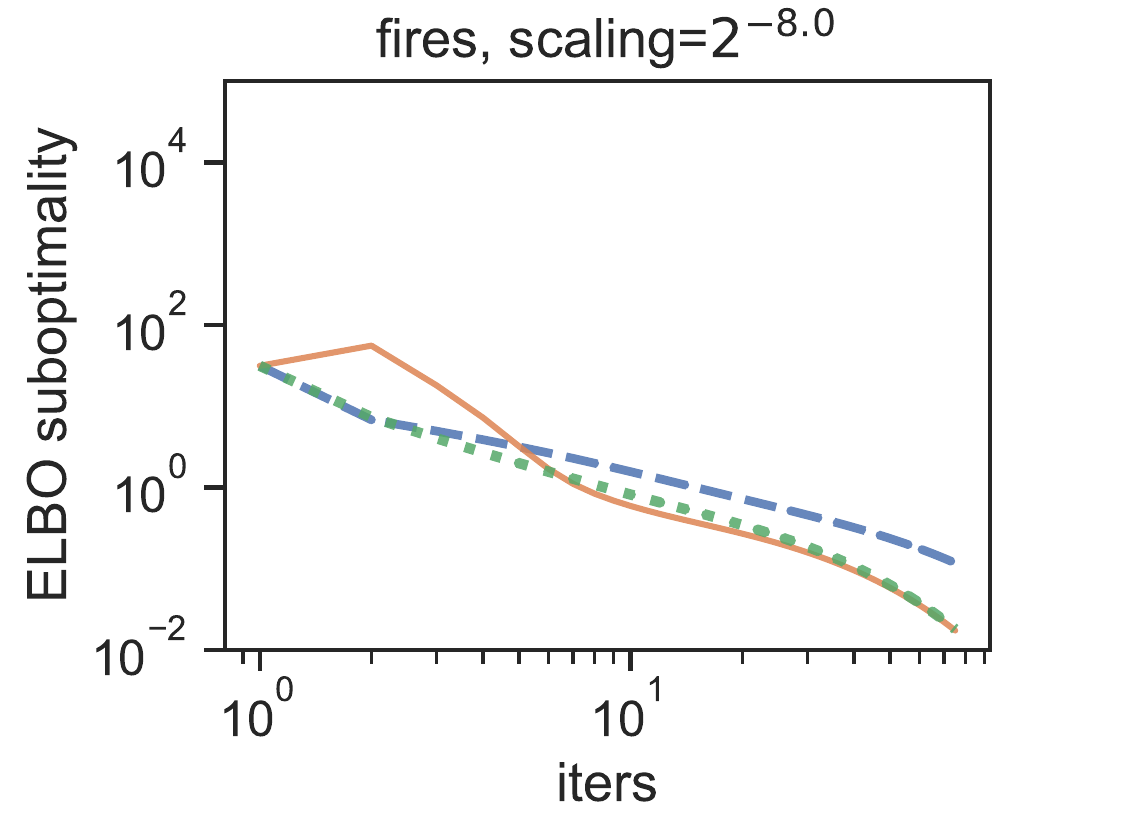}\includegraphics[viewport=60.75bp 44.5329bp 293.625bp 251.8131bp,clip,scale=0.6]{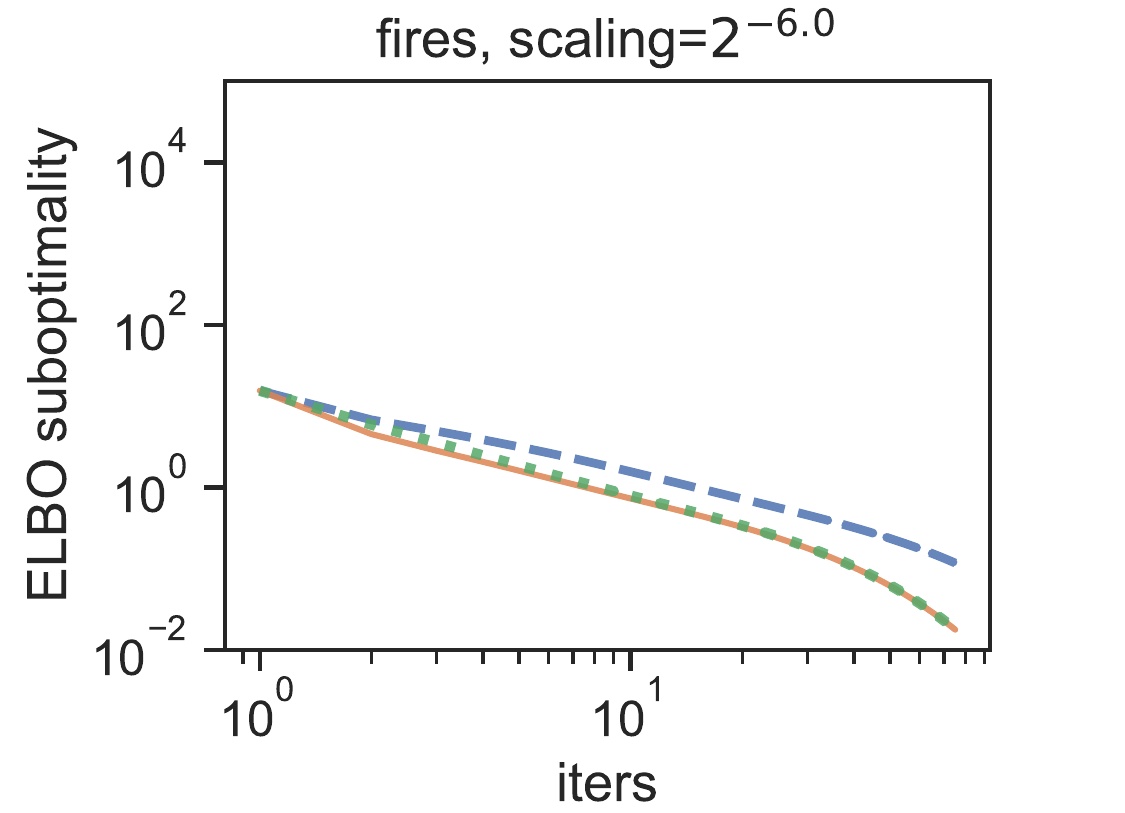}\includegraphics[viewport=60.75bp 44.5329bp 293.625bp 251.8131bp,clip,scale=0.6]{GLMs/figures_individual/fires_11}
\par\end{centering}
\begin{centering}
\includegraphics[viewport=0bp 0bp 293.625bp 251.8131bp,clip,scale=0.6]{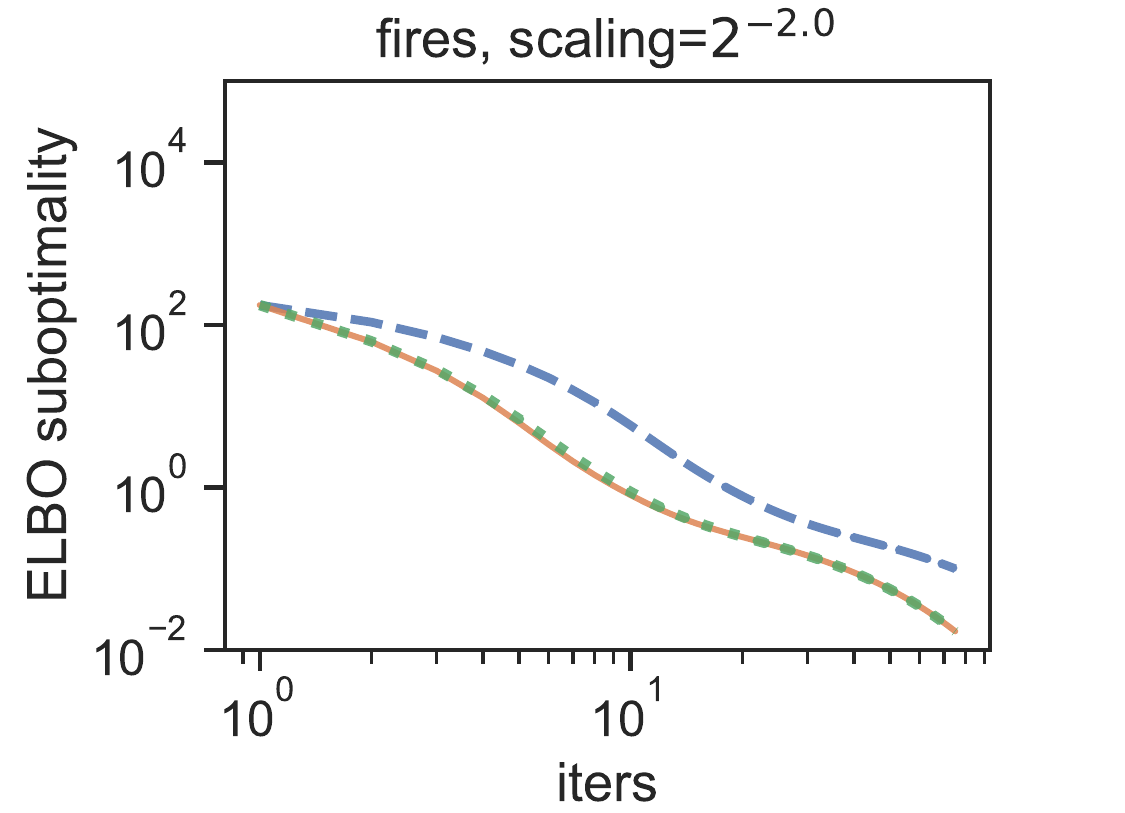}\includegraphics[viewport=60.75bp 0bp 293.625bp 251.8131bp,clip,scale=0.6]{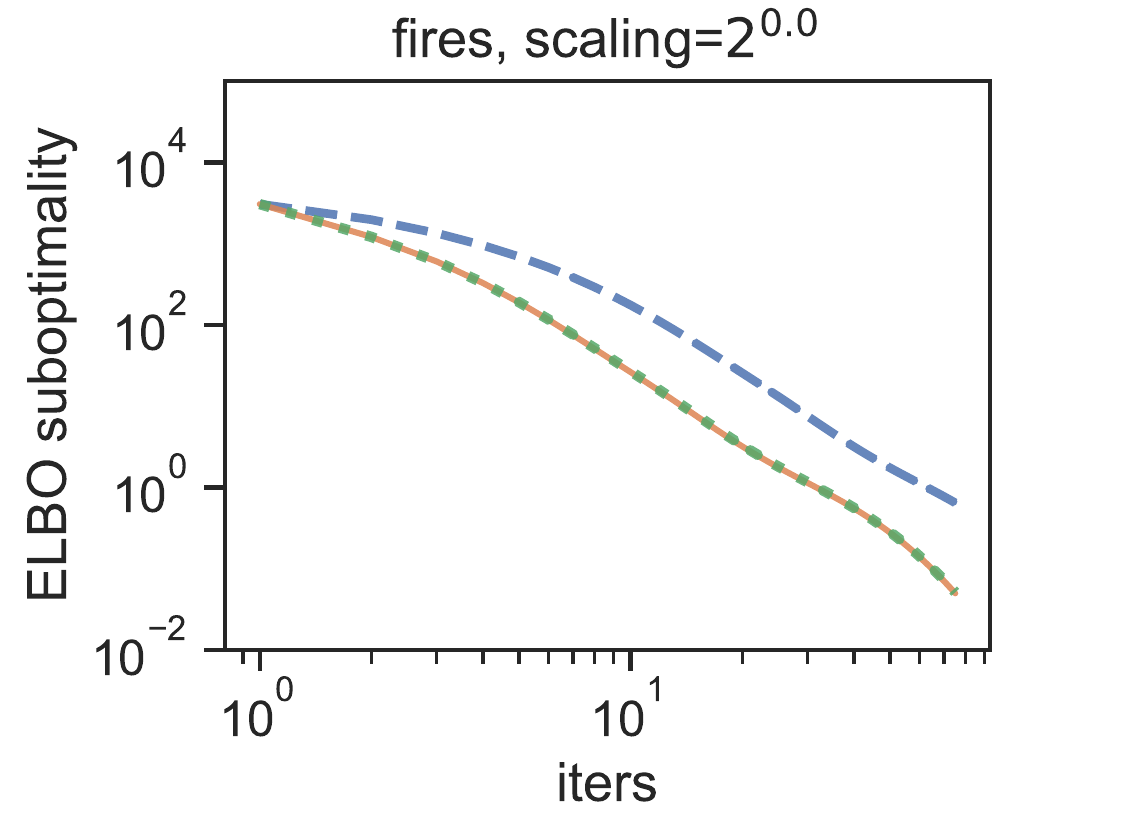}\includegraphics[viewport=60.75bp 0bp 293.625bp 251.8131bp,clip,scale=0.6]{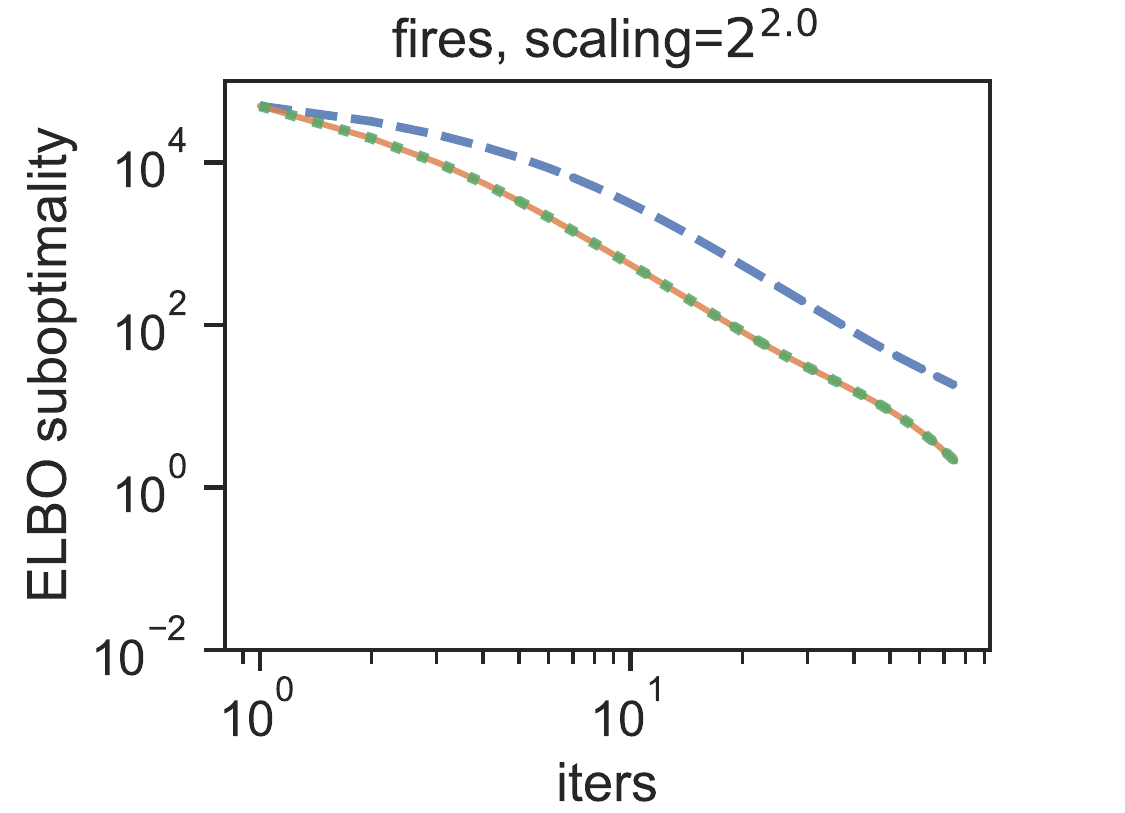}
\par\end{centering}
\begin{centering}
\includegraphics[viewport=0bp 12.14533bp 303.75bp 251.8131bp,clip,scale=0.78]{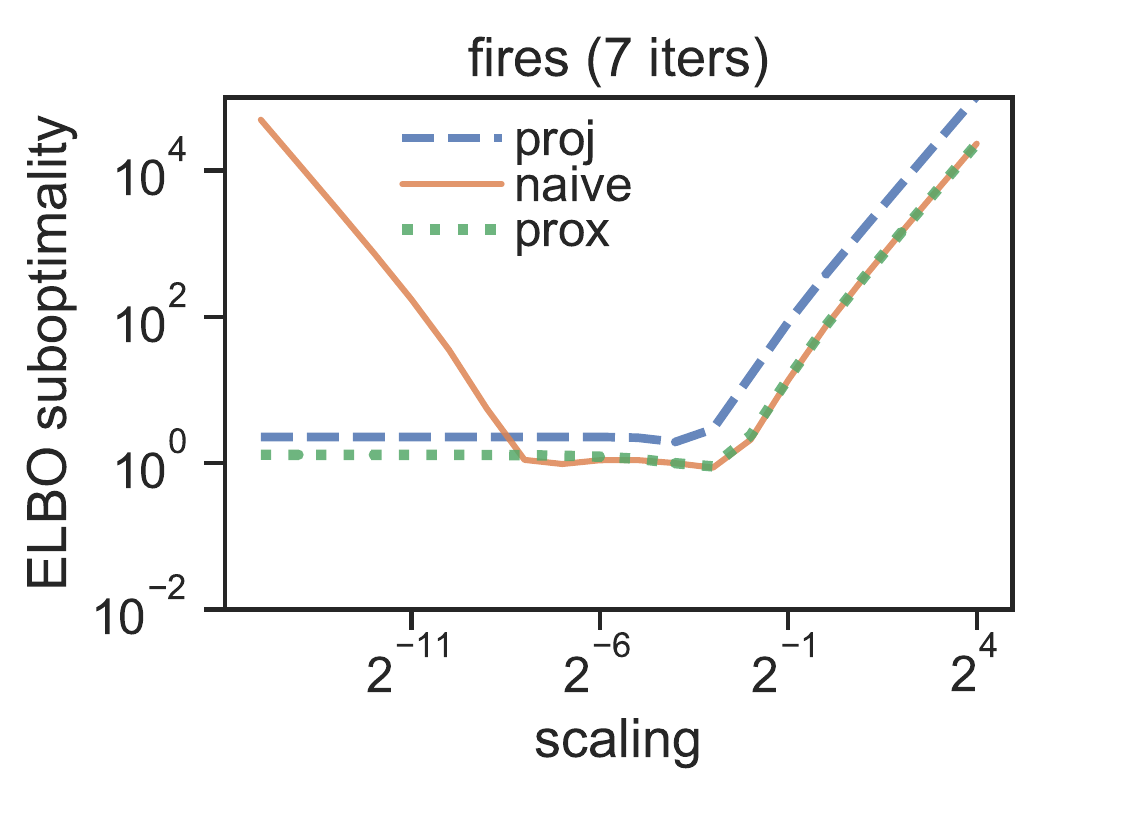}\includegraphics[viewport=60.75bp 12.14533bp 303.75bp 251.8131bp,clip,scale=0.78]{GLMs/figures_final/fires_1}
\par\end{centering}
\caption{Looseness of the objective obtained by naive gradient descent ($\gamma=1/M$),
projected gradient descent ($\gamma=1/\protect\pp{2M}$) and proximal
gradient descent ($\gamma=1/M$). Optimization starts with $\protect\b m=0$
and $C=\rho I$ where $\rho$ is a scaling factor.\label{fig:bigfig-fires}}
\end{figure*}

\clearpage{}

\section{Proofs for Technical Lemmas\label{sec:Technical-Lemmas}}

This section gives proofs for the technical lemmas used in the main
result. Firstly, we show that $\angs{\cdot,\cdot}_{s}$ is a valid
inner-product.

\validinnerproduct*
\begin{proof}
The space of square integrable functions is $\left\{ \b a:\R^{d}\rightarrow\R^{k}\ \vert\ \E_{\ur\sim s}a_{i}\pp{\ur}^{2}\leq\infty\ \forall i\in\{1,...,k\}\right\} .$
Since each component $a_{i}\pp{\u}$ and $b_{i}\pp{\u}$ is square-integrable
with respect to $s(\u)$ we know (by Cauchy-Schwarz) that $\E_{\ur\sim s}a_{i}(\ur)b_{i}\pp{\ur}\leq\sqrt{\E_{\ur\sim s}a_{i}(\ur)^{2}}\sqrt{\E_{\ur\sim s}b_{i}\pp{\ur}}$
is finite and real. Therefore, we have by linearity of expectation
that

\begin{eqnarray*}
\sum_{i=1}^{k}\E_{\ur\sim s}a_{i}(\ur)b_{i}\pp{\ur} & = & \E_{\ur\sim s}\sum_{i=1}^{k}a_{i}(\ur)b_{i}\pp{\ur}\\
 & = & \E_{\ur\sim s}\b a(\ur)^{\top}\b b\pp{\ur}\\
 & = & \angs{\b a,\b b}_{s}
\end{eqnarray*}
is finite and real for all $\b a,\b b\in V_{s}$. To show that $\pars{V_{s},\angs{\cdot,\cdot}_{s}}$
is a valid inner-product space, it is easy to establish all the necessary
properties of the inner-product, namely for all $\b a,\b b,\b c\in V_{s},$

$\angs{\b a,\b b}=\angs{\b b,\b a}$

$\angs{\theta\b a,\b b}=\theta\angs{\b a,\b b}$ for $\theta\in\R$

$\angs{\b a+\b b,\b c}=\angs{\b a,\b c}+\angs{\b b,\b c}$

$\angs{\b a,\b a}\geq0$

$\angs{\b a,\b a}=0\Leftrightarrow\b a=\b 0.$ (Where $\b 0(\ep)$
is a function that always returns a vector of $k$ zeros.)
\end{proof}
Next, we give three technical Lemmas, which do most of the work of
the proof.

\gradasinnerproduct*
\begin{proof}
Now, we can write $l(\b w)$ as
\[
l(\b w)=\E_{\br z\sim q_{\b w}}f(\br z)=\E_{\ur\sim s}f\pars{\T_{\b w}\pp{\ur}}.
\]

Since $\T_{\w}\pp{\u}=C\u+\b m$ is an affine function, it's easy
to see that both $\frac{d}{dC_{ij}}\T_{\w}\pp{\u}$ and $\frac{d}{d\b m_{i}}\T_{\w}\pp{\u}$
are independent of $\w$. Therefore, the gradient of $l(\b w)$ can
be written as
\begin{eqnarray*}
\nabla_{w_{i}}l(\b w) & = & \nabla_{w_{i}}\E_{\ur\sim s}f\pars{\T_{\b w}\pp{\ur}}\\
 & = & \E_{\ur\sim s}\nabla_{w_{i}}\b t_{\b w}\pp{\ur}^{\top}\nabla f\pars{\T_{\b w}\pp{\ur}}.\\
 & = & \angs{\b a_{i},\nabla f\circ\T_{\b w}}_{s}.
\end{eqnarray*}
\end{proof}
\orthonormality*
\begin{proof}
It is easy to calculate that
\begin{eqnarray*}
\frac{d}{dm_{i}}\T_{\b w}\pp{\u} & = & \b e_{i}\\
\frac{d}{dC_{ij}}\b t_{\b w}\pp{\u} & = & \b e_{i}u_{j},
\end{eqnarray*}
where $\b e_{i}$ is the indicator vector in the $i$-th component.
Therefore, we have that
\begin{alignat*}{1}
 & \E_{\ur\sim s}\pars{\frac{d}{dm_{i}}\T_{\b w}\pp{\ur}}^{\top}\pars{\frac{d}{dm_{j}}\T_{\b w}\pp{\ur}}\\
 & \hspace{10pt}=\E_{\ur\sim s}\b e_{i}^{\top}\b e_{j}\\
 & \hspace{10pt}=I[i=j]\\
 & \E_{\ur\sim s}\pars{\frac{d}{dC_{ij}}\T_{\b w}\pp{\ur}}^{\top}\pars{\frac{d}{dm_{k}}\T_{\b w}\pp{\ur}}\\
 & \hspace{10pt}=\E_{\ur\sim s}\ur_{j}\b e_{i}^{\top}\b e_{k}\\
 & \hspace{10pt}=I[i=k]\ \E_{\ur\sim s}\ur_{j}\\
 & \hspace{10pt}=0\\
 & \hspace{10pt}\text{(since zero mean)}\\
 & \E_{\ur\sim s}\pars{\frac{d}{dC_{ij}}\T_{\b w}\pp{\ur}}^{\top}\pars{\frac{d}{dC_{kl}}\T_{\b w}\pp{\ur}}\\
 & \hspace{10pt}=\E\ur_{j}\ur_{l}\b e_{i}^{\top}\b e_{k}\\
 & \hspace{10pt}=I[i=k]\E_{\ur\sim s}\ur_{j}\ur_{l}\\
 & \hspace{10pt}=I[i=k]I[j=l]\\
 & \hspace{10pt}\text{(since unit variance and zero mean)}
\end{alignat*}

These three identities are equivalent to stating that $\left\{ \b a_{i}\right\} $
are orthonormal in $\angs{\cdot,\cdot}_{s}$.
\end{proof}
\expecteddifferenceoftransforms*
\begin{proof}
Let $\Delta\b m$ and $\Delta S$ denote the difference of the $\b m$
and $S$ parts of $\b w$, respectively. We want to calculate
\begin{align*}
 & \E_{\ur\sim s}\Verts{\T_{\b w}\pp{\ur}-\T_{\b v}\pp{\ur}}_{2}^{2}\\
 & =\E_{\ur\sim s}\Verts{\Delta C\rep+\Delta\b m}_{2}^{2}\\
 & =\E_{\ur\sim s}\pars{\Verts{\pp{\Delta C}\ur}_{2}^{2}+2\Delta\b m^{\top}\Delta C\ur+\Verts{\Delta\b m}_{2}^{2}}.
\end{align*}
It is easy to see that the expectation of the middle term is zero,
and the last is a constant. The expectation of the first term is
\begin{eqnarray*}
\E_{\ur\sim s}\Verts{\pp{\Delta C}\ur}_{2}^{2} & = & \E_{\ur\sim s}\ur^{\top}\pp{\Delta C}^{\top}\pp{\Delta C}\ur\\
 & = & \E_{\ur\sim s}\tr\pars{\ur^{\top}\pp{\Delta C}^{\top}\pp{\Delta C}\ur}\\
 & = & \E_{\ur\sim s}\tr\pars{\pp{\Delta C}^{\top}\pp{\Delta C}\ur\ur^{\top}}\\
 & = & \tr\pars{\pp{\Delta C}^{\top}\pp{\Delta C}}=\Verts{\nabla C}_{F}^{2}.\\
 &  & \text{(since zero mean and unit variance)}
\end{eqnarray*}
Putting this together gives that
\begin{eqnarray*}
\E_{\ur\sim s}\Verts{\T_{\b w}\pp{\ur}-\T_{\b v}\pp{\ur}}_{2}^{2} & = & \Verts{\Delta C}_{F}^{2}+\Verts{\Delta\b m}_{2}^{2}\\
 & = & \Verts{\b w-\b v}_{2}^{2}.
\end{eqnarray*}
\end{proof}

\clearpage{}

\section{Proof for Example Function}

\egfungeneral*
\begin{proof}
For a general distribution, we have that

\begin{eqnarray*}
\E f(\br z) & = & \frac{a}{2}\E\Verts{\zr-\E\bb{\zr}+\E\bb{\zr}-\zo}_{2}^{2}\\
 & = & \frac{a}{2}\E\Bigl(\Verts{\zr-\E\bb{\zr}}_{2}^{2}\\
 &  & +2\pars{\zr-\E\bb{\zr}}^{\top}\pars{\E\bb{\zr}-\zo}+\Verts{\E\bb{\zr}-\zo}_{2}^{2}\Bigr)\\
 & = & \frac{a}{2}\pars{\tr\V\bb{\zr}+\Verts{\E\bb{\zr}-\zo}_{2}^{2}}.
\end{eqnarray*}
Now, if $\qw$ is a location-scale family, we have that $\zr=C\ur+\b m$.
Thus,
\begin{eqnarray*}
\tr\V\bb{\zr} & = & \tr\V\bb{C\ur+\b m}\\
 & = & \tr\V\bb{C\ur}\\
 & = & \tr C\V\bb{\ur}C^{\top}\\
 & = & \tr CC^{\top}\V\bb{\ur}.
\end{eqnarray*}
Meanwhile, we have that
\begin{eqnarray*}
\Verts{\E\bb{\zr}-\zo}_{2}^{2} & = & \Verts{\E\bb{C\ur+\b m}-\zo}_{2}^{2}\\
 & = & \Verts{C\E\bb{\ur}+\b m-\zo}_{2}^{2}
\end{eqnarray*}

Thus,
\[
\E f(\br z)=\frac{a}{2}\pars{\tr C\V\bb{\ur}C^{\top}+\Verts{C\E\bb{\ur}+\b m-\zo}_{2}^{2}}.
\]
The case where $s$ is standardized follows from substituting $\E\bb{\ur}=0$
and $\V\bb{\ur}=I$ and applying the fact that $\tr CC^{\top}=\Verts C_{F}^{2}.$
\end{proof}
\clearpage{}

\section{Proofs for Solution Guarantees\label{sec:Solution-Guarantees-appendix}}

\diagCgradbound*
\begin{proof}
Define $\w'$ to be $\w$ but with $C_{ii}$ set to zero. We will
first show that $\frac{dl\pp{\w'}}{dC_{ii}}=0$. Using the definition
of $\T_{w}$ and the fact that $\frac{d}{dC_{ij}}\T_{w}\pp u=\b e_{i}u_{j}$
gives that
\begin{alignat}{1}
\frac{d}{dC_{ii}}l\pp{\w'} & =\E_{\ur\sim s}\frac{d}{dC_{ii}}f\pp{\T_{\w'}\pp{\ur}}\label{eq:dl_dCii_diagC}\\
 & =\E_{\ur\sim s}\ur_{i}\b e_{i}^{\top}\nabla f\pp{\T_{\w'}\pp{\ur}}\\
 & =0.
\end{alignat}

The final equality above follows from the facts that $\E\ur_{i}=0$
and $\ur_{i}\perp\b e_{i}^{\top}\nabla f\pp{\T_{\w'}\pp{\ur}}$ (Since
$\T_{\w'}\pp{\u}$ ignores $\u_{i}$) so the expectation in \ref{eq:dl_dCii_diagC}
is over two independent random variables, one with mean zero. Now,
by \ref{thm:lsmoothness}, $l$ is also $M$-smooth, thus

\begin{alignat*}{1}
\verts{\frac{dl\pp{\w}}{dC_{ii}}} & =\verts{\frac{dl\pp{\w'}}{dC_{ii}}-\frac{dl\pp{\w}}{dC_{ii}}}\\
 & \leq\Verts{\nabla l\pp{\w'}-\nabla l\pp{\w}}_{2}\\
 & \leq M\Verts{\w'-\w}_{2}\\
 & =M\verts{C_{ii}}.
\end{alignat*}
\end{proof}
\smoothnesssol*
\begin{proof}
First, suppose that $C$ is diagonal. Since $\w$ minimizes $l+h$,
$\nabla l\pp{\w}=-\nabla h\pp{\w}.$ The gradient of $h$ with respect
to $C$ is $-C^{-\top}.$ Thus, $\vv{\frac{dl\pp{\w}}{dC_{ii}}}=\vv{\frac{dh\pp{\w}}{dC_{ii}}}=\frac{1}{\vv{C_{ii}}}.$
But by \ref{lem:diag-C-grad-bound}, $\vv{\frac{dl\pp{\w}}{dC_{ii}}}\leq M\vv{C_{ii}}.$
This establishes the claim for diagonal $C$.

Now, consider some non-diagonal $C$. Let the singular value decomposition
be $C=USV^{\top}.$ Define $f_{U}\pp{\z}=f\pp{U\z}$ and define $l_{U}$
with respect to $f_{U}.$ Let $\w'=\pp{S,U^{\top}\b m}.$ Then, the
following statements are equivalent to $\w\in\argmin_{\w}l\pp{\w}+h\pp{\w}$:

\begin{align*}
 & \pp{C,\b m}\in\argmin_{\pp{C,\b m}}\E_{\ur\sim s}f\pars{C\ur+\b m}-\log\verts C\\
 & \Leftrightarrow\pp{S,\b m}\in\argmin_{\pp{S,\b m}}\E_{\ur\sim s}f\pars{USV^{\top}\ur+\b m}-\log\verts{USV^{\top}}\\
 & \Leftrightarrow\pp{S,\b m}\in\argmin_{\pp{S,\b m}}\E_{\ur\sim s}f\pars{US\ur+\b m}-\log\verts S\\
 & \Leftrightarrow\pp{S,\b m}\in\argmin_{\pp{S,\b m}}\E_{\ur\sim s}f_{U}\pars{S\ur+U^{\top}\b m}-\log\verts S\\
 & \Leftrightarrow\w'\in\argmin_{\w}l_{U}\pp{\w}+h\pp{\w}.
\end{align*}

Thus, $\w$ minimizing $l+h$ is equivalent to $\w'$ minimizing $l_{U}+h.$
Since $f_{U}$ is $M$-smooth and $S$ is diagonal, we know that $S_{ii}\geq\frac{1}{\sqrt{M}}$
for all .
\end{proof}
\clearpage{}

\section{Proofs with Convexity}

\vilocconvex*It's easy to see that $l$ is minimized by $\bar{\b w}=\pp{\zo,\b 0_{d\times d}}.$
By \ref{thm:convexity}, $l\pp{\b w}$ is $c$-strongly convex. Thus
applying a standard inner-product result on strong convexity \citep[Thm. 2.1.9]{Nesterov_2014_Introductorylecturesconvex},

\begin{align*}
c\Verts{\b w-\bar{\b w}}_{2}^{2}\leq & \angs{\nabla l\pp{\b w}-\nabla l\pp{\bar{\b w}},\b w-\bar{\b w}}\\
 & \text{(since \ensuremath{l} is strongly convex)}\\
= & \angs{\nabla l\pp{\b w},\b w-\bar{\b w}}\\
 & \text{(since \ensuremath{\nabla l\pp{\bar{\b w}}=0})}\\
= & -\angs{\nabla h\pp{\b w},\b w-\bar{\b w}}\\
 & \text{(since \ensuremath{\nabla l(\b w)+\nabla h\pp{\b w}=0})}\\
= & \tr\pars{C^{-\top}C}\\
 & \text{(since \ensuremath{\nabla_{C}h\pp{\b w}=-C^{-\top},}\ensuremath{\nabla_{\b m}h\pp{\b w}=0}).}\\
= & \tr I=d.
\end{align*}

The result follows from observing that $\Verts{\b w-\bar{\b w}}_{2}^{2}=\Verts C_{F}^{2}+\Verts{\b m-\zo}_{2}^{2}.$

\clearpage{}

\section{Convergence Considerations\label{sec:Convergence-Proofs}}

\begin{restatable}{lem1}{hsmooth}

Let $\qw=\locscale(\b m,C,s)$ with parameters $\w=\pp{\b m,C}$.
Then, $h\pp{\w}=\E_{\zr\sim\qw}\bracs{\log\qw\pp{\zr}}$ is $M$-smooth
over $\mathcal{W}_{M}$.\label{thm:h-smooth}

\end{restatable}
\begin{proof}
Take $\w=\pp{C,\b m}\in\mathcal{W}_{M}$ and $\b v=\pp{B,\b n}\in\mathcal{W}_{M}.$
We write $h\pp C$ since $h\pp{\w}$ is independent of $\b m$. The
gradient is $\nabla h\pp C=C^{-T}$. Now, use that $\Verts{AX}_{F}\leq\Verts A_{2}\Verts X_{F}$
to get that
\begin{alignat*}{1}
\Verts{\nabla h(B)-\nabla h(C)}_{F} & =\Verts{B^{-1}-C^{-1}}_{F}\\
 & =\Verts{B^{-1}\pp{B-C}C^{-1}}_{F}\\
 & \leq\Verts{B^{-1}}_{2}\Verts{C^{-1}}_{2}\Verts{B-C}_{F}.
\end{alignat*}
But, since $\w\in\mathcal{W}_{M},$ $\Verts{C^{-1}}_{2}=\frac{1}{\sigma_{\min}\pp C}\leq\sqrt{M}$
and similarly for $C$. This establishes that $\Verts{\nabla h(B)-\nabla h(C)}_{F}\leq M\Verts{B-C}_{F},$
equivalent to the result.
\end{proof}
\begin{restatable}{thm1}{gaussprox}

\label{thm:gaussprox}Suppose $h\pp{\b w}$ corresponds to a location-scale
family with a standardized $s$, and $\b w=\pp{\b m,C}$.
\begin{itemize}
\item If $C$ has singular value decomposition $C=USV^{\top},$ then $\mathrm{proj}_{\mathcal{W}_{M}}\pp{\b w}=\pp{\b m,UTV^{\top}},$
where $T$ is a diagonal matrix with $T_{ii}=\max\pars{S_{ii},\frac{1}{\sqrt{M}}}.$
\item If $C$ is triangular with a positive diagonal, then $\prox_{\gamma}\pp{\b w}=\pp{\b m,C+\Delta C},$
where $\Delta C$ is a diagonal matrix with $\Delta C_{ii}=\frac{1}{2}\pars{\sqrt{C_{ii}^{2}+4\gamma}-C_{ii}}.$
\end{itemize}
\end{restatable}

\begin{proof}
\textbf{(Proximal Operator)} We know that $h\pp{\b w}=\mathrm{Const.}-\log\verts C.$
Write $\b w=\pp{\b m,C}$ and $\b v=\pp{\b n,B}.$ Then, we can write
the proximal operator as

\begin{alignat*}{1}
\prox_{\lambda}\pp{\b w} & =\argmin_{\b v}-\log\verts B+\frac{1}{2\lambda}\Verts{\b v-\b w}_{2}^{2}
\end{alignat*}
Now, assuming that $C$ is triangular, the solution will leave all
entries of $\b w$ other than the diagonal entries of $C$ unchanged.
Then, we will have that $\log\verts B=\sum_{i=1}^{d}\log B_{ii}.$
Since
\begin{eqnarray*}
\argmin_{x>0}-\log x+\frac{1}{2\lambda}\pp{x-y}^{2} & = & \frac{y+\sqrt{y^{2}+4\lambda}}{2}
\end{eqnarray*}

The solution is to set
\begin{alignat*}{1}
B_{ii} & =\frac{1}{2}\pars{C_{ii}+\sqrt{C_{ii}^{2}+4\lambda}}\\
 & =C_{ii}+\frac{1}{2}\pars{\sqrt{C_{ii}^{2}+4\lambda}-C_{ii}}.
\end{alignat*}

\textbf{(Projection Operator)} Von-Neumann's trace inequality states
that $\verts{\tr A^{\top}B}\leq\sum_{i}\sigma_{i}\pp A\sigma_{i}\pp B.$
Consider any candidate solution $B$ with SVD $QTP^{\top}.$ Then,
we can write that
\begin{eqnarray*}
\Verts{B-C}_{F}^{2} & = & \tr\pars{B-C}^{\top}\pp{B-C}\\
 & = & \Verts B_{F}^{2}-2\tr\pp{B^{\top}C}+\Verts C_{F}^{2}\\
 & \geq & \Verts T_{F}^{2}-2\sum_{i}T_{ii}S_{ii}+\sum_{i}S_{ii}^{2}\\
 & = & \sum_{i}\pars{T_{ii}-S_{ii}}^{2}.
\end{eqnarray*}
We can minimize this lower bound by choosing $T_{ii}=\max\pp{1/\sqrt{M},S_{ii}},$
with a corresponding value of $\sum_{i}\max\pp{0,1/\sqrt{M}-S_{ii}}^{2}.$
Thus any valid solution will have $\Verts{B-C}_{F}^{2}$ at least
this large.

However, suppose we choose $B=UT_{ii}V^{\top}$ with $T_{ii}$ as
above. Then,
\[
\Verts{B-C}_{F}^{2}=\Verts{UTV^{\top}-USV^{\top}}_{F}^{2}=\sum_{i}\pp{T_{ii}-S_{ii}}^{2},
\]
so this value $B$ is optimal.
\end{proof}
\clearpage{}

\begin{thm}
~
\end{thm}

\[
\E_{p(\r x)}\bb{\r x}=\argmin_{\mu}\ \E_{p\pp{\r x}}\Verts{\mu-\r x}_{2}^{2}
\]

\end{document}